%% file: main.tex
\documentclass{article}

\usepackage{microtype}
\usepackage{graphicx}
\usepackage{subcaption}
\usepackage{booktabs} % for professional tables

\usepackage{hyperref}

\usepackage[preprint]{neurips_2026}

\usepackage{amsmath}
\usepackage{amssymb}
\usepackage{mathtools}
\usepackage{amsthm}
\usepackage{xspace}
\usepackage[capitalize,noabbrev]{cleveref}

\theoremstyle{plain}
\newtheorem{theorem}{Theorem}[section]
\newtheorem{proposition}[theorem]{Proposition}

\newtheorem{definition}[theorem]{Definition}
\newtheorem{assumption}[theorem]{Assumption}

\theoremstyle{remark}
\newtheorem{remark}[theorem]{Remark}
\usepackage{amsmath}
\usepackage{amsfonts}
\usepackage{amssymb}
\usepackage[table]{xcolor}   % for \rowcolor, \cellcolor and color shades
\usepackage{colortbl}        % enhances table coloring
\usepackage[capitalize]{cleveref}
\usepackage{wrapfig}
\usepackage{tabularx}     
\usepackage{booktabs}    
\usepackage[table]{xcolor} 
\usepackage{makecell}     
\usepackage{amssymb}     
\usepackage[textsize=tiny]{todonotes}
\hypersetup{hidelinks}
\usepackage{etoc}
\etocdepthtag.toc{mtchapter}
\usepackage{titletoc}
\usepackage{minitoc}
\usepackage{colortbl}
\usepackage{tcolorbox}
\tcbset{
  highlightbox/.style={
    colback=greyL, % Background color of the box
    colframe=black, % Border color
    boxrule=1pt, % Border thickness
    arc=5pt, % Rounded corner radius
    boxsep=5pt, % Padding inside the box
    left=5pt, right=5pt, top=0pt, bottom=0pt, % Padding adjustments
    width=\textwidth, % Box width
    fontupper=\itshape % Italic text inside the box
  }
}

\hypersetup{
    colorlinks=true,
    citecolor=cyan, 
    linkcolor=cyan,  
    urlcolor=cyan,
}

\usepackage[table]{xcolor}
\title{Shaping Schema via Language Representation as the Next Frontier for LLM Intelligence Expanding}
 \author{
Zhiqin Yang$^{1}\thanks{Equal contributions}$ \quad Yuhan Liu$^{2\textcolor{cyan}{*}}$ \quad  Jingwen Fu$^3\thanks{  Corresponding author}$  \quad Pei Fu$^2$ \quad
\\\bf Bo Han$^4$ \quad   Masashi Sugiyama$^{5,6}$ \quad   
 Nanning Zheng$^{7}$ \quad  \\
 $^1$The Hong Kong University of Science and Technology \quad
 $^2$ MiLM Plus, Xiaomi Inc\quad \\
 $^3$Zhongguancun Academy \quad
 $^4$Hong Kong Baptist University\quad 
 $^5$ The University of Tokyo \\
 $^6$RIKEN Center for Advanced Intelligence Project \quad 
 $^7$ Xi'an Jiaotong University  
 \setcounter{footnote}{0}
}
\begin{document}

\maketitle

\begin{abstract}

Although natural language is the default medium for Large Language Models (LLMs), its limited expressive capacity creates a profound bottleneck for complex problem-solving. While recent advancements in AI have relied heavily on scaling, merely internalizing knowledge does not guarantee its effective application. Defining language representation as the linguistic and symbolic constructs used to map and model the real world, this paper argues that shaping schemas through advanced language representation is the next frontier for expanding LLM intelligence. We posit that an LLM’s knowledge activation and organization—its schema—depends heavily on the structural and symbolic sophistication of the language used to represent a given task. This paper contributes both a formalization of this claim and the empirical evidence to support it.
With a new formalization, \textit{we present multiple lines of evidence to support our position}: \textbf{Firstly}, we review recent empirical practices and emerging methodologies that demonstrate the substantial performance gains achievable through deliberate language representation design, even without modifying model parameters or scale. \textbf{Secondly}, we conduct controlled experiments showing that LLM performance and its internal feature activations vary under different language representations of the same underlying task. Together, these findings highlight language representation design as a promising direction for future research.

\end{abstract}

\input{main/1_Intro}
\input{main/2_Back}
\input{main/3_Defi}

\input{main/4_Case}

\input{main/5_Exp}

\input{main/6_Alter}
\input{main/7_Open_probs}
\input{main/8_Action}

\input{main/9_Conclu}

\clearpage
\bibliography{example_paper}
\bibliographystyle{unsrt}

%%%%%%%%%%%%%%%%%%%%%%%%%%%%%%%%%%%%%%%%%%%%%%%%%%%%%%%%%%%%%%%%%%%%%%%%%%%%%%%
%%%%%%%%%%%%%%%%%%%%%%%%%%%%%%%%%%%%%%%%%%%%%%%%%%%%%%%%%%%%%%%%%%%%%%%%%%%%%%%
% APPENDIX
%%%%%%%%%%%%%%%%%%%%%%%%%%%%%%%%%%%%%%%%%%%%%%%%%%%%%%%%%%%%%%%%%%%%%%%%%%%%%%%
%%%%%%%%%%%%%%%%%%%%%%%%%%%%%%%%%%%%%%%%%%%%%%%%%%%%%%%%%%%%%%%%%%%%%%%%%%%%%%%
\newpage

\appendix

\etocdepthtag.toc{mtappendix}
\etocsettagdepth{mtchapter}{none}
\etocsettagdepth{mtappendix}{subsection}

% \part{Appendix} % Start the appendix part
% \renewcommand{\contentsname}{}
\renewcommand{\contentsname}{\textcolor{cyan}{Appendix Table of Contents}}

\tableofcontents
\newpage
% \vspace{1cm}
\input{appendix/appendix_more_bg}

\input{appendix/appendix_related}
\input{appendix/appendix_more_exp}

\end{document}

%% file: main/1_Intro.tex
\section{Introduction}
\begin{center}
    \itshape
    ``The limits of my language mean the limits of my world.''
\end{center}

\begin{flushright}
    \hspace*{2em} ---Ludwig Wittgenstein~\citep{wittgenstein1922tractatus}
\end{flushright}

Large Language Models (LLMs)~\cite{radford2019language,brown2020language,jaech2024openai,guo2025deepseek,team2025kimi} have emerged as the dominant paradigm in contemporary artificial intelligence~\cite{radford2019language}, largely driven by the empirical success of scaling model size and training data~\cite{weiemergent,hoffmann2022training}. While this scaling strategy enables LLMs to internalize vast amounts of knowledge within their parameters~\cite{petroni2019language,roberts2020much}, the mere presence of knowledge does not guarantee its effective activation, organization, or use~\cite{brown2020language,yao2022react}. Crucially, as illustrated in Figure~\ref{fig:figbegin}, natural language itself acts as a massive bottleneck. The complexity of the real-world task space far exceeds what natural language can naturally express, creating a massive information gap (e.g., an estimated $10^{14}$
  bits for a numerical weather model versus a mere $10^{2}$
  bits for a natural language forecast). \textit{In practice, LLM performance is often constrained not by what the model has learned, but by this narrow linguistic channel through which the complexities of the real world are encoded, accessed, and composed during inference.}

% \vspace{-5pt}
\begin{figure}[t]
\centering \includegraphics[width=1\columnwidth]{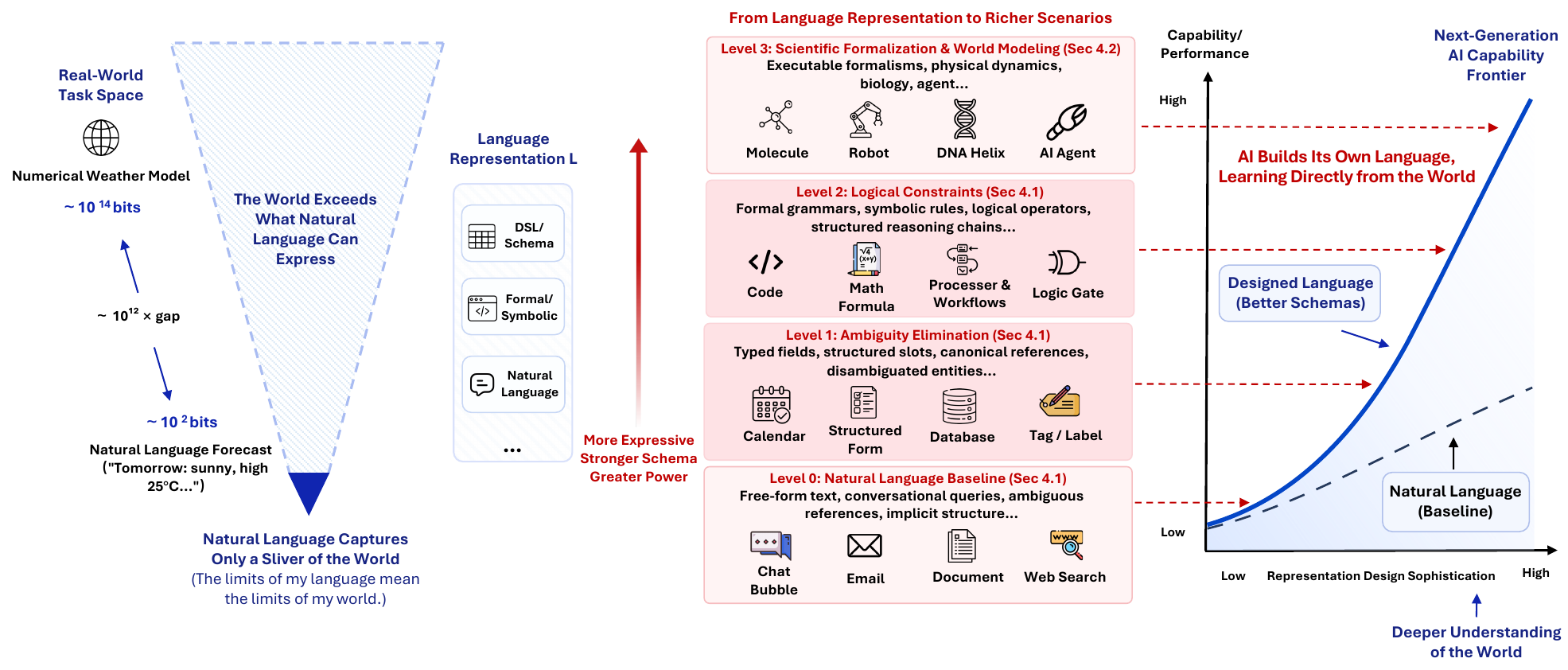}
\caption{\textbf{Language representation as a frontier for LLM intelligence.} Natural language encodes only a fraction of world information (Left). We organize representations along an axis of increasing design sophistication, from natural-language baselines (Level 0) through ambiguity elimination (Level 1) and logical constraints (Level 2) to scientific formalization and world modeling (Level 3). Each level induces progressively richer internal schemas, pushing the capability frontier beyond the natural-language baseline (Right).}
\label{fig:figbegin}
\end{figure}

Inspired by cognitive science~\cite{hassabis2017neuroscience,zhao2023brain,mitchell2024debates}, we introduce the notion of a \emph{schema}~\cite{Bartlett1932RememberingAS} to characterize the internal framework through which knowledge is activated and structured. A schema refers to the representational and organizational patterns that determine how different pieces of knowledge are invoked, related, and operationalized in response to a task~\cite{Bartlett1958ThinkingAE,tompkins1993teaching}. In LLMs, these schemas are intrinsically tied to \emph{language representations}. Crucially, in this context, \textit{a ``language representation'' refers to the linguistic and symbolic constructs used to map and model the real world.} It is the designed medium through which real-world entities, physics, logic, and constraints are translated into a format the LLM can process. To overcome the natural language bottleneck, we propose organizing these real-world representations along an axis of increasing design sophistication (Figure~\ref{fig:figbegin}, center). This progression moves from the ambiguous, free-form baseline of natural language (Level 0), through ambiguity elimination via structured formats (Level 1) and rigorous logical constraints like code and math (Level 2), ultimately reaching complex scientific formalization and explicit world modeling (Level 3).

As LLMs approach the limits of their current representational capacities~\cite{john2025power,sutskever2025dwarkesh,mohsin2025fundamental}, we argue that further progress cannot rely solely on continued parameter scaling or external tool use. Instead, this paper holds the point that \textbf{shaping schema via language representation is the next frontier for LLM intelligence expanding.} As shown in our capability trajectory (Figure~\ref{fig:figbegin}, right), elevating how we use language to represent the world pushes the performance frontier well beyond the natural language baseline, unlocking a deeper understanding of reality.

Overall, the contributions of this paper are summarized as follows:
\begin{itemize}
\item We formalize the notions of schema, language representation, and language representation design, where representation is framed as the linguistic modeling of real-world structure. Based on this, we propose a unified analytical framework organized along an axis of design sophistication from level 0 to 3.
\item To substantiate the critical importance of language representation design, we review recent empirical practices and emerging methodologies, and conduct controlled experiments to isolate its effects.
\item We identify open questions and roadmap, outlining promising directions for advancing the frontier of LLM intelligence toward AI-constructed formal languages.
\end{itemize}

%% file: main/2_Back.tex
\section{Background: Language and Intelligence}
The essence of general intelligence is widely believed to lie in integrating diverse cognitive functions~\cite{hassabis2017neuroscience,zhao2023brain,mitchell2024debates,mirjalili2025using}, enabling advanced reasoning and complex problem-solving~\cite{haber2022prefrontal}.  To elucidate the underlying cognitive mechanisms, schema~\cite{Bartlett1932RememberingAS,Bartlett1958ThinkingAE,tompkins1993teaching} was introduced as a compelling framework for how the brain organizes knowledge, drawing upon connections to prior experiences to structure and guide the interpretation of new information~\cite{rumelhart2017schemata,chen2025schema,Smith2021schema}. 

Within this process, language serves as a crucial bridge for cognitive representation and interaction by encoding cognitive schemas that shape the way we think and act~\cite{jamali2024semantic}. According to the weak version of the Sapir–Whorf hypothesis~\cite{whorf1956language,lucy1997linguistic}, linguistic systems do not dictate thought in any absolute sense, but instead, they subtly guide and channel it by framing the cognitive schemas through which people interpret their lived experiences~\cite{bisk2020experience,ansorge2022linguistic,piantadosi2022meaning}. These schemas, or mental frameworks, organize and interpret sensory information, guiding attention, classification, and reasoning. Language plays a central role in this by providing schematic representations of key concepts such as space, time, causality, and events~\cite{edwards1993language,fausey2008english}. Through the specific vocabulary, grammar, and metaphors of each language, these linguistic schemas direct how speakers categorize objects, assign temporal relations, and infer causal connections~\cite{talmy2000toward,boroditsky2001does}. This interplay between language and schemas is central to how cognition is shaped: language not only reflects but also constructs the frameworks that govern perception and reasoning~\cite{Wilhelm1996from,boroditsky2011language}.

The proliferation of LLMs has sparked comparisons to human intelligence and fueled speculation that their advancement could lead to artificial general intelligence (AGI)~\cite{lake2017building,binz2023using}.
Recent research demonstrated that LLMs possess schema-like structures that shape their performance~\cite{ameisen2025circuit}.
Prior studies also revealed that LLMs exhibit low-level semantic correlation structures akin to those observed in humans~\cite{2025semantic}. Whereas human cognition is guided by schemas, recent research further suggested that LLMs possess analogous schema-like structures that shape their performance~\cite{ameisen2025circuit}.
Given that human cognition is guided by schemas~\cite{Bartlett1932RememberingAS}, we \textbf{conceptualize schemas in LLMs as} an abstract, internalized graph-like framework that captures how the embedded knowledge of LLMs is activated and organized.

Specifically, numerous studies have further unlocked the potential of LLMs by implicitly or explicitly providing or modifying schemas within them~\cite{wang2025under,chen2025schema}. \textbf{First of all}, different content of inputs can activate distinct schemas in LLMs. For instance, in-context~\cite{dong2024survey} information modulates embeddings and attention weights across layers~\cite{yousefi2023decoding}, while chain-of-thought (CoT)~\cite{wei2022chain} prompting elicits reasoning capabilities, even when invalid reasoning is provided~\cite{wang2023towards}. \textbf{Secondly}, different languages also represent different reasoning schemas. Wang et al.~\cite{wang2025under} found that the model placed more attention on causes when given Chinese prompts, while it was more balanced in terms of cause and effect when given English prompts. \textbf{Furthermore}, both explicit and implicit schemas serve as vital mechanisms for enhancing LLM performance. Explicit schemas provide a cognitive scaffolding: Schema-Activated in Context Learning (SA-ICL)~\cite{chen2025schema} shows that retrieving these schemas guides reasoning, while in semantic parsing, they facilitate the translation of natural language into Structured Query Language (SQL)~\cite{gupta2025schema, labate2024infusing}. More fundamentally, clone-structured causal graphs (CSCGs)~\cite{swaminathan2023schema} enabled generalization by rebinding novel tokens into the slots of template circuits (schemas). Beyond explicit structures, Dhanraj et al.~\cite{dhanraj-eliasmith-2025-improving} probed the hidden states, decoding them into structured neurosymbolic representations that enable targeted manipulation and performance improvements. 
% More information about language and intelligence can be referred to the Appendix~\ref{appendix:more_bg}.

%% file: main/3_Defi.tex
\section{Language Representation Design}

\subsection{Formulation}
\label{subsec:formulation}

Given a question space \(Q\) and a specific question \(q \in Q\), the objective is to obtain an answer \(a \in A\), where \(A\) denotes the answer space. We assume there exists a target mapping \(f \in \mathcal{F}\) with
$f: Q \to A$,
which defines the ideal correspondence between questions and answers. Since a large language model (LLM) operates purely on linguistic representations, both questions and answers must be expressed in a common language space \(L\). We introduce a language encoding map \(L(\cdot)\) such that the question \(q\) and answer \(a\) are represented as \(L(q)\) and \(L(a)\), respectively. The set of all possible languages is denoted as $\mathcal{L}$. We also denote \(\pi\) as the LLM, which maps language representations $L(q)$ to language representations $L(a)$.

\begin{assumption}
    For all $L\in \mathcal{L}$, the corresponding function $L(\cdot)$ is isomorphism.
\end{assumption}
\begin{remark}
    This assumption ensures that the language can accurate describe the questions and answers.
\end{remark}
 The overall induced mapping from \(Q\) to \(A\) is therefore given by:
$L^{-1}  \pi  L \in \mathcal{F}$. Language design aims to identify an appropriate language space \(L \in \mathcal{L}\) such that the induced mapping best approximates the target function \(f\). 

\begin{definition}[Language Design]
Given a distance measure \(d(\cdot,\cdot)\) defined on the function space \(\mathcal{F}\), language design is formulated as the following optimization problem:
\begin{equation}
\operatorname*{argmin}_{L \in \mathcal{L}} d\bigl(f,\; L^{-1}  \pi  L\bigr).
\end{equation}
\end{definition}

From this perspective, prompt engineering can be interpreted as an operation on the question space \(Q\). Specifically, we consider a class of transformations \(\mathcal{G}\), where each \(g \in \mathcal{G}\) is a mapping: $g: Q \to Q$,
that modifies the input question prior to its encoding in the language space. Such modifications may involve augmenting the question with additional information or incorporating explicit hints to guide the model’s reasoning.

\begin{definition}[Prompt Engineering]
Prompt engineering seeks to solve the following optimization problem:
\begin{equation}
\operatorname*{argmin}_{g \in \mathcal{G}} d\bigl(f,\; L^{-1}  \pi  L  g\bigr).
\end{equation}
\end{definition}

The essential difference between language design and prompt engineering lies in their respective constraints and scope of influence. The language map \(L\) is typically required to be an isomorphism, as it must faithfully represent both questions and answers within the language space. In contrast, transformations in \(\mathcal{G}\) are subject to far fewer constraints and affect only the input side. Consequently, language design influences both the question and answer representations, whereas prompt engineering modifies only the question representation prior to model inference.

\subsection{Shaping Schema with Language Representation}
\label{subsec:schema}
\begin{assumption}
    \label{ass:schema}
    There is a schema space $\mathcal{S}$ and a small value $\epsilon$, such that we can construct functions $\pi_a:\mathcal{S} \to A$ and $\pi_s:Q\to\mathcal{S}$, for any $L\in \mathcal{L}$ :
    \begin{equation}
        d(L^{-1} \pi_a\pi_s L,L^{-1} \pi L)\leq \epsilon,
    \end{equation}
    where $d(\cdot,\cdot)$ is a distance measure on the function space.
\end{assumption}

For a task \(f:Q\to A\), we denote \(s_{f}\) as the target schema.
The distribution of the schema on the schema representation with the language $L$ is denoted with $s_f^L=L(s_f)$.
The language-induced schema is \(s_\pi^L=\pi_s L(q)\).

Given a task $q\sim Q$ and $a=f(q)$, the schema-mismatch of language \(L\)  is the Kullback–Leibler divergence:
\begin{equation}
\boxed{\mathrm{SM}(L)\triangleq D_{\mathrm{KL}}\!\left(s_{f}^L\,\middle\|\,s_\pi^L\right).}
\end{equation}
A language is schema-matched when \(\mathrm{SM}(L)=0\); any positive value quantifies the extra bits required to re-route the model’s internal circuitry from the language-evoked pattern to the task-required pattern.

\begin{proposition}[Bounds on prediction error]
\label{prop:bounds_on_prediction_error}
Let $\mathcal{I}_{\pi_a}(s)$ be the Fisher Information Matrix of the action mapping $\pi_a$ at schema $s$. For any distance $d$ on $\mathcal{F}$ defined as the squared Fisher-Rao distance in the action space, the prediction error satisfies:
\begin{equation}
\frac{\sigma^{2}_{\min}}{2}\,\mathrm{SM}(L) \;\le\; d(f,\hat{f}_{L}) \;\le\; \frac{\sigma^{2}_{\max}}{2}\,\mathrm{SM}(L),
\end{equation}
where $\hat{f}_L\triangleq L^{-1}\pi_a\pi_sL$, $\sigma^{2}_{\min} = \inf_s \lambda_{\min}(\mathcal{I}_{\pi_a}(s))$, $\sigma^{2}_{\max} = \sup_s \lambda_{\max}(\mathcal{I}_{\pi_a}(s))$.
\end{proposition}

\begin{remark}
    Let $f_L = L^{-1}\pi L$. Under Assumption~\ref{ass:schema}, the law of distance implies
    $
    \frac{\sigma^{2}_{\min}}{2}\,\mathrm{SM}(L)-\epsilon
    \;\le\;
    d\!\left(f,f_L\right)
    \;\le\;
    \frac{\sigma^{2}_{\max}}{2}\,\mathrm{SM}(L)+\epsilon.
    $
    Therefore, the discrepancy between the target mapping $f$ and its language-induced realization $f_L$ is tightly controlled by the schema mismatch of the language representation.  As a result, language design is recasted as a constrained optimization problem,
    $
    \min_{L\in\mathcal{L}} \;\mathrm{SM}(L),
    $
    highlighting schema alignment as the fundamental objective governing prediction accuracy.
\end{remark}

%% file: main/4_Case.tex
\section{Expanding the Intelligence Frontier with Language Representation Design}

Following the axis introduced in Figure \ref{fig:figbegin}, this section examines how representation design unfolds beyond the natural-language baseline (Level 0). Section~\ref{subsec: case_1} covers Levels 1–2, which optimize established methods. Section~\ref{subsec: case_2} covers Level 3, which overcome barriers to tackle new domains. Furthermore, we provide some experimental evidence in Section~\ref{subsec:exp} to support our position.

\subsection{Level 0-2: Strengthening Current Abilities}
\label{subsec: case_1}

Although LLMs already demonstrate strong performance on tasks such as question answering and multi-step reasoning, their outputs remain unstable when expressed purely in natural language (Level 0) \cite{cao2024worst,zhu2023promptrobust}, which inherently lacks explicit logical constraints and is riddled with semantic ambiguity \cite{piantadosi2012communicative,bender2020climbing}. These limitations suggest that the performance bottleneck often arises not from a lack of latent capability, but from the inadequacy of natural language as a stable interface \cite{wei2022emergent,reynolds2021prompt}. Levels 1–2 directly address these two deficiencies: Ambiguity Elimination (Level 1) sharpens token-to-entity precision, while Logical Constraints (Level 2) enforces structural rigor on the inference trajectory, together shaping the model's internal schema to compensate for the deficiencies of natural language.

\begin{wrapfigure}{r}{0.55\textwidth}
  \vspace{-10pt} % 向上收缩空白，抵消 wrapfigure 默认的顶部留白
  \centering
  \includegraphics[width=\linewidth]{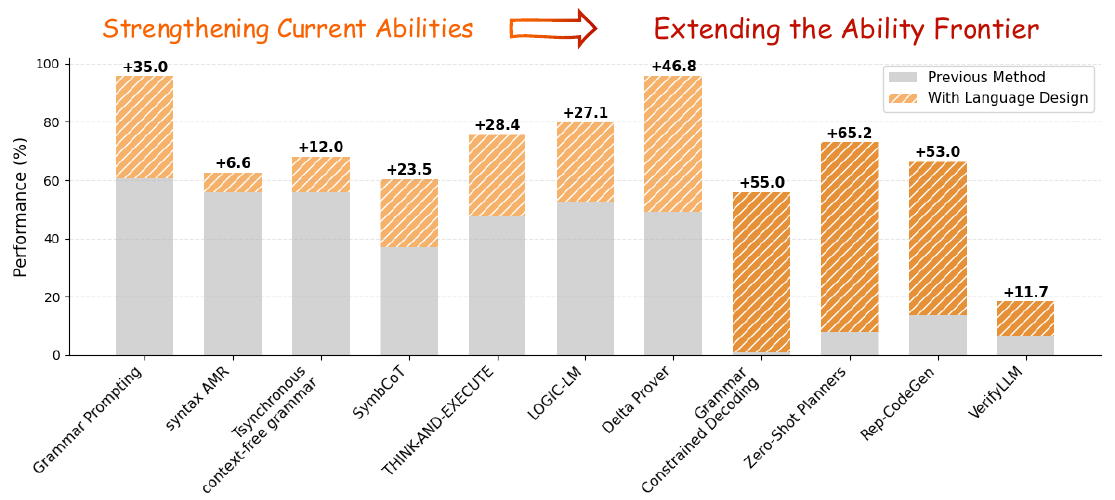}
  \vspace{-5mm}
  \caption{Performance gains and capability expansion through language representation design.}
  \label{fig:performance gap}
  \vspace{-15pt} % 向下收缩空白，让正文文字更贴近图片底部
\end{wrapfigure}

\textbf{Ambiguity Elimination.} Simultaneously, language design eliminates linguistic ambiguity to ensure the precise activation of task-relevant knowledge nodes \cite{wang2023grammar,shin2021constrained,park2024grammar,wei2024improving}. 
Empirical studies demonstrate that even semantically equivalent variations in wording can activate disparate internal representations, leading to inconsistent predictions and unstable reasoning trajectories for the same underlying task \cite{gao2021making,perez2021true,naik2018stress,salinas2024butterfly}.
By employing task-specific constructs and symbolic conventions, it replaces fuzzy natural language cues with precise representations \cite{labate2024infusing,zhou2025solving,geng2025generating,barradas2025combining}. This clarity prevents the model from incorrectly inferring schemas from noisy or underspecified inputs, which often leads to disparate internal activations for the same task. The synergy between structured organization and precise activation allows the internal schema to remain strictly aligned with task requirements.

\textbf{Logical Constraints.} Language representation design addresses the inherent lack of structure in natural language by imposing explicit logical constraints \cite{chae2024language,pan2023logic,suris2023vipergpt,ma2026thinking}. Natural language lacks the formal structural anchors required to guide precise, step-by-step reasoning. Meanwhile, it introduces diverse surface realizations and inconsistent structural cues, forcing models to infer task-relevant schemas from noisy and underspecified textual inputs \cite{ramji2026thinking,pinker2007language,wang2026agentspex,zou2026constituent}. By utilizing rule-based generation and formal specifications, this approach provides the structural anchors necessary for the coherent organization of the model’s internal schema \cite{xu2024faithful,chen2025geometrically}. These constraints minimize structural drift and formatting errors by anchoring the reasoning trajectory within a predefined logical framework. This organizational stability ensures that the model can process complex tasks with a level of consistency that free-form natural language cannot provide.

\subsection{Level 3: Extending the Ability Frontier}
\label{subsec: case_2}

Beyond enhancing current performance (Figure~\ref{fig:performance gap}, left), language representation design pushes the intelligence frontier into domains where natural language (Level 0) is inherently inadequate (Figure~\ref{fig:performance gap}, right). In these complex domains, natural language suffers from expressive poverty, lacking the formal precision to capture high-dimensional scientific logic or intricate physical dynamics \cite{ghallab2004automated,ishay2025llm,lake2017building}. At Level 3, language no longer merely describes a task, it constructs a formal model of the underlying domain itself, encoding constraints, dynamics, and structures that natural language fundamentally cannot operationalize, enabling the precise activation and organization of specialized internal schema \cite{smirnov2024generating,huang2022language,huangcode,raspanti-etal-2025-grammar}. This shift manifests through two complementary directions: Scientific Formalization, which encodes the abstract logic of a domain, and World Modeling, which encodes its physical dynamics and causal structure.

\textbf{Scientific Formalization.} Scientific formalization serves as the primary gateway for expanding intelligence, as it maps complex, rigorous domains into executable and verifiable reasoning spaces that natural language cannot support \cite{cao2025towards,polu2020generative}. 
In formal logic, systems like Seed-Prover provide a logical scaffold by translating natural language into verifiable scripts like Lean \cite{chen2025seed,zhou2025solving}. This enables the model to verify consistency and decompose complex objectives into manageable sub-goals.
Similarly, in materials science, Rep-CodeGen provides a structural syntax for the physical world, allowing the model to optimize material structures under complex symmetry constraints that natural language cannot adequately capture \cite{huangcode}. 
By leveraging these formal structures, language design allows models to operationalize intricate reasoning that natural language cannot adequately capture, allowing LLMs to operationalize specialized knowledge previously beyond their reach.

\textbf{World Modeling.} Furthermore, language representation design expands the intelligence frontier through world modeling, specifically by providing the essential mechanisms to represent physical laws, causal logic, and state evolution \cite{wang2023voyager,liang2022code,ahn2025towards}. While natural language is often too underspecified to capture the constraints of physical reality, designed languages bridge this gap by functioning as a structural interface between high-level intent and actionable execution \cite{huang2022language,valmeekam2022large,shi2025world,choi2025nesyc,choi2025nesypr}. By utilizing formalisms like Planning Domain Definition Language or Linear Temporal Logic, LLMs can construct consistent action domains and verify the logical feasibility of task plans before execution \cite{smirnov2024generating, grigorev2025verifyllm,huang2025limit}. This modeling process ensures that the model’s internal reasoning is grounded in the physical mechanics of the environment rather than mere linguistic probability. Consequently, by enabling the precise representation of physical dynamics, language representation design allows LLMs to navigate complex interactions that remain otherwise indescribable through conventional text.

%% file: main/5_Exp.tex
\subsection{Experimental Evidence}
\label{subsec:exp}
To verify the influence of language representation, we conduct a series of controlled experiments to empirically validate our central position: (i) The performance of LLMs varies substantially across different language representations; (ii) This performance variation arises from the distinct \emph{schemas of reasoning} implicitly induced by each language representation.
These experiments are designed to systematically examine how representation choice shapes the model’s internal inference process.
% , and to what extent such differences affect the model’s ability to generalize from examples to unseen cases.
\subsubsection{Experiment Design}
To verify the impact of different language representations on the performance of LLMs, we chose the logic circuit simulation task. This task has clear formal semantics, can establish a deterministic mapping between various language representations, and maintains strict semantic equivalence among different language representations.

\textbf{Construction of Question Set $Q$.}To evaluate the performance of different language representations, we construct a question set $Q$ with $|Q| = 100$. To enable fully automatic generation of questions and ground-truth answers, we simulate the complete execution of combinational logic circuits. Specifically, we consider a circuit $C = (I, G, O)$, where $I$ is the set of input signals, $G$ is the set of logic gates (AND, OR, NOT, XOR, NAND, NOR), and $O$ is the set of output signals. Each gate $g_i \in G$ has deterministic input connections and logic functions. Given an input assignment $v: I \to \{0,1\}$, the task requires the model to simulate signal propagation and compute how many outputs in $O$ would change if one particular input signal is flipped.

\textbf{Generation Diverse Language Representations $L$.}
To study the effect of linguistic formulation, each question $q \in Q$ is encoded into 15 semantically equivalent yet distinct representations $L_1, \dots, L_{15}$. Strict semantic invariance is ensured by first randomly generating combinational circuit topologies (5--6 inputs, 12--16 gates, max depth 6--8 layers) through iterative random gate-type selection and layer-wise wiring. Random Boolean assignments are then applied to the inputs, followed by topological forward propagation to compute all gate outputs deterministically. Natural-language questions are finally rendered from these fixed circuit instances and query templates, producing the 15 parallel formulations while keeping the underlying logic problem and correct answer identical across all representations.

\textbf{Reflection of Internal Schema $\mathcal{S}$.} 
To quantify how different language representations $L$ reshape the internal schema $\pi_{s}$ of the LLM, we propose two neurally-grounded metrics derived from attention dynamics~\cite{vig2019analyzing,abnar2020quantifying}, building on recent interpretability evidence that attention patterns trace internal computational circuits~\cite{ameisen2025circuit}. We frame these metrics as \textit{quantitative proxies} for two complementary facets of schema rather than as direct measurements of schema itself: \textbf{1) Knowledge Activation Index (KAI)} operationalizes the 
\textit{activation} facet via normalized attention entropy. A higher KAI 
reflects the representation's ability to minimize semantic noise and 
direct the model's internal resources toward task-relevant nodes. \textbf{2) Knowledge Organization Index (KOI)} operationalizes the \textit{organization} facet via inter-layer similarity. A higher KOI reflects the representation's ability to stabilize the internal logical flow and ensure consistent structural propagation across the model's 
depth. 
Both have $\mathcal{O}(LN^2)$ complexity. For more details, please refer to the Appendix \ref{KAI}. Ultimately, we encourage further investigation into diagnostic metrics that can further elucidate the mechanistic nature of model-internal schemas.

\input{main/tabs/qwen32b_new}
\subsubsection{Experimental Results}
Table~\ref{tab:language-format-performance-qwen32b} summarizes the performance of 15 language representations on the logic circuit simulation task, reporting accuracy, average inference time (across 100 queries), and token consumption. The prompt token count reflects the encoding efficiency of each representation, while the completion token count indicates the computational volume required for answer generation. Collectively, these metrics offer a quantitative proxy for how efficiently the model “thinks” under different linguistic constraints.

As illustrated in Table~\ref{tab:language-format-performance-qwen32b}, the choice of representation yields markedly different outcomes. Canonical Boolean Expressions achieve the highest accuracy, followed closely by Lisp Tree Notation. While Natural Language demonstrates broad adaptability, it is inefficient in both task expression and reasoning costs, as evidenced by the significantly higher average token usage in Table~\ref{tab:language-format-performance-qwen32b}. This performance gap stems from the alignment between syntax and task logic: the circuit simulation task inherently relies on directional signal propagation. Boolean expressions and Lisp notation naturally encode the circuit's topological sort, thereby mapping the signal flow directly onto the model's sequential generation process. Conversely, the Graph Adjacency List, despite being explicitly designed to represent structure, lacks this intrinsic causal ordering and consequently yields inferior performance. \textit{This experiment underscores our central claim: language representation design not only boosts model performance but also significantly improves reasoning efficiency.}

The metrics KAI and KOI, reflect how different language designs reshape the model’s internal schema.
Canonical Boolean Expression exhibits high values for both, as shown in row 1 of Table \ref{tab:language-format-performance-qwen32b}: its use of explicit, highly-visible logical operators (e.g., $G1 = \text{AND}(A, B)$) allows the model to precisely anchor its attention on these gates. 
Since gates are defined sequentially (e.g., $G2$ follows $G1$ and directly references it), the model can reuse the same logical path, ensuring information is organized orderly without the need for large-scale structural corrections in deeper layers.

In contrast, Petri Net performs poorly, as shown in row 13 of Table \ref{tab:language-format-performance-qwen32b}, because its non-sequential logic requires frequent jumps to locate distant dependencies (e.g., when defining a transition, its inputs may refer to states located far apart in the text). This forces the model to constantly readjust its processing path, leading to structural oscillation. 
\textit{These findings prove that the internal characteristics of the language such as syntactic features and topological attributes dictate the activation efficiency and organizational stability of the model's internal schema and performance. }

%% file: main/tabs/qwen32b_new.tex
\begin{table*}[t]
\centering
\caption{The performance comparison of different language representations for the logic circuit simulation task by \textbf{Qwen3-32B}. We use color to annotate the best (only when the accuracy is above 80, it will be considered as a candidate; otherwise, efficiency is meaningless). We also report the results by GPT-5-chat in Table~\ref{tab:language-format-performance-gpt-5} listed in Appendix~\ref{appendix:more_exp}.}
\label{tab:language-format-performance-qwen32b}
\small
\resizebox{\linewidth}{!}{
\begin{tabular}{lcccc}
\toprule[1.5pt]
Language Format & Accuracy (\%) & Avg. Time (s) & Avg. Tokens (Prompt / Completion) & KOI / KAI \\
\midrule
\hline
Canonical Boolean Expression & \cellcolor{red!50}$100.00{\pm}0.0$ & \cellcolor{cyan!50}$32.42{\pm}9.5$ & \cellcolor{green!50}$253.4{\pm}43.5$ / $1061.5{\pm}311.9$ & \cellcolor{violet!60}0.407 / 0.370 \\
Layered Execution Plan & \cellcolor{red!40}$95.00{\pm}0.0$ & \cellcolor{cyan!10}$39.62{\pm}5.4$ & \cellcolor{green!20}$469.4{\pm}10.0$ / $1361.0{\pm}185.9$ & 0.376 / 0.357 \\
Lisp Tree Notation & \cellcolor{red!30}$93.75{\pm}2.2$ & \cellcolor{cyan!40}$34.34{\pm}12.8$ & \cellcolor{green!40}$327.2{\pm}76.4$ / $1132.1{\pm}421.3$ & 0.350 / 0.388 \\
Netlist Language & \cellcolor{red!20}$92.50{\pm}1.8$ & $41.19{\pm}8.7$ & \cellcolor{green!10}$535.4{\pm}10.0$ / $1482.2{\pm}311.4$ & \cellcolor{violet!35}0.381 / 0.411 \\
Graph Adjacency Notation & \cellcolor{red!10}$90.00{\pm}2.4$ & $42.36{\pm}8.4$ & $540.7{\pm}22.2$ / $1465.7{\pm}290.7$ & 0.373 / 0.422 \\
Natural Language & $88.75{\pm}4.7$ & \cellcolor{cyan!30}$37.89{\pm}7.5$ & $949.4{\pm}10.0$ / $1359.8{\pm}267.5$ & 0.364 / 0.348 \\
Compact Gate Notation & $83.75{\pm}3.5$ & \cellcolor{cyan!20}$38.12{\pm}6.8$ & \cellcolor{green!30}$371.8{\pm}9.8$ / $1365.6{\pm}244.8$ & 0.358 / 0.412 \\
\midrule
\hline
Dependency Chain Language & $73.75{\pm}3.5$ & $43.85{\pm}7.3$ & $416.9{\pm}13.3$ / $1573.9{\pm}261.6$ & \cellcolor{violet!15}0.378 / 0.391 \\
Reverse Polish Notation & $71.25{\pm}5.9$ & $42.02{\pm}13.9$ & $274.1{\pm}58.3$ / $1348.4{\pm}446.7$ & \cellcolor{violet!45}0.384 / 0.423 \\
Dataflow Language & $45.00{\pm}0.0$ & $50.32{\pm}10.5$ & $468.6{\pm}12.7$ / $1917.1{\pm}401.2$ & 0.368 / 0.360 \\
Matrix Representation & $27.50{\pm}0.0$ & $53.24{\pm}15.0$ & $1988.0{\pm}2.0$ / $1776.0{\pm}501.1$ & \cellcolor{violet!25}0.379 / 0.247 \\
Constraint Satisfaction Format & $21.25{\pm}1.6$ & $52.12{\pm}9.6$ & $647.8{\pm}13.1$ / $1815.0{\pm}335.2$ & 0.374 / 0.390 \\
Partial Truth Table & $21.25{\pm}1.8$ & $33.71{\pm}9.3$ & $464.4{\pm}10.0$ / $1080.2{\pm}296.8$ & 0.391 / 0.392 \\
Petri Net Notation & $13.75{\pm}5.8$ & $38.16{\pm}10.4$ & $1002.2{\pm}31.3$ / $1518.2{\pm}414.7$ & 0.337 / 0.111 \\
Signal Propagation Trace & $12.50{\pm}4.7$ & $33.15{\pm}10.3$ & $423.2{\pm}2.8$ / $1312.6{\pm}408.7$ & 0.369 / 0.331 \\
\bottomrule[1.5pt]
\end{tabular}}
\vspace{-0.6cm}
\end{table*}

%% file: main/6_Alter.tex
\section{Alternative Views}
The scaling law~\citep{kaplan2020scaling} demonstrates that model performance systematically improves as both the number of parameters and the amount of training data increase \cite{radford2021learning,awadalla2024mint,zhang20252,wang2025scaling,dong2025scalable}. Empirical studies have shown a near power-law relationship between scale and performance, suggesting that larger models yield predictable improvements in generalization and reasoning ability \cite{lipredictable,hoffmann2022training,ruan2024observational}. More intriguingly, recent findings~\citep{wei2022emergent} indicate that LLMs exhibit \textit{emergent abilities}, qualitative capabilities that arise abruptly once a model surpasses a certain scale threshold \cite{berti2025emergent,guo2025deepseek}. Examples include in-context learning, compositional reasoning, and multi-step tool use \cite{purohit2025sample}.

\textbf{Alternative View 1}: \textit{The frontier of intelligence can be advanced solely through the scaling of model and data size.}

From this perspective, a dominant hypothesis posits that intelligence is fundamentally an emergent property of scale: given sufficient model size and training data, a system could, in principle, master any task expressible in language
\cite{shukor2025scaling,srivastava2023beyond,muennighoff2023scaling}. This perspective views intelligence as a continuum, an asymptotic outcome of scale rather than a discrete leap in architecture or representation \cite{wu2024inference}.

% If a model were sufficiently large and trained on enough data, it could acquire the capacity to solve any task expressible in language 

\textbf{Alternative View 2}: \textit{An LLM equipped with external tools can, in theory, solve any problem that can be formulated as a language-based task.}

In contrast to pure scaling, this view emphasizes \textit{system composition} over raw capacity. Here, intelligence is seen not as an emergent property of a single monolithic model, but as the result of a collaborative system where the LLM acts as the cognitive core \cite{patil2024gorilla,lu2023chameleon}, orchestrating specialized tools such as search engines \cite{yu2024visrag, wu2025mmsearch}, code interpreters \cite{gao2023pal}, or symbolic planners~\citep{schick2023toolformer,liu2023llm+}. This paradigm extends the LLM’s effective reach without requiring further scaling, enabling problem-solving across modalities, data sources, and reasoning domains \cite{shen2023hugginggpt,qin2023toolllm}.

\textbf{Compared with Our View:} \textit{Language Representation Design as the Next Frontier for Expanding LLM Intelligence}

While scaling and tool augmentation have propelled LLMs to unprecedented capability, we argue that the next qualitative leap in intelligence will arise from \textit{language design}. Both of the previous views rely on existing human-designed languages: natural language, programming languages, or formal symbolic notations, as the substrate of reasoning and communication. However, these languages were not optimized for alignment with the inductive and representational biases of large models under specific tasks \cite{chae2024language}. As a result, there remains a fundamental bottleneck between what the model internally knows and what it \textit{can express} externally through language.

Our position proposes that by \textit{designing new language representations}, structured, interpretable, and learnable by both humans and machines, we can bridge this gap. Such languages could encode reasoning processes, abstractions, and world models more naturally than existing linguistic forms \cite{hu2024chain,xu2024symbol}.
In this framework, intelligence does not merely scale; it \textit{reorganizes} \cite{besta2024graph,swaminathan2023schema}. Language becomes the infrastructure through which higher-order reasoning, collaboration, and interpretability emerge. Thus, while scaling expands the quantitative frontier and tool integration extends the functional boundary, language design reshapes the \textbf{qualitative space} of what is thinkable, learnable, and expressible. The differences between alternative views and our position are also summarized in Table~\ref {tab:paradigm_comparison} in Appendix.

%% file: main/7_Open_probs.tex
\section{Open Problems and Roadmap}
\subsection{Open Problems}
Despite growing empirical evidence that carefully designed language representations can substantially improve LLM performance, a principled understanding of how to design, adapt, and theoretically ground such representations remains largely open \cite{chae2024language,wang2023grammar,barradas2025combining}. We highlight several key open problems that define this emerging research frontier.

\textbf{Q1: How can we systematically design effective language representations for a given problem?}

Given a task or problem class, how can one algorithmically or methodologically construct a language representation that induces an effective internal schema in an LLM \cite{chen2025schema}? Open challenges include identifying which structural elements (e.g., symbolic constraints, intermediate variables, modular decomposition, or control tokens) are essential \cite{dhanraj2025improving,hu2024chain,besta2024graph}, how task properties should guide representation choices, and whether there exist general design principles that transfer across domains \cite{dong2024survey}. A central question is whether language representation design can be formalized as an optimization problem over a space of linguistic structures, rather than relying on ad-hoc engineering.

\textbf{Q2: How should LLMs be adapted or aligned to operate optimally under specific language representations?}

Even when an effective language representation is identified, current LLMs may not fully exploit it due to mismatches between pretraining distributions and task-specific linguistic constructs. This raises questions about how models should be adapted through instruction tuning, representation-aware finetuning, in-context learning strategies, or architectural biases to better internalize and utilize new schema \cite{pan2023logic,zelikman2024quiet}. More broadly, what does it mean for an LLM to be \emph{representation-aware}, and how can models flexibly switch or generalize across multiple language representations without catastrophic interference \cite{yang2024buffer,liu2025conditions}?

\textbf{Q3: What is the theoretical relationship between language representations and internal representations of LLMs?}

A fundamental open problem is to develop a theoretical understanding of how different language representations shape the internal activations, attention patterns, and feature compositions of LLMs~\cite{nikolaou2025language,kumon2025analyzing}. In particular, it remains unclear how variations in linguistic structure—such as syntax, abstraction level, or compositional primitives—are reflected in the emergence and organization of internal schemas. Under what conditions do specific representations promote more compositional, or generalizable internal features, and how do these effects interact with depth, scale, and training dynamics? Addressing these questions requires bridging language representation design with theories of representation learning, mechanistic interpretability \cite{templeton2024scaling}, and optimization dynamics. Such a synthesis may ultimately enable predictive theories specifying when and why particular linguistic forms reliably induce superior reasoning, and generalization behavior \cite{wang2025logical}.

%% file: main/8_Action.tex
\subsection{Roadmap}

To overcome the natural language bottleneck, we propose a concise, four-phase roadmap: \textbf{(i) Phase 1: Systematic Design}: Treat language representation design as a core modeling decision, formalizing it as an optimization problem to identify structures that best induce effective reasoning schemas. \textbf{(ii) Phase 2: Representation-Aware Alignment}: Adapt LLMs through representation-aware finetuning to handle new, optimal language designs without suffering from catastrophic interference.
\textbf{(iii) Phase 3: Mechanistic Foundations}: Bridge representation design with mechanistic interpretability to theoretically understand how linguistic variations shape internal activations and attention patterns. \textbf{(iv) Phase 4: Autonomous Construction}: Push the ultimate capability frontier by empowering AI to construct its own optimal languages directly from world interaction, bypassing human linguistic bottlenecks entirely.

%% file: main/9_Conclu.tex
\section{Conclusion}

In this paper, we argued that schema shaping through language representation design constitutes a critical next frontier for expanding the intelligence of LLMs. To substantiate the central role of language representations, we (i) first surveyed recent empirical practices and emerging methodologies that achieve substantial performance gains through deliberate representation design; (ii) then conducted controlled experiments demonstrating that both task performance and internal feature activations vary systematically across different language representations of the same task. Building on these results, we identified three open research questions: (i) how to systematically design effective language representations, (ii) how to adapt or align LLMs to operate optimally under specific representations, and (iii) how language representations relate to internal Transformer representations. We concluded by calling action for increased attention to language representation design and for deeper mechanistic insights into how language structures schemas in LLMs.

%% file: appendix/appendix_more_bg.tex
\section{Proof of Proposition \ref{prop:bounds_on_prediction_error}}

\begin{proposition}[Bounds on prediction error]
Let $\mathcal{I}_{\pi_a}(s)$ be the Fisher Information Matrix of the action mapping $\pi_a$ at schema $s$. For any distance $d$ on $\mathcal{F}$ defined as the squared Fisher-Rao distance in the action space, the prediction error satisfies:
\[
\frac{\sigma^{2}_{\min}}{2}\,\mathrm{SM}(L) \;\le\; d(f,\hat{f}_{L}) \;\le\; \frac{\sigma^{2}_{\max}}{2}\,\mathrm{SM}(L),
\]
where $\hat{f}_L\triangleq L^{-1}\pi_a\pi_sL$, $\sigma^{2}_{\min} = \inf_s \lambda_{\min}(\mathcal{I}_{\pi_a}(s))$, $\sigma^{2}_{\max} = \sup_s \lambda_{\max}(\mathcal{I}_{\pi_a}(s))$, and $\mathrm{SM}(L) = D_{\mathrm{KL}}(s_f^L \| s_\pi^L)$.
\end{proposition}

\begin{proof}
The proof relies on the principles of Information Geometry and the Riemannian structure of the statistical manifold $\mathcal{S}$ of schema representations.

\paragraph{Step 1: Definition of the Output Distance.}
The prediction error $d(f, \hat{f}_L)$ is defined as the squared Fisher-Rao distance between the target action distribution $p(a|s_f^L)$ and the language-induced action distribution $p(a|s_\pi^L)$. Let $\mathcal{D}_{\mathcal{A}}^2$ denote the squared Fisher-Rao distance in the action manifold $\mathcal{A}$:
\begin{equation}
d(f, \hat{f}_L) = \mathcal{D}_{\mathcal{A}}^2(\pi_a(s_f^L), \pi_a(s_\pi^L))
\end{equation}

\paragraph{Step 2: Pullback Metric and the Energy Functional.}
The mapping $\pi_a: \mathcal{S} \to \mathcal{A}$ allows us to pull back the metric from the action space to the representation space. The squared distance is defined as the infimum of the energy functional over all smooth paths $\gamma(t)$ such that $\gamma(0) = s_\pi^L$ and $\gamma(1) = s_f^L$:
\begin{equation}
d(f, \hat{f}_L) = \inf_{\gamma} \int_0^1 \left\| \frac{d}{dt} \pi_a(\gamma(t)) \right\|_{\mathcal{A}}^2 \, dt
\end{equation}
By the chain rule, the velocity in the action space is related to the velocity in the representation space via the Jacobian of $\pi_a$. The local metric is given by the Fisher Information Matrix $\mathcal{I}_{\pi_a}(s)$:
\begin{equation}
\left\| \frac{d}{dt} \pi_a(\gamma(t)) \right\|_{\mathcal{A}}^2 = \dot{\gamma}(t)^\top \mathcal{I}_{\pi_a}(\gamma(t)) \dot{\gamma}(t)
\end{equation}

\paragraph{Step 3: Spectral Bound on the Quadratic Form.}
We utilize the spectral properties of $\mathcal{I}_{\pi_a}$. For any schema $s \in \mathcal{S}$ and any tangent vector $v \in T_s \mathcal{S}$, the quadratic form is bounded by the extremal eigenvalues:
\begin{equation}
\lambda_{\min}(\mathcal{I}_{\pi_a}(s)) \|v\|^2 \le v^\top \mathcal{I}_{\pi_a}(s) v \le \lambda_{\max}(\mathcal{I}_{\pi_a}(s)) \|v\|^2
\end{equation}
Applying the definitions $\sigma^2_{\min} = \inf_s \lambda_{\min}$ and $\sigma^2_{\max} = \sup_s \lambda_{\max}$, we obtain:
\begin{equation}
\sigma^2_{\min} \|\dot{\gamma}(t)\|^2 \le \dot{\gamma}(t)^\top \mathcal{I}_{\pi_a}(\gamma(t)) \dot{\gamma}(t) \le \sigma^2_{\max} \|\dot{\gamma}(t)\|^2
\end{equation}

\paragraph{Step 4: Integration along the Geodesic.}
Integrating the inequalities from Step 3 along the path $\gamma$ from $t=0$ to $t=1$:
\begin{equation}
\sigma^2_{\min} \int_0^1 \|\dot{\gamma}(t)\|^2 \, dt \le \int_0^1 \dot{\gamma}(t)^\top \mathcal{I}_{\pi_a}(\gamma(t)) \dot{\gamma}(t) \, dt \le \sigma^2_{\max} \int_0^1 \|\dot{\gamma}(t)\|^2 \, dt
\end{equation}
The integral $\int_0^1 \|\dot{\gamma}(t)\|^2 \, dt$ represents the squared Riemannian distance in the representation space $\mathcal{D}_R^2(s_f^L, s_\pi^L)$. Thus:
\begin{equation}
\sigma^2_{\min} \mathcal{D}_R^2(s_f^L, s_\pi^L) \le d(f, \hat{f}_L) \le \sigma^2_{\max} \mathcal{D}_R^2(s_f^L, s_\pi^L)
\end{equation}

\paragraph{Step 5: Relation to Schema-Mismatch (KL Divergence).}
Information geometry establishes that the Kullback-Leibler divergence is the canonical divergence associated with the Fisher Information Metric. Locally, for a metric $g$, the KL divergence is given by:
\begin{equation}
D_{\mathrm{KL}}(s_f^L \| s_\pi^L) = \frac{1}{2} g_{ij} \Delta s^i \Delta s^j + \mathcal{O}(\Delta s^3)
\end{equation}
In the global Riemannian context, the identity between the squared Fisher-Rao distance and KL divergence is:
\begin{equation}
\mathrm{SM}(L) = D_{\mathrm{KL}}(s_f^L \| s_\pi^L) = \frac{1}{2} \mathcal{D}_R^2(s_f^L, s_\pi^L)
\end{equation}
Rearranging gives $\mathcal{D}_R^2(s_f^L, s_\pi^L) = 2 \, \mathrm{SM}(L)$.

\paragraph{Step 6: Final Substitution.}
Substituting $2 \, \mathrm{SM}(L)$ for $\mathcal{D}_R^2(s_f^L, s_\pi^L)$ into the inequalities from Step 4:
\begin{equation}
\sigma^2_{\min} (2 \, \mathrm{SM}(L)) \le d(f, \hat{f}_L) \le \sigma^2_{\max} (2 \, \mathrm{SM}(L))
\end{equation}
Simplifying the coefficients leads to the final result:
\begin{equation}
\frac{\sigma^{2}_{\min}}{2} \mathrm{SM}(L) \le d(f, \hat{f}_L) \le \frac{\sigma^{2}_{\max}}{2} \mathrm{SM}(L)
\end{equation}
\end{proof}

%% file: appendix/appendix_related.tex
\section{More Related Information about Our Position}

Table~\ref{tab:paradigm_comparison} presents a comparative overview of alternative perspectives, offering an intuitive illustration of the limitations and frontier potential associated with each paradigm. This comparison further underscores the significance and broad research scope of language representation design in the context of current LLMs. We also provide a visualization of language design in Figure \ref{fig:visualization}.

\input{main/tabs/alter_tab}

Additionally, Figure~\ref{fig:example} illustrates how language representation shapes the underlying schema. As shown, different language representations induce distinct schemas, and only a well-designed representation can activate the appropriate knowledge regions and organize them into a correct, task-specific schema. In the depicted example, the structured representation LL
L activates knowledge spanning geo-spatial, logistics, and environmental domains, organizing it into a coherent reasoning graph (i.e., the schema) that yields the optimal itinerary. In contrast, a poorly designed representation would either activate irrelevant regions or fail to organize them coherently, ultimately leading to degraded performance.

\begin{figure}[h]
    \centering
    \includegraphics[width=1.0\linewidth]{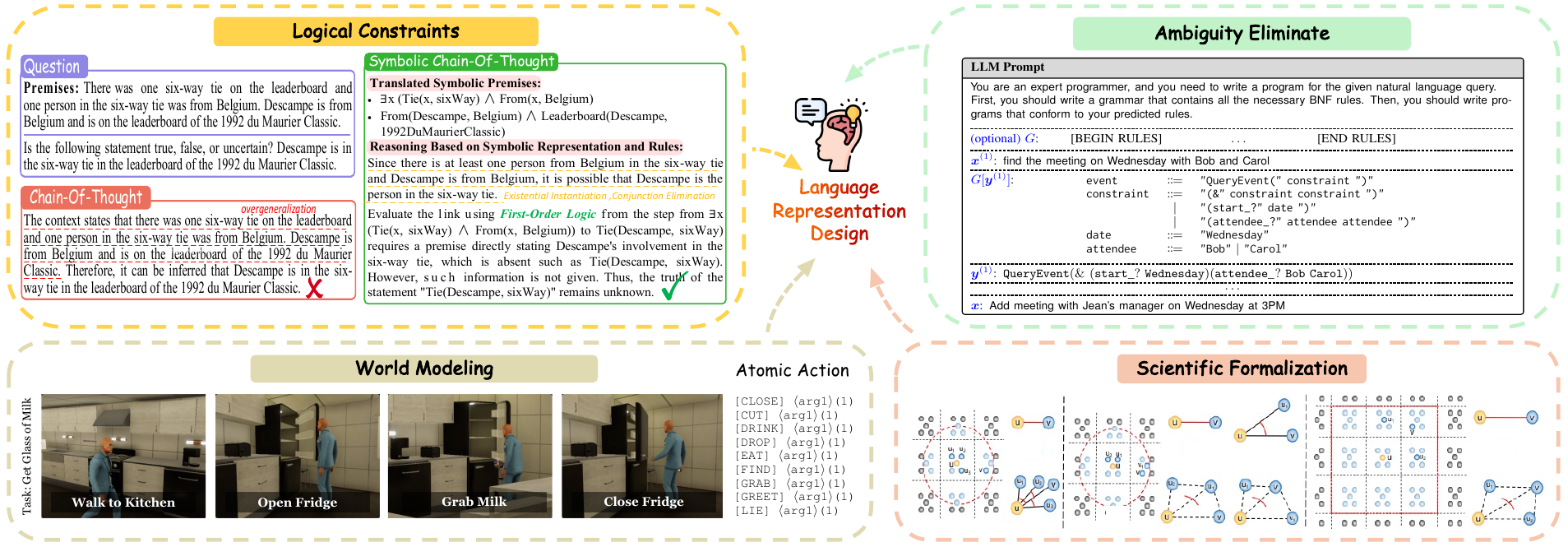}
    % \vspace{-0.4cm}
    \caption{\textbf{Practical paradigms of language representation design across various domains.} This figure showcases specific examples of language representation design, including logical constraint, ambiguity elimination, world modeling, and scientific formalization languages for science and embodied agents that enhance model performance by inducing more effective internal thinking schemas. Figures adapted from \cite{wang2023grammar,xu2024faithful,huang2022language,huangcode}.}
    \label{fig:visualization}
    % \vspace{-0.4cm}
\end{figure}

\begin{figure}[t]
    \centering
    \includegraphics[width=0.7\linewidth]{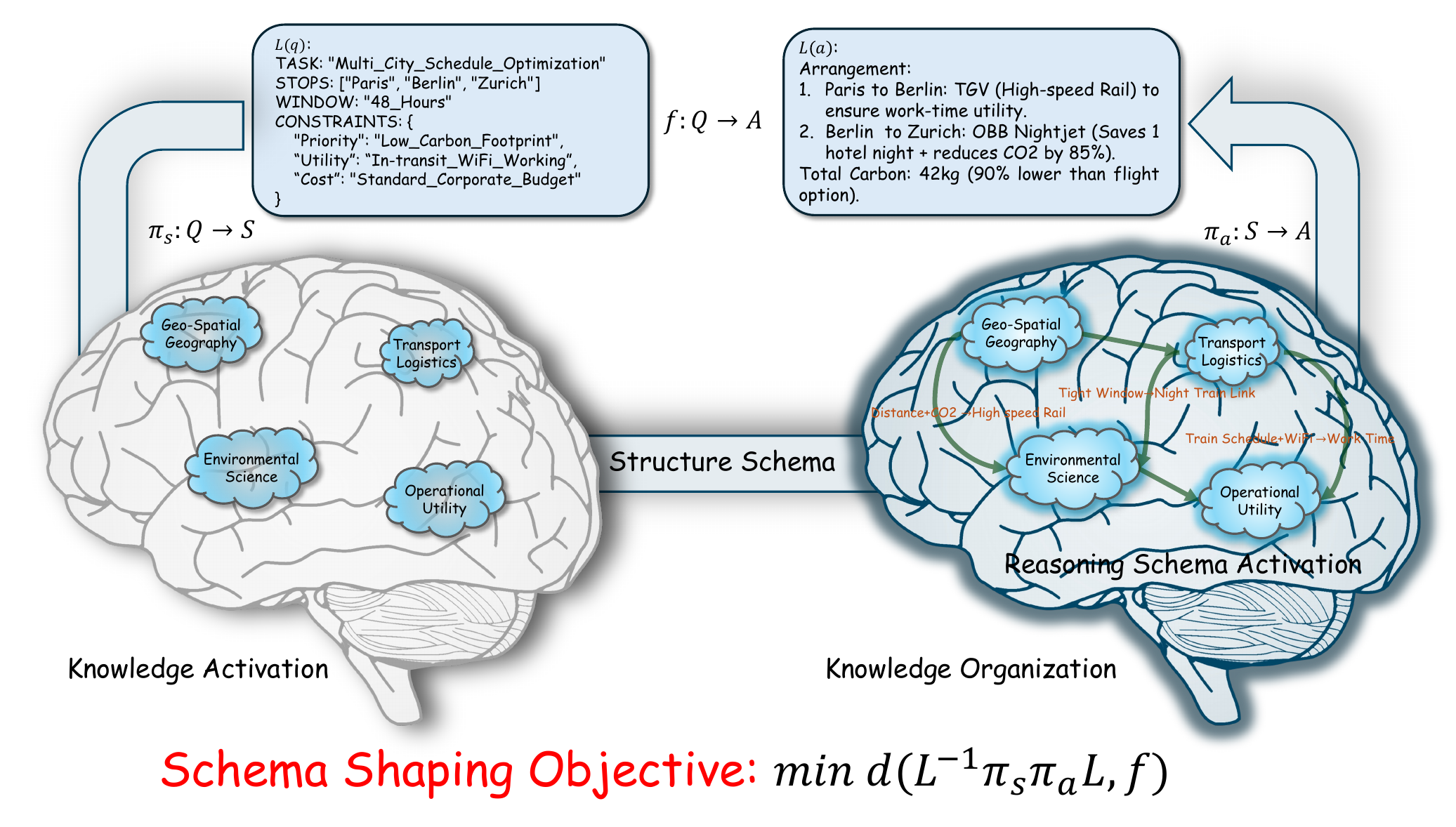}
    \caption{\textbf{Conceptual overview of shaping schema via language representation design.} We posit that the intelligence of Large Language Models (LLMs) can be expanded not just by scaling, but by designing language representations ($L$) that deliberately induce optimal internal schemas. As conceptualized in our position, a complex task—such as a 48-hour multi-city schedule constrained by punctuality, low-carbon preferences, and in-transit Wi-Fi needs—is processed through a two-stage evolution. The designed representation $L$ first activated the isolated knowledge stored in LLMs across domains like geo-spatial geography and operational utility. Acting as an operational framework, $L$ then guides the model to restructure these nodes into a cohesive, task-specific logic graph, or ``Schema''. }
    \label{fig:example}
    % \vspace{-0.6cm}
\end{figure}

%% file: main/tabs/alter_tab.tex
\newcommand{\bulletitem}{\raisebox{0.1ex}{\tiny$\bullet$}\ }

\definecolor{tableblue}{HTML}{F0F7FF}
\definecolor{headergray}{HTML}{F2F2F2}

\begin{table*}[t]
\centering
\caption{Comparative analysis of intelligence expansion strategies: Highlighting the qualitative potential of Language Representation Design as our core position.}
\renewcommand{\arraystretch}{2.0} 
\small
\begin{tabularx}{\textwidth}{l >{\raggedright\arraybackslash}p{4cm} X X}
\toprule
\rowcolor{headergray}
\textbf{Paradigm} & \textbf{Core Mechanism} & \textbf{Key Limitation} & \textbf{Frontier Potential} \\ \midrule
\textbf{Scaling Law} & 
Increase model parameters and data size to achieve emergent capabilities through statistical generalization. & 
\bulletitem Computationally expensive \newline 
\bulletitem Diminishing returns at large scales \newline 
\bulletitem Data quality and physical resource constraints & 
Expands quantitative performance frontier; reveals emergent phenomena that inform theoretical understanding of intelligence. \\ \midrule
\textbf{LLM + Tools} & 
Combine language models with external tools (e.g., search, code execution, symbolic reasoning) via natural language interfaces. & 
\bulletitem Dependence on expressiveness and precision of natural language \newline 
\bulletitem Coordination errors between the model and the tools & 
Extends functional capability without requiring massive scaling; enables grounded, interactive, and multi-modal reasoning. \\ \midrule
\rowcolor{tableblue}
\textbf{\makecell[l]{Language Design \\ (Our Position)}} & 
Develop new intermediate or hybrid languages optimized for machine reasoning, learning, and human interpretability. & 
\bulletitem Needs systematic co-design of syntax, semantics, and learning dynamics \newline 
\bulletitem Challenges in adoption and standardization & 
\textbf{Qualitative Leap:} Bridges neural and symbolic intelligence, enhances transparency, and redefines the space of expressible cognition. \\ \bottomrule
\end{tabularx}
\label{tab:paradigm_comparison}
\end{table*}

%% file: appendix/appendix_more_exp.tex
\section{More Detailed Experimental Setup}
\label{appendix:more_exp}
\subsection{Detail about Different Language Format for Logic Circuit Simulation Task}
\label{appendix:lang_detail}

We convert each circuit instance into 15 different linguistic representations, ranging from natural language to specialized notations designed for circuit description. Below we describe each language representation in detail.

\begin{itemize}
    \item \textbf{Natural Language} \\
    A natural language description that extensively describes each gate's function, input connections, and the processing flow. This representation is designed to highlight the limitations of natural language for structured tasks, as it requires significant redundancy to maintain clarity and produces the longest prompts among all representations. An example is shown in Figure~\ref{fig:logic_circuit_nl}.
    
    \item \textbf{Netlist Language} \\
    A hardware description format inspired by Verilog/VHDL netlists, using module-port-wire syntax common in electronic design automation. This representation explicitly declares inputs, wires, gates, and outputs in a structured format familiar to hardware engineers. An example is shown in Figure~\ref{fig:logic_circuit_netlist}.
    
    \item \textbf{Graph Adjacency Notation} \\
    A graph-theoretic representation that explicitly lists all nodes (inputs, gates, outputs) and edges (signal connections) in adjacency list form. This format emphasizes the circuit's graph structure and is analogous to representations used in graph neural networks. An example is shown in Figure~\ref{fig:logic_circuit_graph}.
    
    \item \textbf{Matrix Representation} \\
    An adjacency matrix representation where circuit nodes are indexed and connections are encoded in a binary matrix. This format provides a dense, numerical encoding suitable for linear algebraic operations and is commonly used in network analysis. An example is shown in Figure~\ref{fig:logic_circuit_matrix}.
    
    \item \textbf{Lisp Tree Notation} \\
    A nested S-expression format that recursively expands each output as a tree of gate operations. This representation naturally captures the hierarchical structure of signal dependencies and is reminiscent of abstract syntax trees in programming language theory. An example is shown in Figure~\ref{fig:logic_circuit_lisp}.
    
    \item \textbf{Dataflow Language} \\
    A streaming computation model that organizes gates into pipeline stages by layer, emphasizing the temporal flow of signals through the circuit. This representation uses stream variables and pipeline notation to express computation as a sequence of transformations. An example is shown in Figure~\ref{fig:logic_circuit_dataflow}.
    
    \item \textbf{Partial Truth Table} \\
    A tabular format that shows input values and traces gate evaluations in execution order, similar to truth tables in digital logic textbooks. This representation provides an explicit step-by-step evaluation trace, making the computation process highly transparent. An example is shown in Figure~\ref{fig:logic_circuit_partial}.
    
    \item \textbf{Compact Gate Notation (CGN)} \\
    An ultra-compact notation using shorthand syntax \texttt{[GateID: Type](Input1, Input2,...)} designed to minimize token count while preserving all structural information. Gate types are abbreviated to single characters (e.g., A for AND, O for OR), and whitespace is minimized. An example is shown in Figure~\ref{fig:logic_circuit_GCN}.
    
    \item \textbf{Reverse Polish Notation (RPN)} \\
    A postfix expression format that eliminates parentheses and nested structures by placing operators after operands. This representation is inspired by stack-based computation models and removes the need for precedence rules, potentially simplifying parsing for language models. An example is shown in Figure~\ref{fig:logic_circuit_PRN}.
    
    \item \textbf{Dependency Chain Language (DCL)} \\
    A format that explicitly encodes signal dependencies using logical operators and dependency arrows (e.g., $\leftarrow$, $\land$, $\lor$). This representation emphasizes the causal relationships between signals and uses mathematical notation familiar from formal logic. An example is shown in Figure~\ref{fig:logic_circuit_DCL}.
    
    \item \textbf{Layered Execution Plan (LEP)} \\
    A stratified format that groups gates by computational layer, explicitly showing the temporal stages of circuit evaluation. Each layer processes signals from previous layers, making the execution order completely transparent and avoiding any ambiguity in the computation sequence. An example is shown in Figure~\ref{fig:logic_circuit_LEP}.
    
    \item \textbf{Signal Propagation Trace (SPT)} \\
    A temporal trace format that simulates circuit execution step-by-step, showing signal values at each time step as they propagate through layers. This representation provides the most explicit computational trajectory, essentially giving the model a worked example of the signal flow. An example is shown in Figure~\ref{fig:logic_circuit_SPT}.
    
    \item \textbf{Constraint Satisfaction Format (CSF)} \\
    A declarative format that encodes the circuit as a constraint satisfaction problem, listing all signal variables and their logical constraints. This representation is inspired by SAT solvers and emphasizes the circuit as a system of Boolean equations to be satisfied. An example is shown in Figure~\ref{fig:logic_circuit_CSF}.
    
    \item \textbf{Canonical Boolean Expression (CBE)} \\
    A format that recursively expands each output into a complete Boolean expression in terms of input variables, using standard logical operators ($\land$, $\lor$, $\neg$, $\oplus$). This representation eliminates all intermediate signals and expresses outputs as mathematical formulas. An example is shown in Figure~\ref{fig:logic_circuit_CBE}.
    
    \item \textbf{Petri Net Notation (PNN)} \\
    A representation based on Petri net formalism, using places (signal states) and transitions (logic gates) with token-based semantics. This format models the circuit as a concurrent system where tokens flow through places according to firing rules, providing a formal concurrency-theoretic view. An example is shown in Figure~\ref{fig:logic_circuit_PNN}.
\end{itemize}

\textbf{Representation Design Rationale: }The first seven representations (from Natural Language to Partial Truth Table) cover common ways to describe circuits, ranging from informal natural language to standard technical notations. The latter eight representations (CGN to PNN) are specifically designed to optimize different aspects of circuit description: compactness (CGN), parsing simplicity (RPN), logical clarity (DCL, CSF, CBE), execution transparency (LEP, SPT), and formal semantics (PNN). This diverse set allows us to systematically investigate which linguistic features correlate with model performance on circuit reasoning tasks.

\subsection{Formal Definition of KAI and KOI}
\label{KAI}
\begin{figure*}[t]
    \centering
    \includegraphics[width=1.0\linewidth]{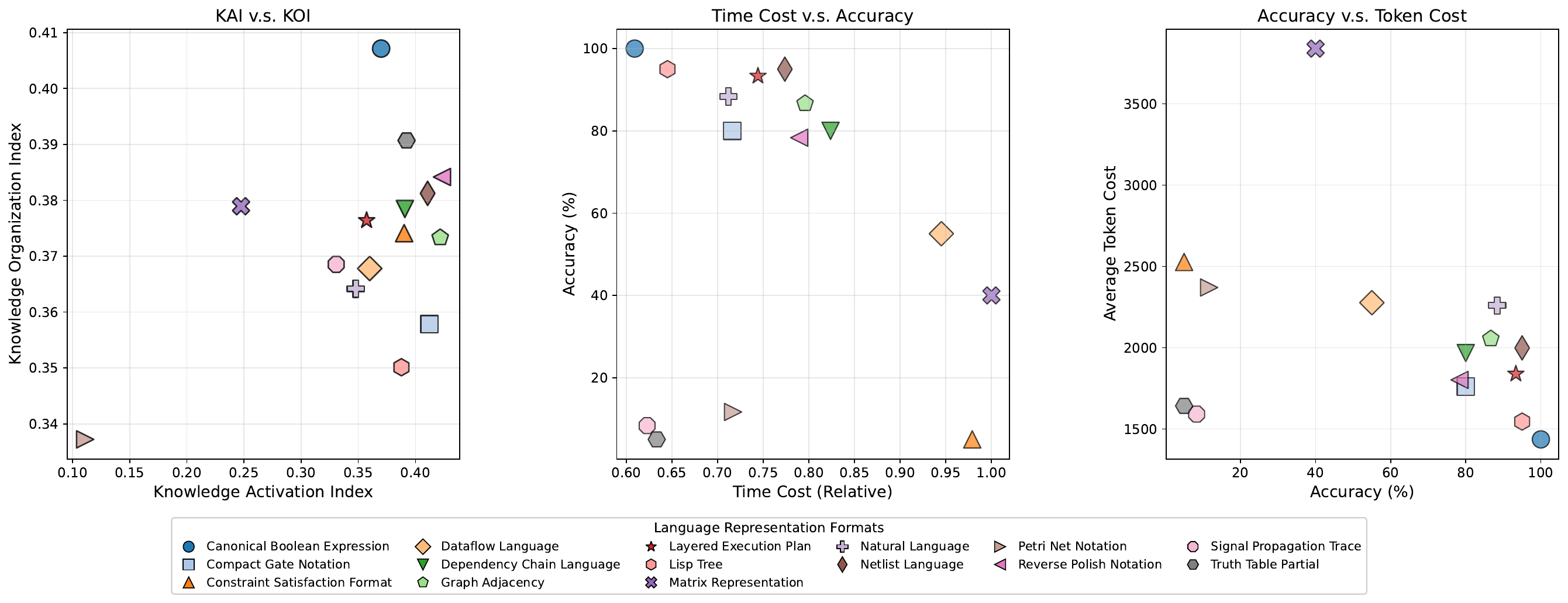}
    \caption{Multi-dimensional evaluation of language representation formats across internal dynamics, reasoning performance, and computational efficiency. The left panel illustrates the relationship between KAI and KOI, reflecting how diverse notations induce focused schema activation and structural consistency. In the middle, the efficiency-performance frontier is depicted through time cost versus accuracy, while the rightmost plot examines the relation between reasoning accuracy and the associated token consumption for each language.}
    \label{fig:main_exp}
\end{figure*}
\textbf{Knowledge Activation Index (KAI).} The KAI quantifies the model's transition into a focused computational state, measuring the effective signal-to-noise ratio of attention distributions relative to critical task components. To enhance sensitivity to high-fidelity activations, the power factor is applied during aggregation.

For a sequence of length $N$ and a model with $L$ layers, the KAI is calculated as follows:

\begin{itemize}
    \item \textbf{Normalized Attention Entropy ($H_{norm}$):} For each layer $l$, we compute the average row-wise entropy of the attention matrix $A^{(l)}$, normalized by $\ln(N)$:
    \begin{equation}
        H_{norm}^{(l)} = \frac{1}{N \ln(N)} \sum_{i=1}^{N} \left( -\sum_{j=1}^{N} A_{i,j}^{(l)} \ln(A_{i,j}^{(l)} + \epsilon) \right)
    \end{equation}
    where $\epsilon = 10^{-10}$. The term $(1 - H_{norm}^{(l)})$ represents the \textit{Focus Purity}.

    \item \textbf{Logical Focus Rate ($F^{(l)}$):} We measure the proportion of attention mass allocated to critical logical nodes $\mathcal{K}$:
    \begin{equation}
        F^{(l)} = \frac{1}{N} \sum_{i=1}^{N} \sum_{j \in \mathcal{K}} A_{i,j}^{(l)}
    \end{equation}

    \item \textbf{Non-linear Aggregation:} The final KAI is defined as the mean of the stretched product of purity and focus across all layers, using a power factor $p$:
    \begin{equation}
        KAI = \frac{1}{L} \sum_{l=1}^{L} \left( (1 - H_{norm}^{(l)}) \cdot F^{(l)} \right)^p
    \end{equation}
\end{itemize}

A high KAI indicates that the language representation effectively suppresses semantic noise and concentrates the model’s internal computational budget on task-relevant logic nodes.

\input{main/tabs/gpt5_new}
\textbf{Knowledge Organization Index (KOI).} The KOI measures the consistency and convergence of the internal reasoning structure across the model's depth. To expose structural drift masked by the high residual similarity common in Transformers, we apply a higher-order power transformation.

The KOI tracks the structural evolution of attention patterns:

\begin{enumerate}
    \item \textbf{Structural Vectorization:} The attention matrix $A^{(l)}$ at layer $l$ is flattened into a high-dimensional vector $V^{(l)} \in \mathbb{R}^{N^2}$, representing the layer's organizational snapshot.
    
    \item \textbf{Inter-layer Structural Stability ($S^{(l)}$):} We compute the cosine similarity between adjacent layers:
    \begin{equation}
        S^{(l)} = \frac{V^{(l)} \cdot V^{(l-1)}}{\|V^{(l)}\| \|V^{(l-1)}\|}
    \end{equation}

    \item \textbf{Sensitivity-Enhanced Aggregation:} The final KOI is the mean of the similarity values stretched by a power factor $q$:
    \begin{equation}
        KOI = \frac{1}{L-1} \sum_{l=2}^{L} (S^{(l)})^q
    \end{equation}
\end{enumerate}

A high KOI suggests the emergence of a consistent \textit{Logical Flow}. It implies that the language provides a clear structural blueprint, allowing the model to propagate logical dependencies across layers with minimal structural correction or "cognitive oscillation."

\textbf{Complexity and Robustness of KAI and KOI.}
We provide a detailed analysis of the computational complexity and 
robustness of the two diagnostic metrics.

\begin{enumerate}
    \item \textbf{Computational Complexity.} Both metrics operate directly on the attention matrices that are already 
computed during a single forward pass of the model, requiring no additional 
forward or backward computation. 
KAI and KOI operate on attention matrices already computed during inference. Both yield $O(LN^2)$ for $L$ layers and sequence length $N$, negligible relative to inference cost.

    \item \textbf{Robustness}.
The building blocks of both metrics are grounded in established 
interpretability tools:
KAI builds on attention entropy, validated by~\cite{zhai2023stabilizing} as a 
reliable Transformer diagnostic across vision, translation, speech, and 
language modeling.
KOI uses inter-layer cosine similarity, a lightweight CKA variant~\cite{kornblith2019similarity}. Recent work~\cite{jiang2024tracing} confirms sample-wise cosine similarity aligns closely with CKA for layer-wise analysis.

Our data supports diagnostic value at extremes (e.g., KAI: $0.111$ for Petri 
Net at $11.7\%$ vs.\ $0.422$ for Graph Adjacency at $86.7\%$) in Table \ref{tab:language-format-performance-gpt-5}, though the 
relationship is not strictly monotonic in the mid-range due to the 
multidimensional nature of schema quality. KAI is most robust when explicit 
logical nodes exist; KOI is most indicative for deep sequential reasoning. 
We therefore position both as exploratory proxies.

\end{enumerate}

\subsection{Experimental Results}

In order to test the performance of different models when dealing with different language representations, we respectively used the GPT-5-chat (Table~\ref{tab:language-format-performance-gpt-5}) and Qwen3-32b (Table~\ref{tab:language-format-performance-qwen32b}) models to test the performance of the logic circuit simulation task. In addition to the detailed tabular results, we provide a multi-dimensional comparison in Figure \ref{fig:main_exp} to visualize the relationships between internal dynamics, reasoning accuracy, and computational efficiency.

As presented in Table~\ref{tab:language-format-performance-gpt-5}, four language representations achieved a perfect accuracy rate of $100\%$. Consistent with the findings in Table~\ref{tab:language-format-performance-qwen32b}, Canonical Boolean Expressions (CBE) emerged as the optimal representation, simultaneously demonstrating the highest accuracy and superior computational efficiency. The performance hierarchy of these languages remains remarkably robust across different model architectures. Conversely, representations such as Petri Net Notation, Constraint Satisfaction Format, and Partial Truth Table consistently yielded suboptimal results on both Qwen and GPT models.

Notably, while GPT achieved superior accuracy with Graph Adjacency Notation compared to Qwen3-32B, it exhibited a marked reduction in efficiency. This suggests that while the GPT model possesses stronger structural reasoning capabilities, enabling it to decipher suboptimal representations, it incurs a higher computational cost to do so. Furthermore, although Natural Language achieved high accuracy, its completion token consumption was nearly $4\times$ that of the optimal language. This disparity underscores that accuracy alone is an insufficient metric; internal reasoning efficiency is equally critical. The consistency of these findings across models reinforces the universality of our central claim: \textit{the design of the language representation is a fundamental determinant of both reasoning performance and computational cost.}

  \subsection{Representations Induce Disjoint Internal Geometries}
  \label{sec:tsne-evidence}
  \begin{figure}[t]
      \centering
      \includegraphics[width=0.7\linewidth]{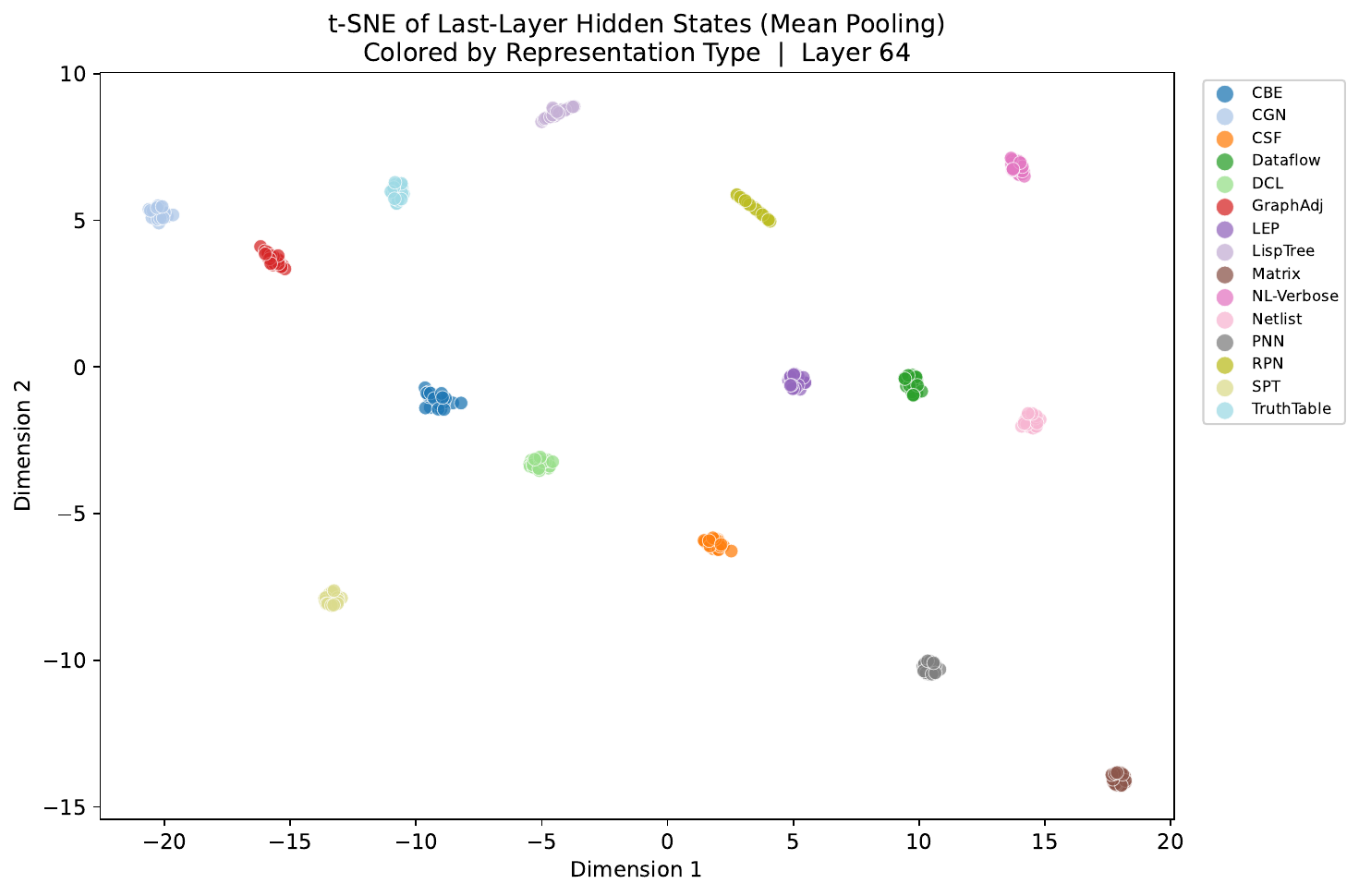}
      \caption{\textbf{t-SNE of last-layer hidden states (mean pooling), colored by representation type, layer 64 of Qwen3-32B.}
      Each point is a single circuit problem; colors denote the $15$ surface representations.
      Despite the underlying logical content being \emph{identical} across formats, the model's
      final-layer states form sharply disjoint clusters, one per representation. The silhouette score of $0.93$
      and the $96.8\%$ between-format variance ratio confirm that representational identity, rather
      than logical semantics, dominates the geometry of the residual stream up to the final layer.}
      \label{fig:tsne-meanpool}
  \end{figure}
  To test whether the performance gap across surface formats reflects a genuine
  \emph{representational bottleneck}, rather than a difference in the underlying
  knowledge accessible to the model. We probe the hidden states of
  \textsc{Qwen3-32B} on a fixed set of logic-circuit problems rendered in all  $15$ representations. For every (problem, representation) pair we extract the mean-pooled hidden state at the final layer (layer~$64$) and project the resulting cloud to two dimensions with t-SNE.

  Figure~\ref{fig:tsne-meanpool} reveals a striking pattern. Although every
  point in the plot encodes the \emph{same logical question}, the projection
  splits cleanly into $15$ tight, well-separated clusters, one per surface
  format. We quantify this separation in two complementary ways: (i) \textbf{Silhouette score $=0.93$.} Values close to $1$ indicate that
      each point lies far closer to others of the same representation than to any point of a different representation. For reference, $0.5$ already
      denotes ``reasonable'' clustering; $0.93$ is essentially perfect
      separation.
(ii) \textbf{Between-cluster variance ratio $=96.8\%$.} Of the total
      variance in the layer-$64$ hidden states, only $3.2\%$ is attributable to
      differences in logical content; the remaining $96.8\%$ is explained by
      the surface format alone.

  This pattern is \emph{not} confined to the output layer. Table~\ref{tab:silhouette}
  reports the silhouette score and variance ratio at five sampled layers
  ($0,16,32,48,64$); both metrics remain $\geq 0.82$ throughout the network,
  with strong separation already present at the embedding layer
  ($0.94$ at layer~$0$) and persisting all the way to the final layer
  ($0.93$ at layer~$64$). In other words, the model commits to a
  representation-specific subspace upon ingestion and \emph{never merges these
  subspaces}, even at the final layer that produces the answer.

  \begin{table}[t]
      \centering
            \caption{Cluster separability of hidden states across $15$ representations
      at five sampled layers of \textsc{Qwen3-32B}. High values at every depth
      indicate that representation-specific subspaces persist throughout the
      network.}
      \begin{tabular}{lccccc}
          \toprule
          Layer & $0$ & $16$ & $32$ & $48$ & $64$ \\
          \midrule
          Silhouette (mean pool) & $0.944$ & $0.953$ & $0.943$ & $0.821$ & $0.931$ \\
          Variance ratio        & $0.978$ & $0.945$ & $0.950$ & $0.879$ & $0.968$ \\
          \bottomrule
      \end{tabular}
      \label{tab:silhouette}
  \end{table}

  \paragraph{Implication.}
  A model that had truly internalized a representation-invariant notion of
  ``logic circuit'' would map semantically equivalent problems to nearby points
  regardless of surface form; clusters in Figure~\ref{fig:tsne-meanpool} would
  overlap, not separate. The opposite is observed: the geometry is dominated by \emph{how} the question is written, not \emph{what} it asks. Combined with the order-of-magnitude accuracy gap between best and worst formats (Tables~\ref{tab:language-format-performance-qwen32b} and ~\ref{tab:language-format-performance-gpt-5}), this provides direct mechanistic evidence for our central claim---\emph{LLM performance on logical reasoning is constrained not by what the model has learned, but by how that knowledge is expressed at the surface}.

\subsection{Attention Pattern Analysis}

To investigate how language representations shape internal processing schemas, we extract attention weights from three key layers during process in logic circuit simulation task using Qwen3-32B\footnote{https://huggingface.co/Qwen/Qwen3-32B~\cite{yang2025qwen3}.}: Layer 6 (early processing), Layer 24 (middle processing), and Layer 48 (output generation).

For each layer and language representation, we identify three representative attention heads based on attention distribution variance:
(1) Most Varied Head: The head with the highest variance across the attention matrix, typically capturing focal attention to key tokens; (2) Medium Varied Head: The head with median variance, typically capturing structural dependencies and sequential relationships; (3) Least Varied Head: The head with the lowest variance, typically performing stable baseline processing.

Heads were selected independently for each layer and language representation by computing the variance of each head's attention matrix and selecting the maximum, median, and minimum. This functional role-based selection allows us to compare processing strategies across languages rather than tracking specific parameter instances.
Each attention pattern is visualized as a heatmap which red indicating high attention, and blue indicates low attention as shown in \cref{fig:attene1,fig:attene2,fig:attene3,fig:attenm1,fig:attenm2,fig:attenm3,fig:attenl1,fig:attenl2,fig:attenl3}. These visualizations serve as the manifestation of the induced schemas. The distinct patterns across different languages reflect their underlying schema divergence: a pattern with sharp, concentrated hotspots indicates a well-defined internal structural organization, whereas a scattered and fuzzy pattern reveals the presence of semantic noise and a lack of clear organizational logic within the schema.

\begin{figure}
    \centering
    \includegraphics[width=0.72\linewidth]{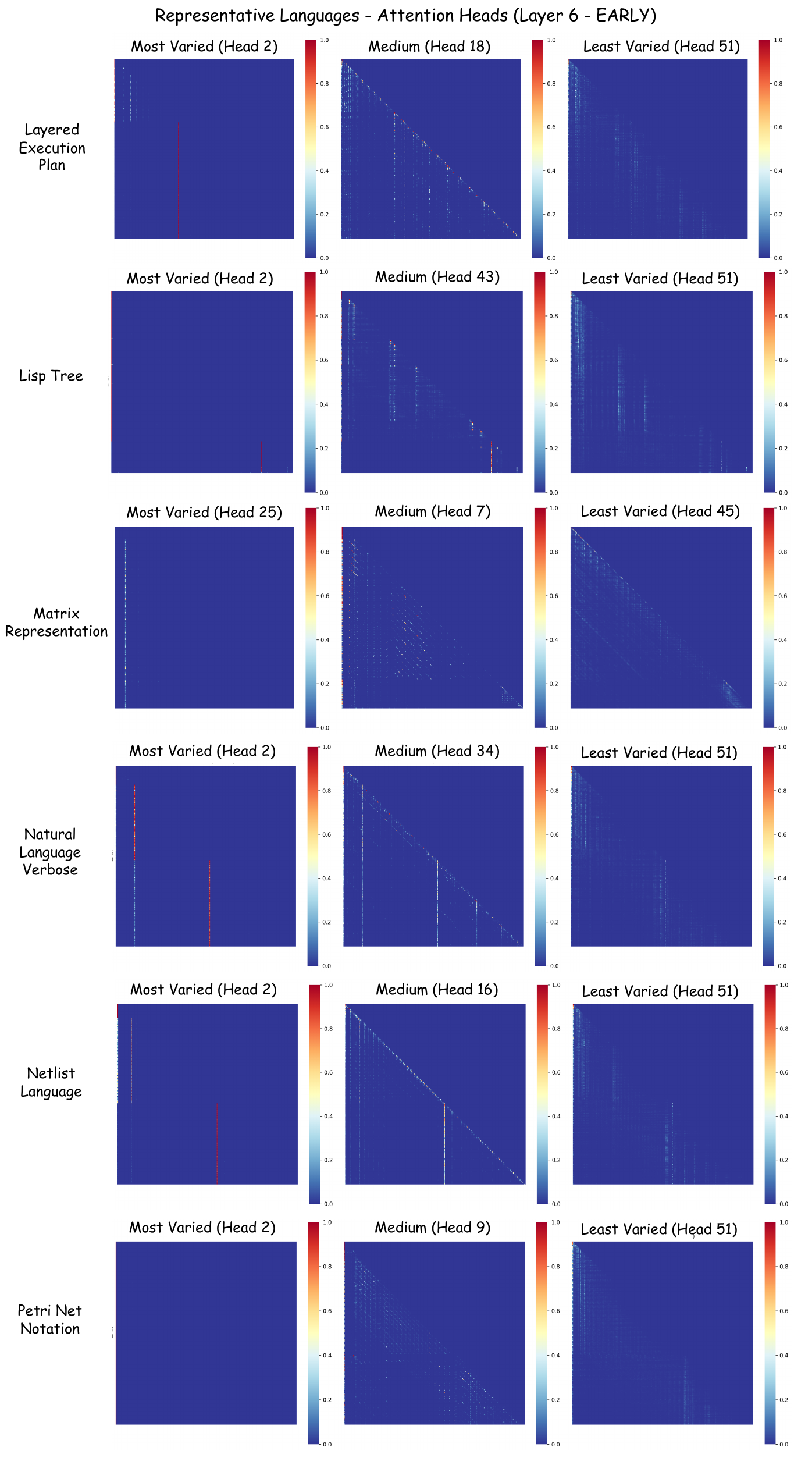}
    \caption{Visualization of attention weights in Layer 6 (early). Rows correspond to distinct language representations, while columns display heads with varying attention distribution variances. Red and blue denote high and low attention weights, respectively.}
    \label{fig:attene1}
\end{figure}

\begin{figure}
    \centering
    \includegraphics[width=0.72\linewidth]{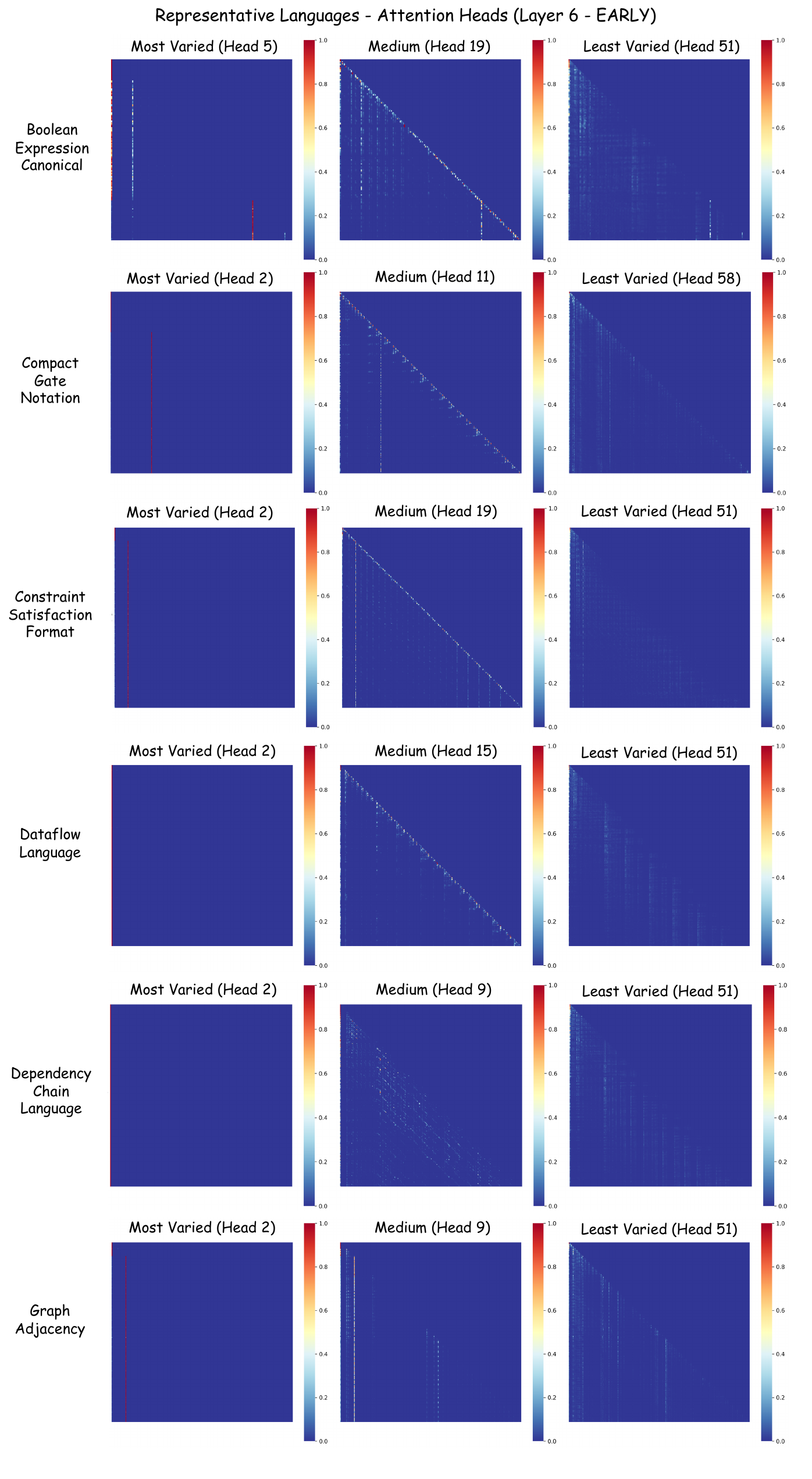}
    \caption{Visualization of attention weights in Layer 6 (early). Rows correspond to distinct language representations, while columns display heads with varying attention distribution variances. Red and blue denote high and low attention weights, respectively.}
    \label{fig:attene2}
\end{figure}

\begin{figure}
    \centering
    \includegraphics[width=0.72\linewidth]{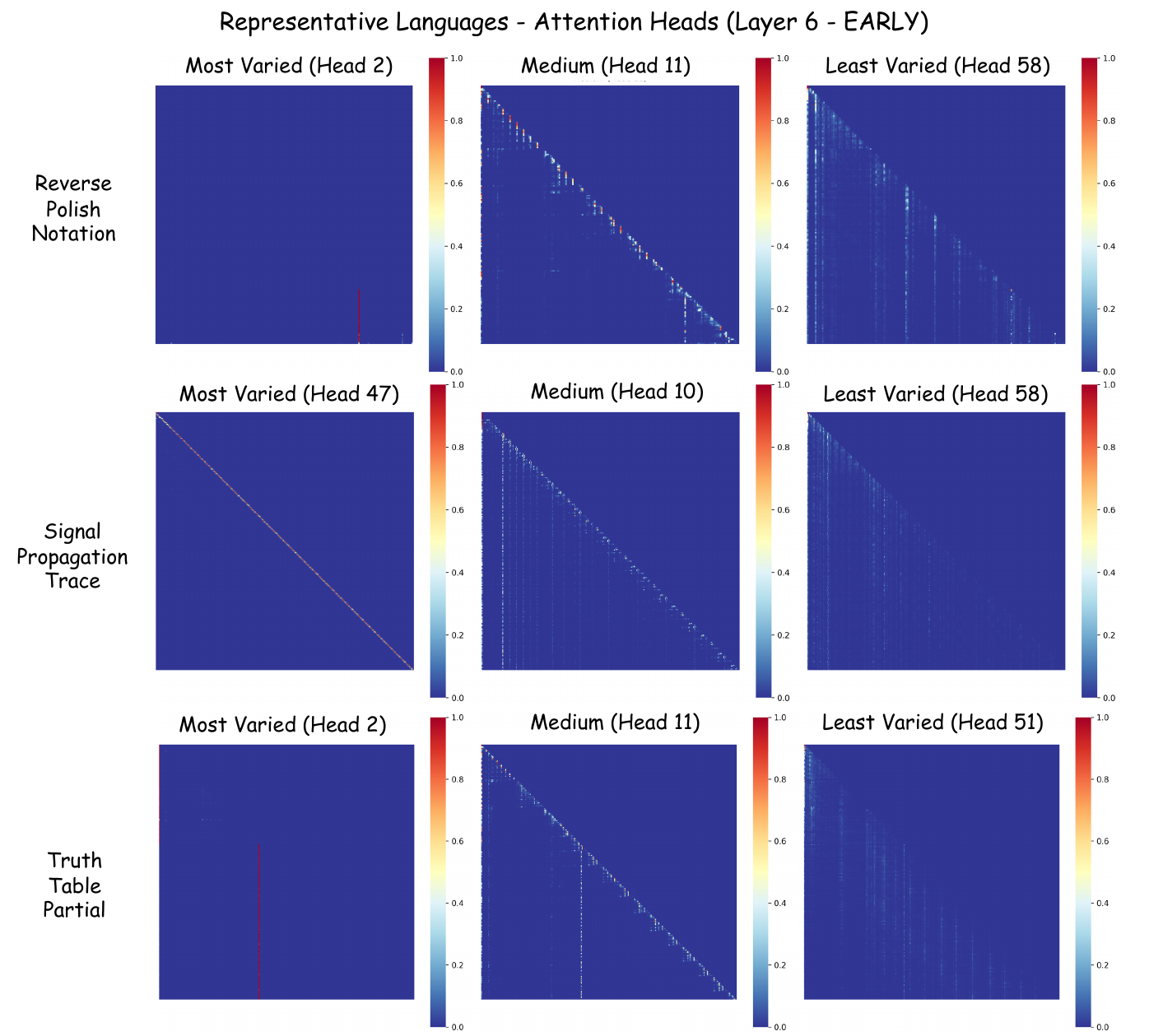}
    \caption{Visualization of attention weights in Layer 6 (early). Rows correspond to distinct language representations, while columns display heads with varying attention distribution variances. Red and blue denote high and low attention weights, respectively.}
    \label{fig:attene3}
\end{figure}

\begin{figure}
    \centering
    \includegraphics[width=0.72\linewidth]{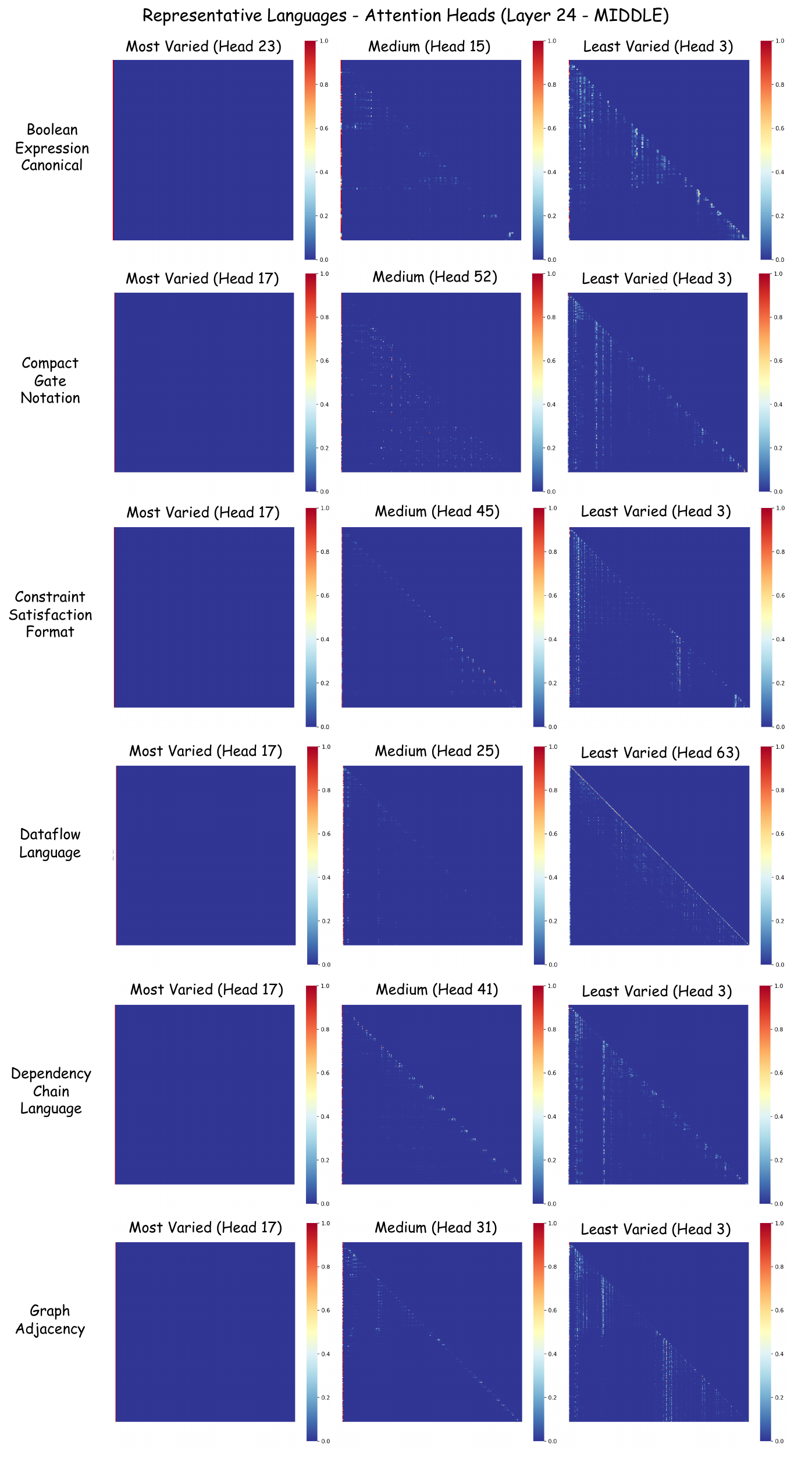}
    \caption{Visualization of attention weights in Layer 24 (middle). Rows correspond to distinct language representations, while columns display heads with varying attention distribution variances. Red and blue denote high and low attention weights, respectively.}
    \label{fig:attenm1}
\end{figure}

\begin{figure}
    \centering
    \includegraphics[width=0.72\linewidth]{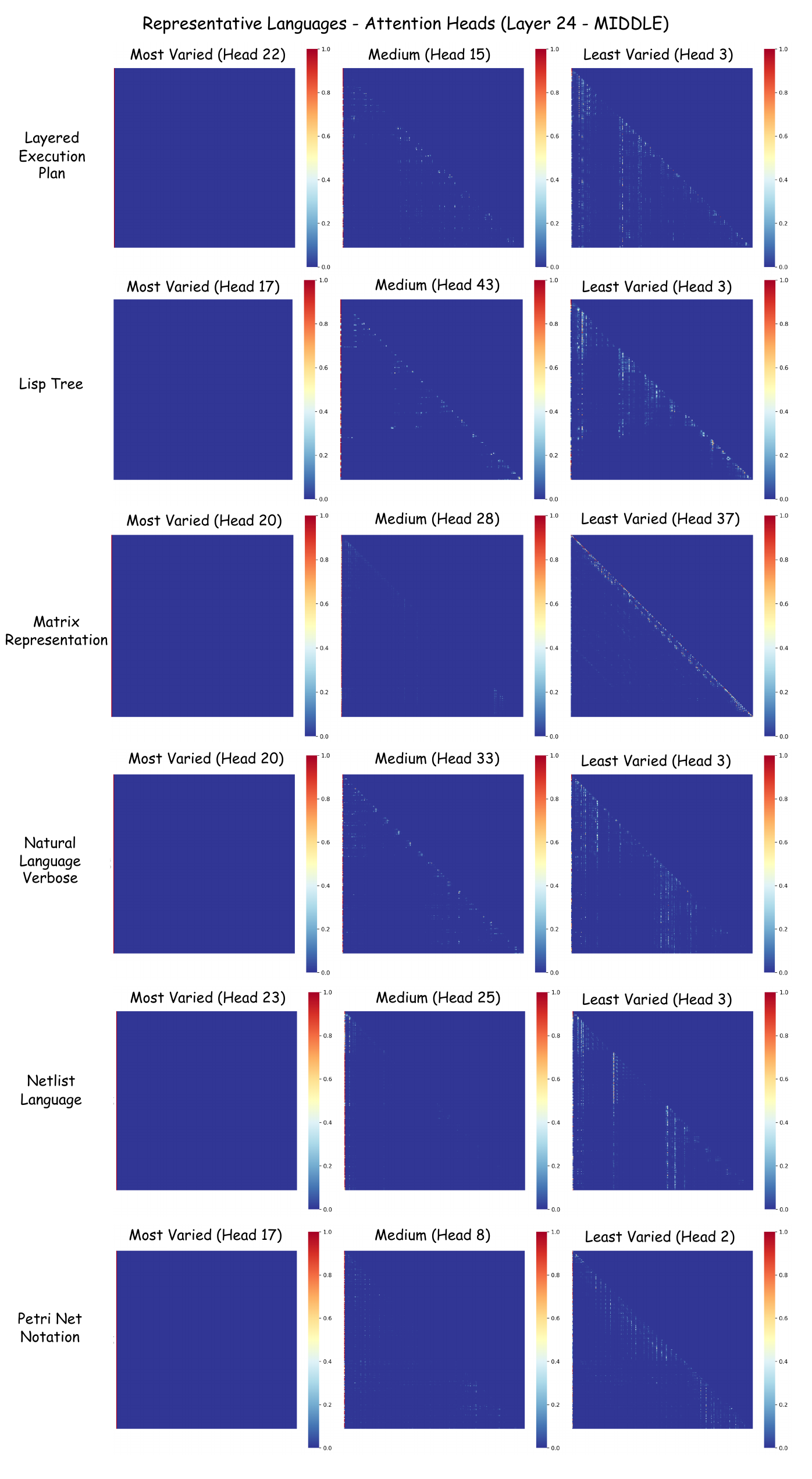}
    \caption{Visualization of attention weights in Layer 24 (middle). Rows correspond to distinct language representations, while columns display heads with varying attention distribution variances. Red and blue denote high and low attention weights, respectively.}
    \label{fig:attenm2}
\end{figure}

\begin{figure}
    \centering
    \includegraphics[width=0.72\linewidth]{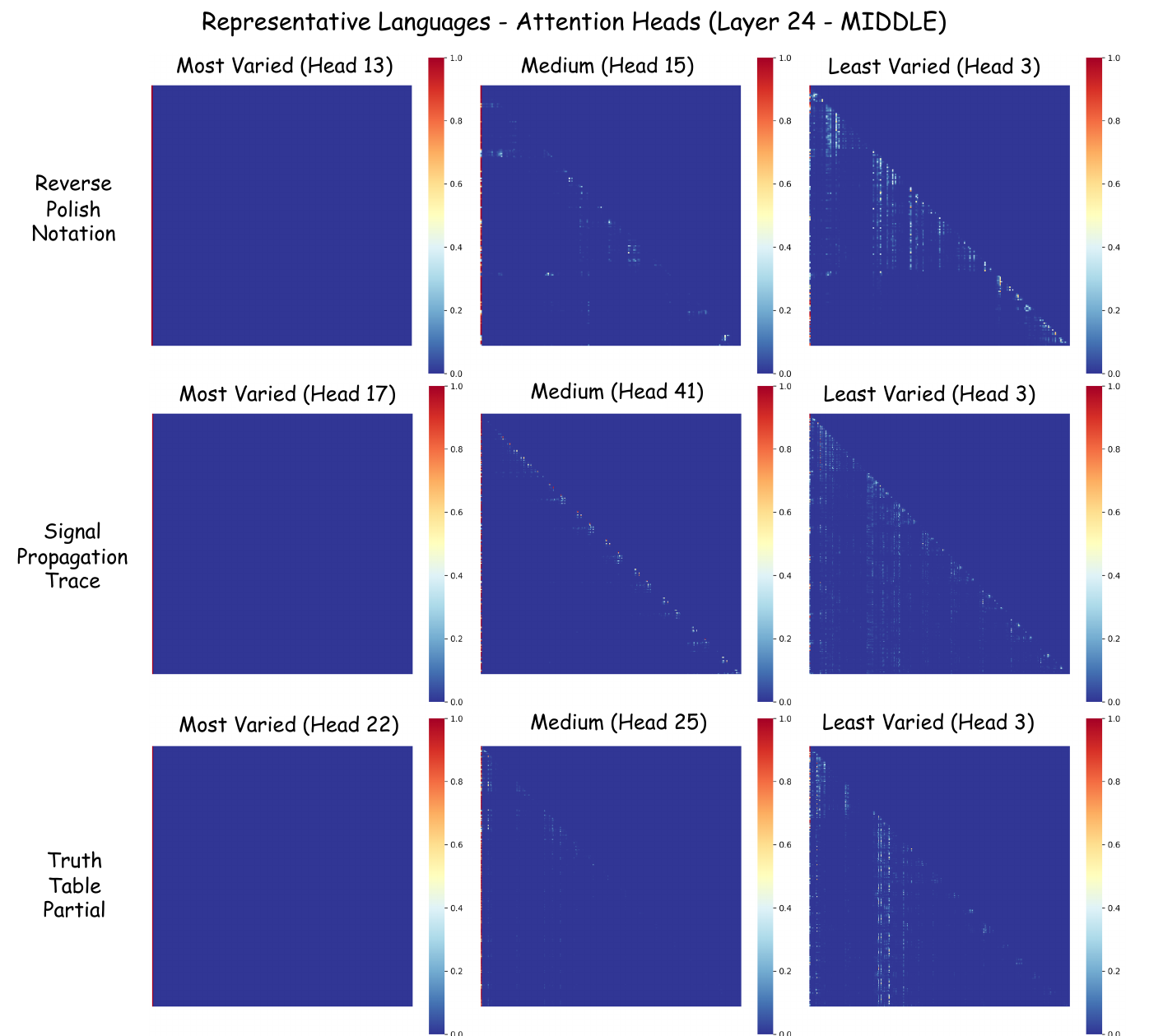}
    \caption{Visualization of attention weights in Layer 24 (middle). Rows correspond to distinct language representations, while columns display heads with varying attention distribution variances. Red and blue denote high and low attention weights, respectively.}
    \label{fig:attenm3}
\end{figure}

\begin{figure}
    \centering
    \includegraphics[width=0.72\linewidth]{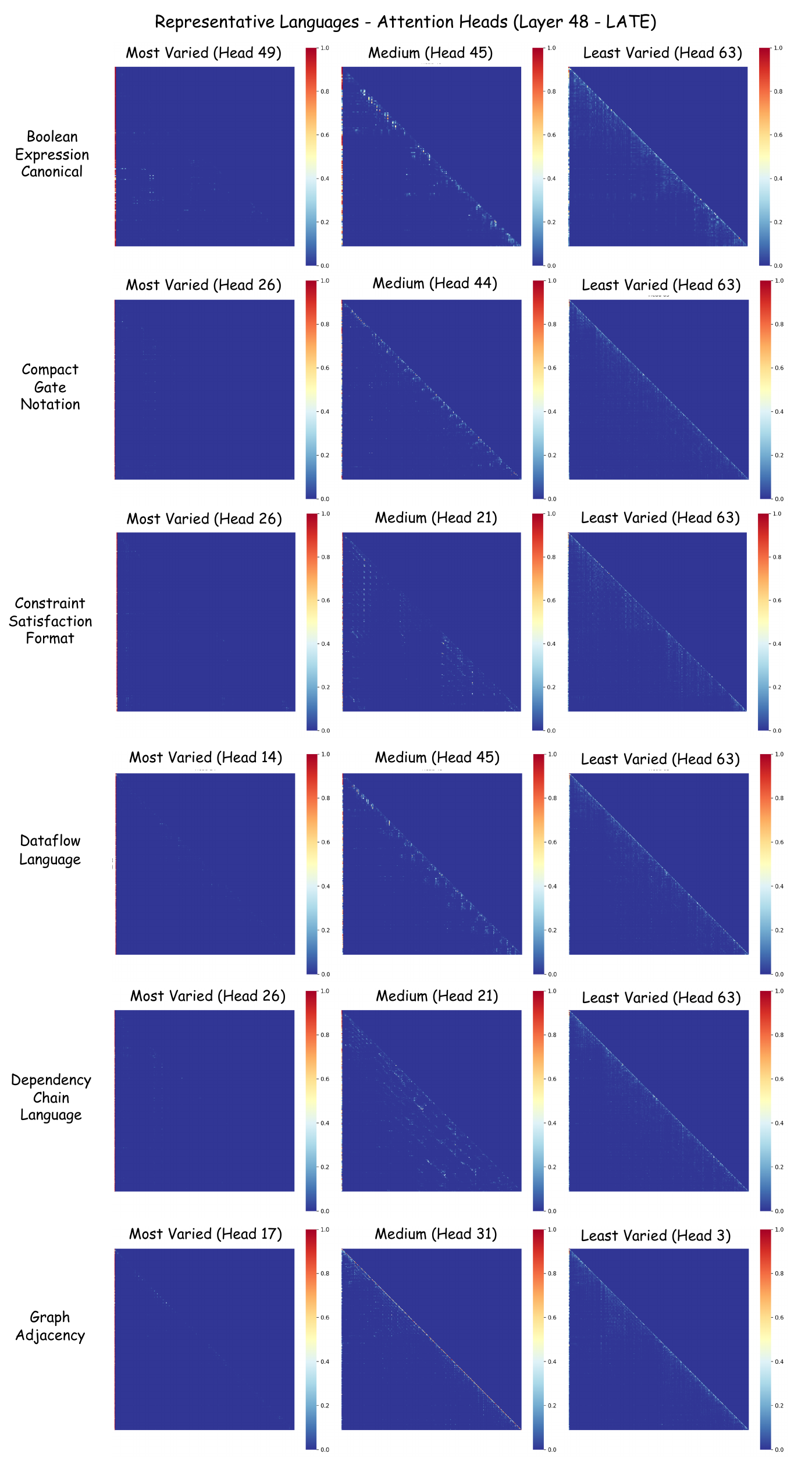}
    \caption{Visualization of attention weights in Layer 48 (late). Rows correspond to distinct language representations, while columns display heads with varying attention distribution variances. Red and blue denote high and low attention weights, respectively.}
    \label{fig:attenl1}
\end{figure}

\begin{figure}
    \centering
    \includegraphics[width=0.72\linewidth]{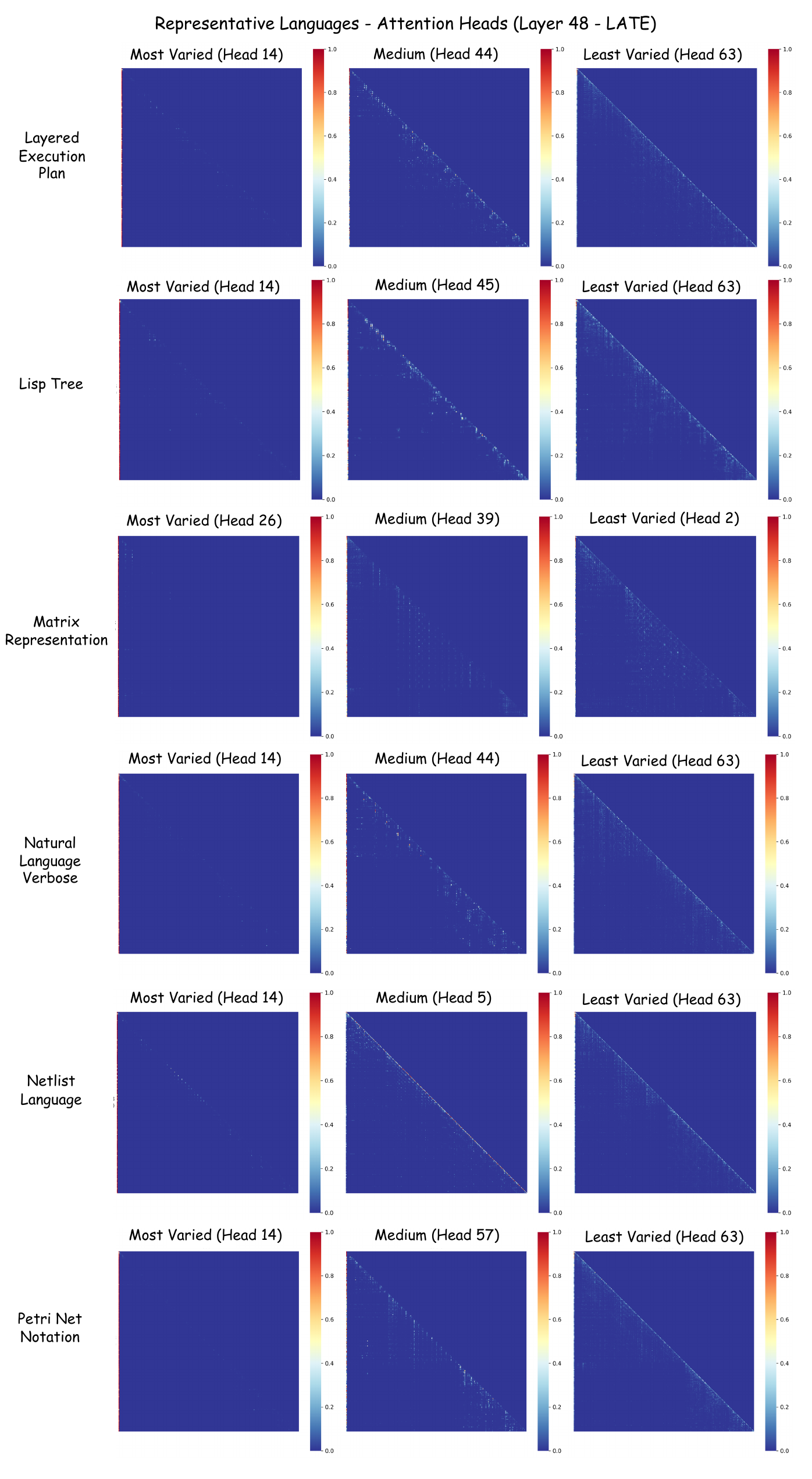}
    \caption{Visualization of attention weights in Layer 48 (late). Rows correspond to distinct language representations, while columns display heads with varying attention distribution variances. Red and blue denote high and low attention weights, respectively.}
    \label{fig:attenl2}
\end{figure}

\begin{figure}
    \centering
    \includegraphics[width=0.72\linewidth]{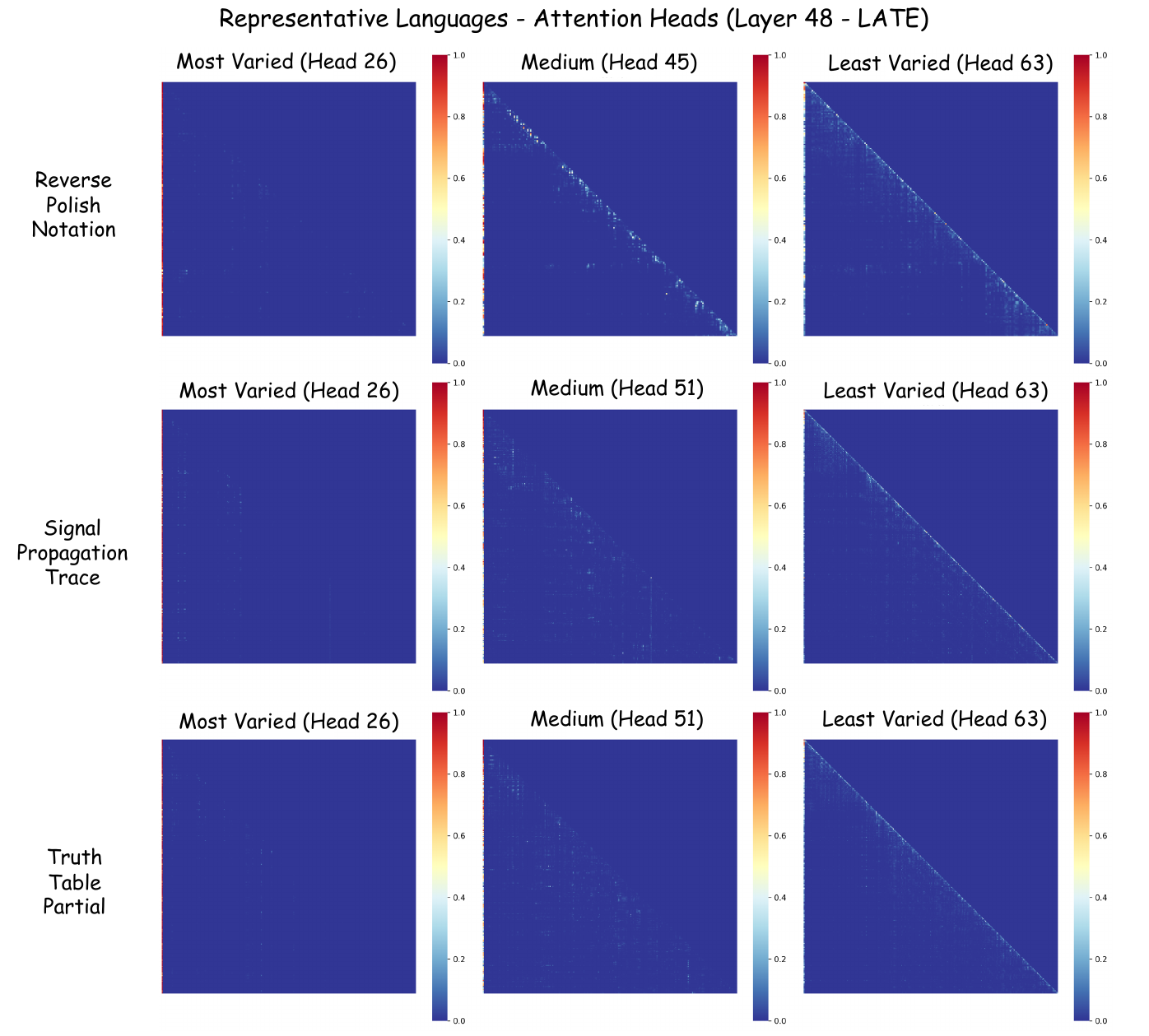}
    \caption{Visualization of attention weights in Layer 48 (late). Rows correspond to distinct language representations, while columns display heads with varying attention distribution variances. Red and blue denote high and low attention weights, respectively.}
    \label{fig:attenl3}
\end{figure}

\input{appendix/subsection_show_language}

%% file: main/tabs/gpt5_new.tex
\begin{table}[t]
\centering
\small
\caption{The performance comparison of different language representations for the logic circuit simulation task by \textbf{Gpt-5-chat}. We use color to annotate the best (only when acc is above 80 will be considered as a candidate; otherwise, efficiency is meaningless.)}
\label{tab:language-format-performance-gpt-5}
\resizebox{\linewidth}{!}{
\begin{tabular}{lccc}
\toprule[1.5pt]
Language Format & Accuracy (\%) & Avg. Time (s) & Avg. Tokens (Prompt / Completion) \\
\midrule
\hline
Canonical Boolean Expression & \cellcolor{red!50}$100.00{\pm}0.00$ & \cellcolor{cyan!50}$7.52{\pm}3.0$ & \cellcolor{green!50}$258.3{\pm}49.2$ / $1071.7{\pm}432.5$ \\
Graph Adjacency Notation & \cellcolor{red!50}$100.00{\pm}0.00$ & $9.74{\pm}1.6$ & $513.5{\pm}20.9$ / $1432.0{\pm}241.2$ \\
Natural Language & \cellcolor{red!50}$100.00{\pm}0.00$ & \cellcolor{cyan!40}$8.14{\pm}1.4$ & $920.8{\pm}10.0$ / $1308.6{\pm}217.8$ \\
Netlist Language & \cellcolor{red!50}$100.00{\pm}0.00$ & \cellcolor{cyan!10}$8.86{\pm}1.3$ & $509.8{\pm}10.0$ / $1500.7{\pm}216.9$ \\
Layered Execution Plan & \cellcolor{red!10}$99.17{\pm}0.59$ & \cellcolor{cyan!30}$8.67{\pm}1.2$ & \cellcolor{green!20}$454.8{\pm}10.0$ / $1484.8{\pm}209.8$ \\
Lisp Tree Notation & $92.50{\pm}2.05$ & \cellcolor{cyan!20}$8.83{\pm}2.5$ & \cellcolor{green!40}$315.5{\pm}75.2$ / $1128.8{\pm}323.3$ \\
Compact Gate Notation & $85.42{\pm}1.18$ & $9.20{\pm}1.3$ & \cellcolor{green!30}$357.2{\pm}9.9$ / $1521.6{\pm}220.8$ \\
Dataflow Language & $85.42{\pm}2.13$ & $9.52{\pm}2.1$ & \cellcolor{green!10}$474.2{\pm}15.4$ / $1623.2{\pm}357.3$ \\
\midrule
\hline
Dependency Chain Language & $79.17{\pm}0.59$ & $12.90{\pm}2.0$ & $414.8{\pm}13.9$ / $1948.9{\pm}301.9$ \\
Reverse Polish Notation & $72.08{\pm}2.13$ & $12.02{\pm}5.8$ & $264.1{\pm}57.8$ / $1683.1{\pm}811.7$ \\
Matrix Representation & $32.50{\pm}4.46$ & $15.80{\pm}2.8$ & $1938.2{\pm}1.8$ / $2084.1{\pm}370.3$ \\
Constraint Satisfaction Format & $21.25{\pm}0.00$ & $10.19{\pm}1.8$ & $663.8{\pm}13.9$ / $1887.6{\pm}335.2$ \\
Partial Truth Table & $21.25{\pm}0.00$ & $8.06{\pm}1.6$ & $451.8{\pm}10.0$ / $1354.5{\pm}265.3$ \\
Petri Net Notation & $20.83{\pm}0.59$ & $10.58{\pm}1.8$ & $955.2{\pm}29.7$ / $1693.0{\pm}295.2$ \\
Signal Propagation Trace & $20.42{\pm}1.18$ & $14.44{\pm}3.2$ & $392.0{\pm}0.0$ / $1688.4{\pm}377.4$ \\
\bottomrule[1.5pt]
\end{tabular}}
\end{table}

%% file: appendix/subsection_show_language.tex
\subsection{Examples for different language representation}
In this section, we show some examples of different language representation formats as shown in \cref{fig:logic_circuit_nl,fig:logic_circuit_netlist,fig:logic_circuit_graph,fig:logic_circuit_matrix,fig:logic_circuit_lisp,fig:logic_circuit_dataflow,fig:logic_circuit_partial,fig:logic_circuit_GCN,fig:logic_circuit_PRN,fig:logic_circuit_DCL,fig:logic_circuit_LEP,fig:logic_circuit_SPT,fig:logic_circuit_CSF,fig:logic_circuit_CBE,fig:logic_circuit_PNN} during inference for the logic circuit simulation task, including 15 languages listed in Appendix~\ref{appendix:lang_detail}.
\begin{figure}
    \centering
    \includegraphics[width=0.75\linewidth]{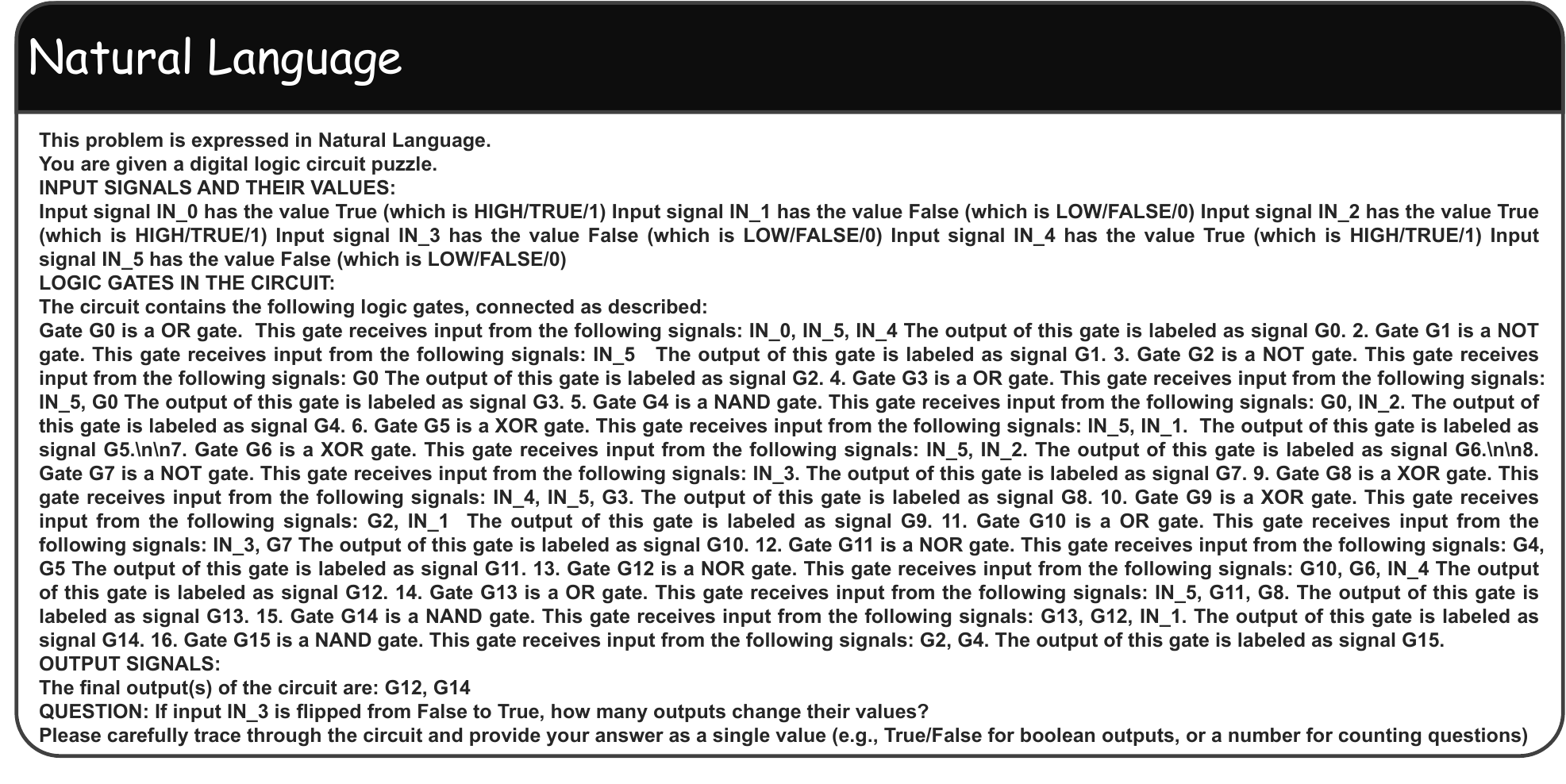}
    \caption{An example for the natural language representation for the logic circuit simulation task.}
    \label{fig:logic_circuit_nl}
\end{figure}
\begin{figure}
    \centering
    \includegraphics[width=0.75\linewidth]{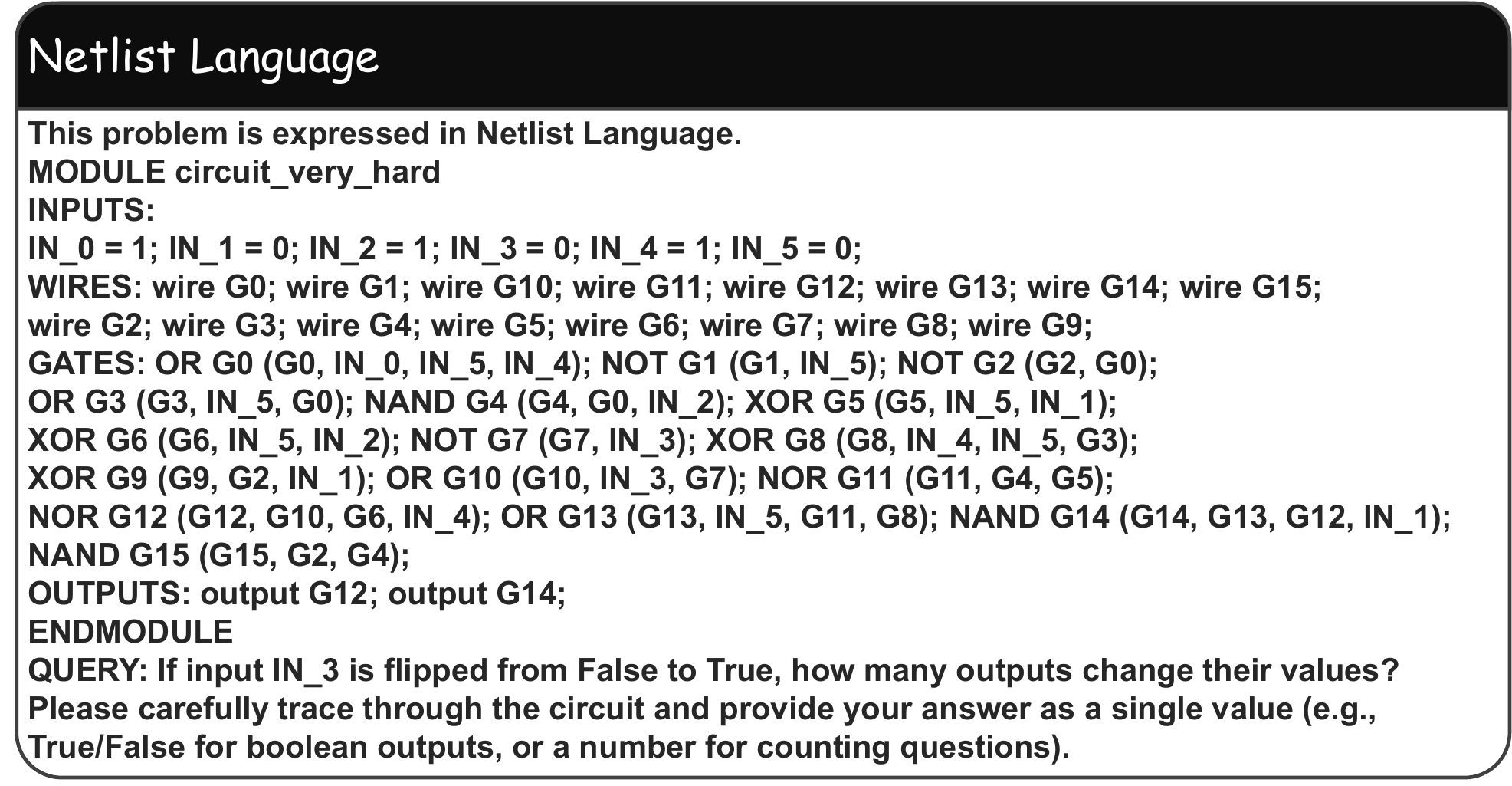}
    \caption{An example for the netlist language representation for the logic circuit simulation task.}
    \label{fig:logic_circuit_netlist}
\end{figure}
\begin{figure}
    \centering
    \includegraphics[width=0.75\linewidth]{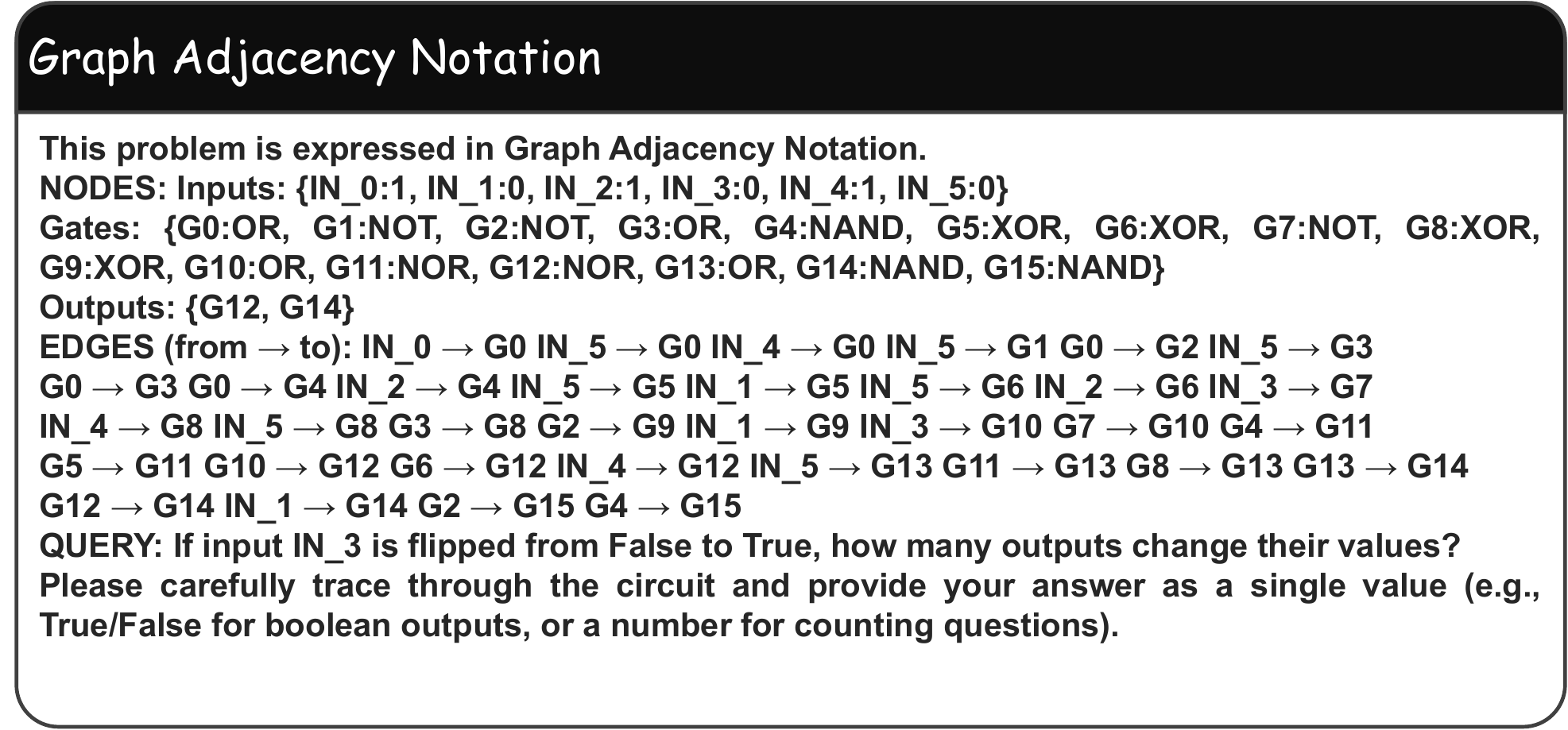}
    \caption{An example for the graph adjacency notation representation for the logic circuit simulation task.}
    \label{fig:logic_circuit_graph}
\end{figure}
\begin{figure}
    \centering
    \includegraphics[width=0.75\linewidth]{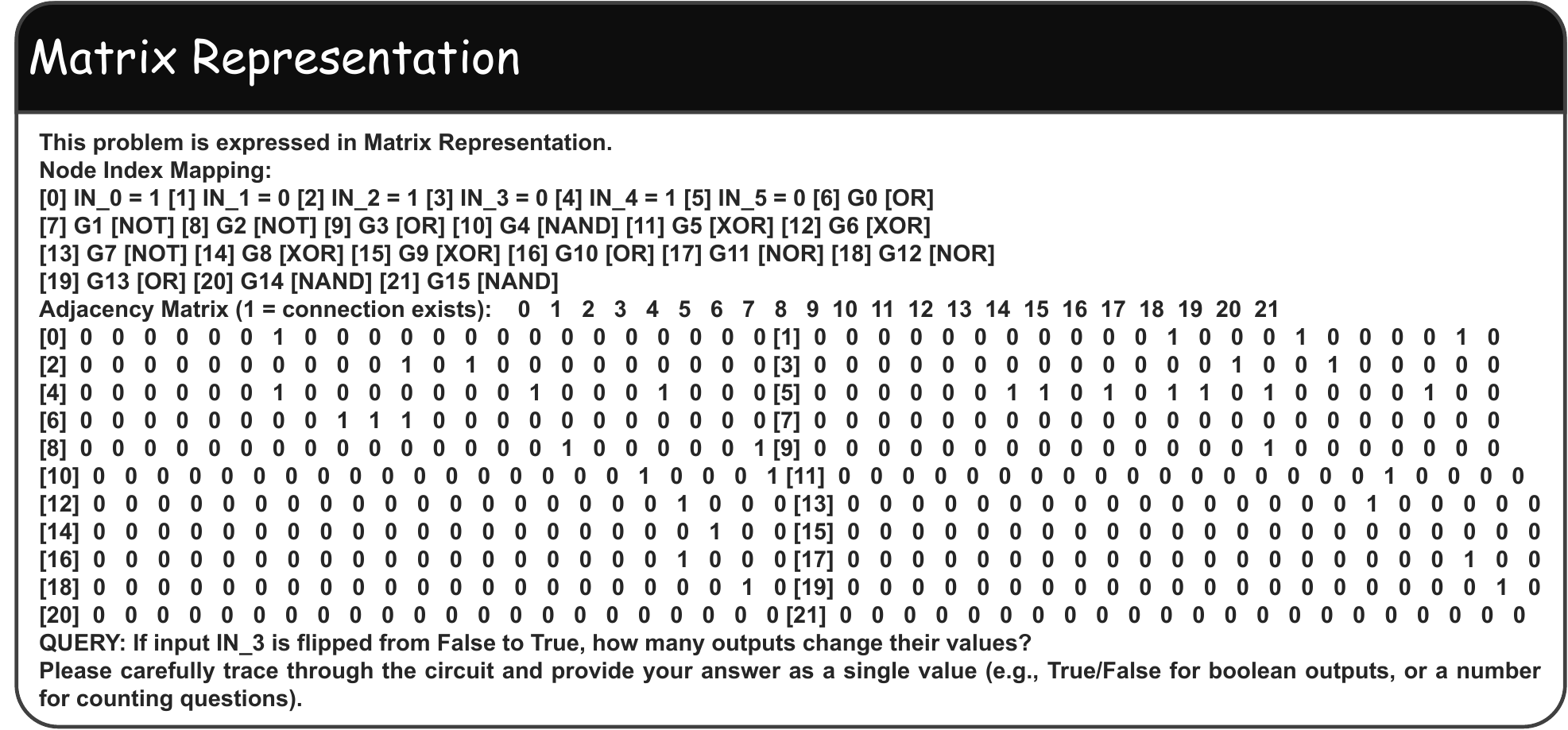}
    \caption{An example for the matrix representation for the logic circuit simulation task.}
    \label{fig:logic_circuit_matrix}
\end{figure}
\begin{figure}
    \centering
    \includegraphics[width=0.75\linewidth]{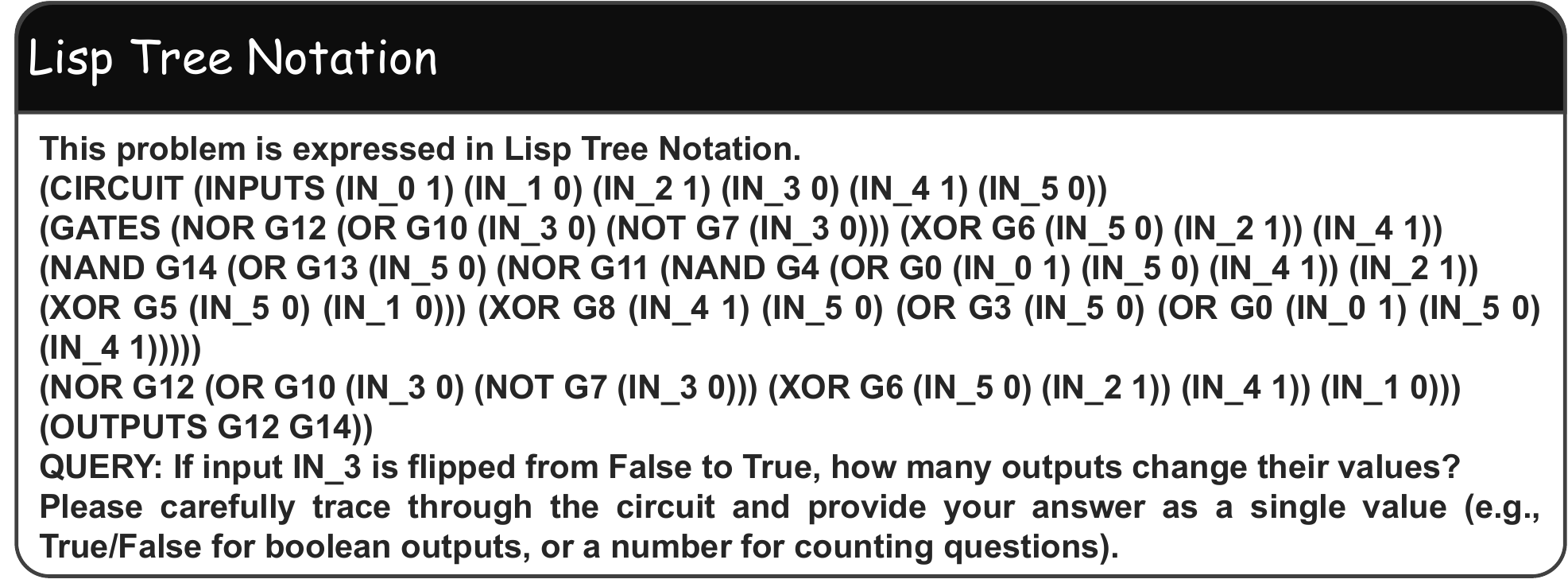}
    \caption{An example for the lisp tree representation for the logic circuit simulation task.}
    \label{fig:logic_circuit_lisp}
\end{figure}
\begin{figure}
    \centering
    \includegraphics[width=0.75\linewidth]{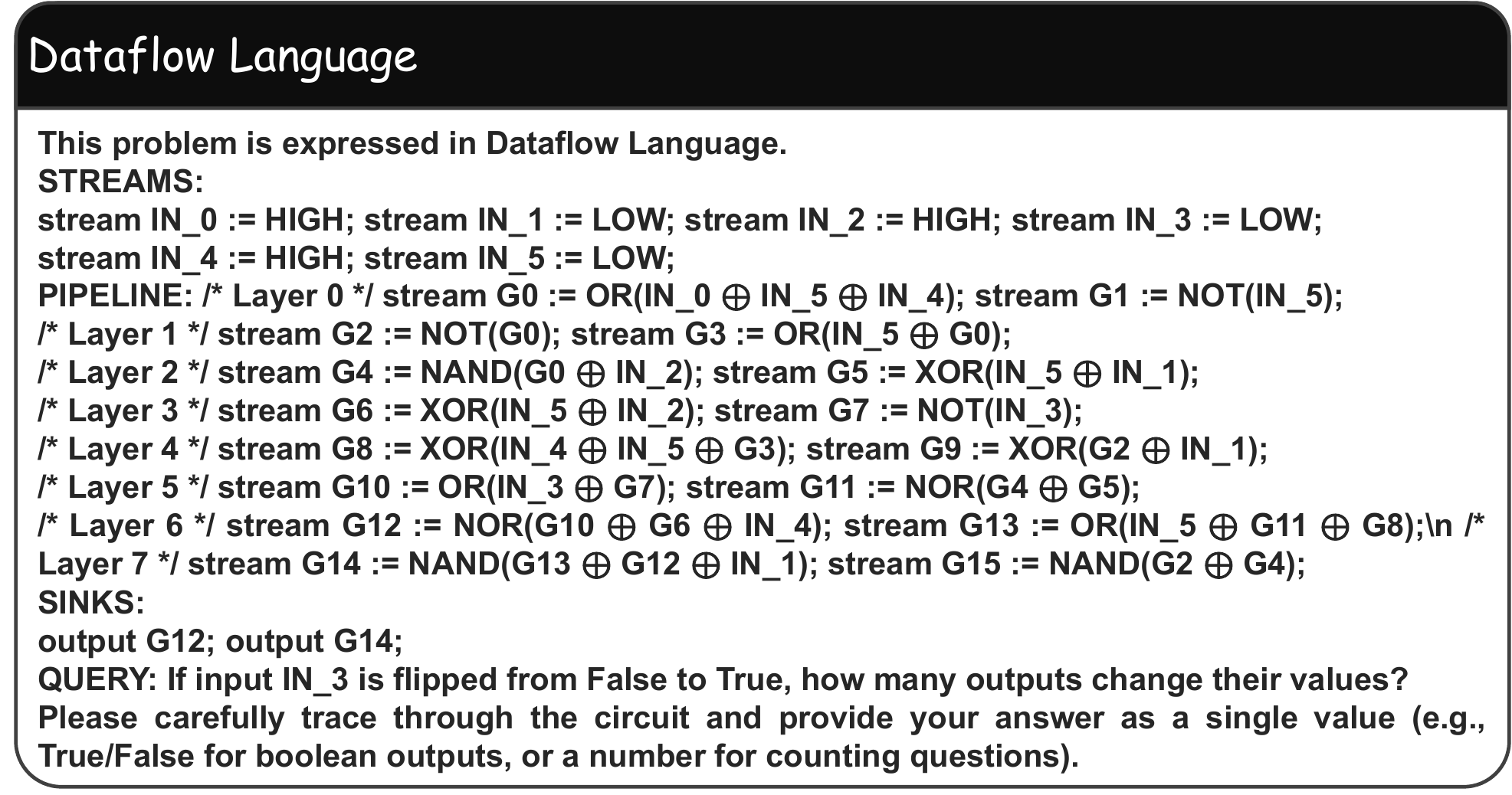}
    \caption{An example for the dataflow language representation for the logic circuit simulation task.}
    \label{fig:logic_circuit_dataflow}
\end{figure}\begin{figure}
    \centering
    \includegraphics[width=0.75\linewidth]{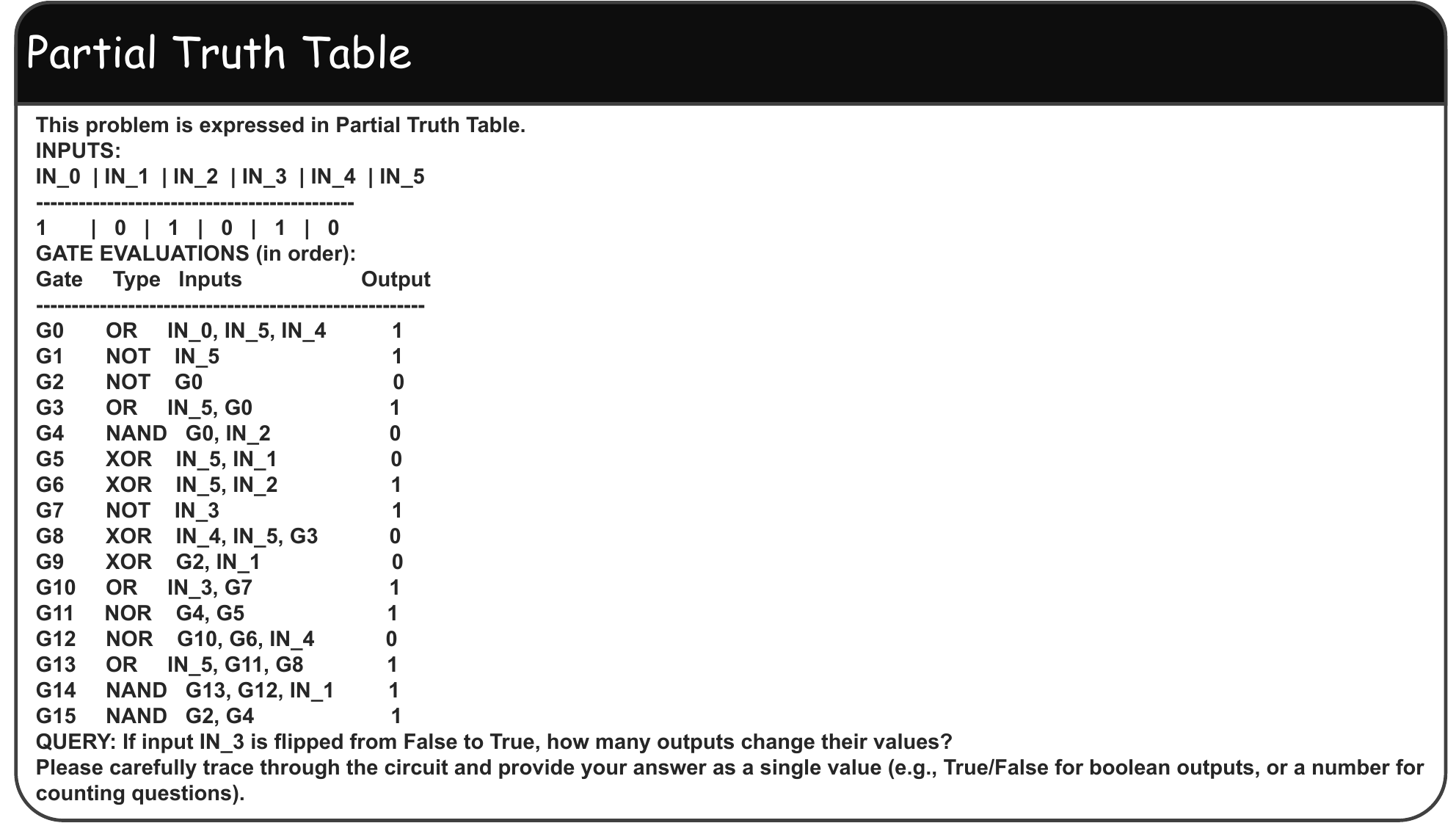}
    \caption{An example for the partial truth table representation for the logic circuit simulation task.}
    \label{fig:logic_circuit_partial}
\end{figure}
\begin{figure}
    \centering
    \includegraphics[width=0.75\linewidth]{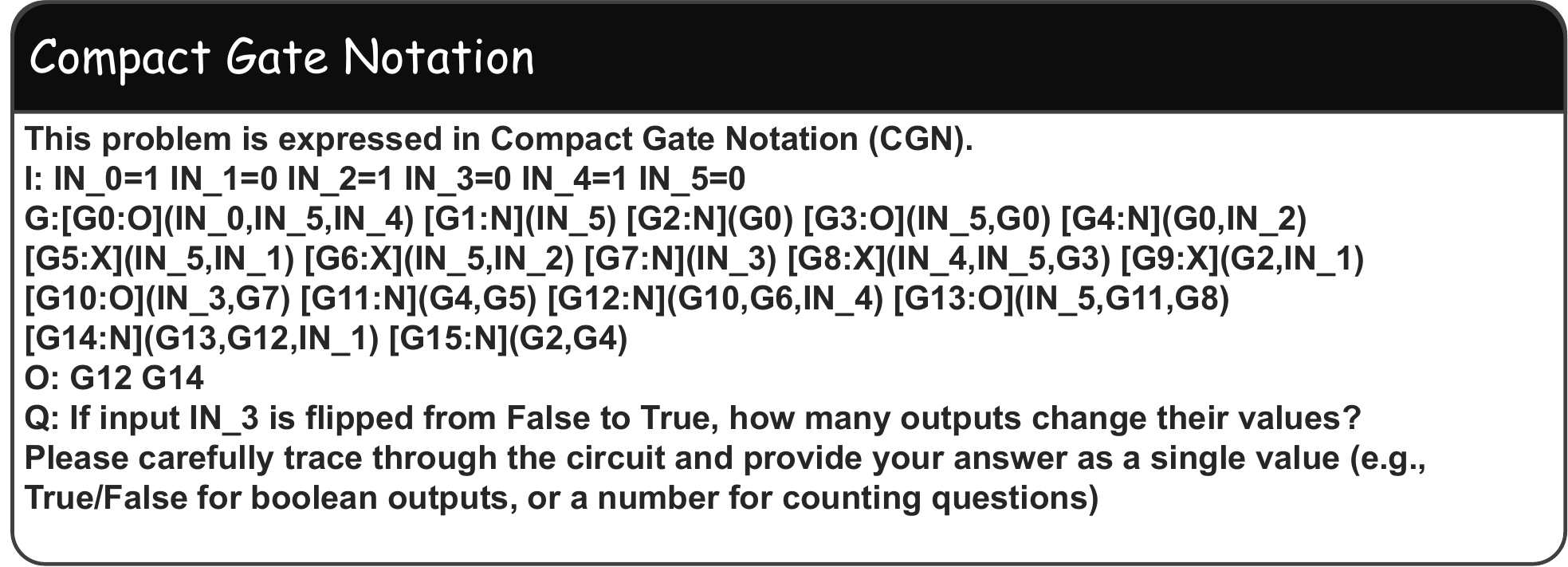}
    \caption{An example for the compact gate notation representation for the logic circuit simulation task.}
    \label{fig:logic_circuit_GCN}
\end{figure}
\begin{figure}
    \centering
    \includegraphics[width=0.75\linewidth]{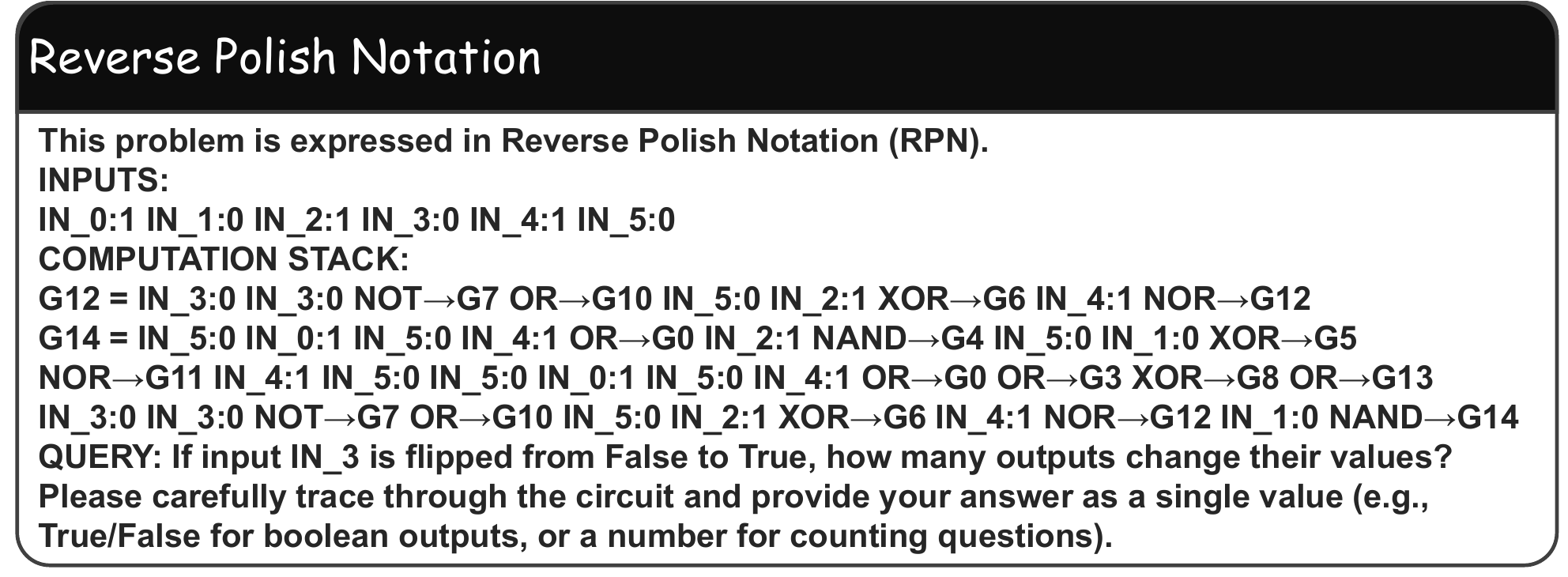}
    \caption{An example for the reverse polish notation representation for the logic circuit simulation task.}
    \label{fig:logic_circuit_PRN}
\end{figure}
\begin{figure}
    \centering
    \includegraphics[width=0.75\linewidth]{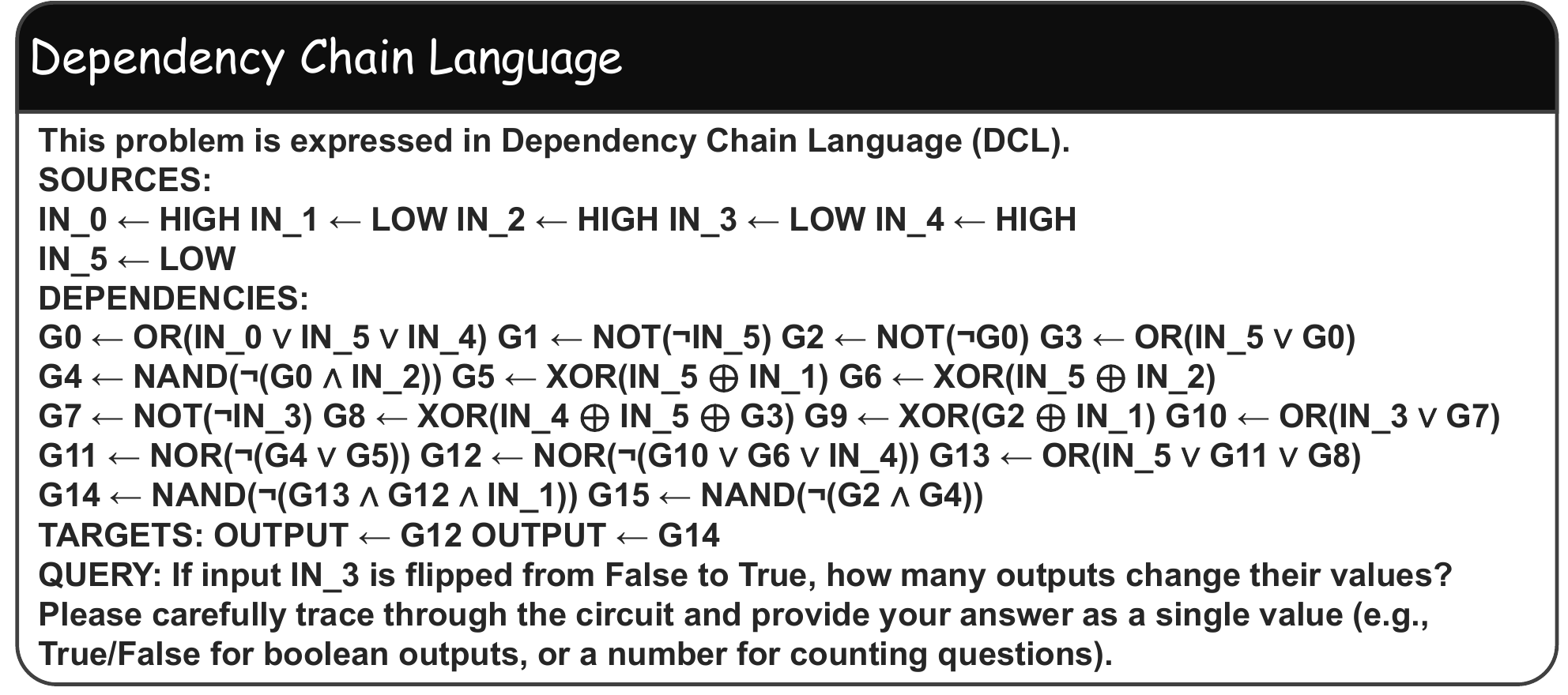}
    \caption{An example for the dependency chain language representation for the logic circuit simulation task.}
    \label{fig:logic_circuit_DCL}
\end{figure}
\begin{figure}
    \centering
    \includegraphics[width=0.75\linewidth]{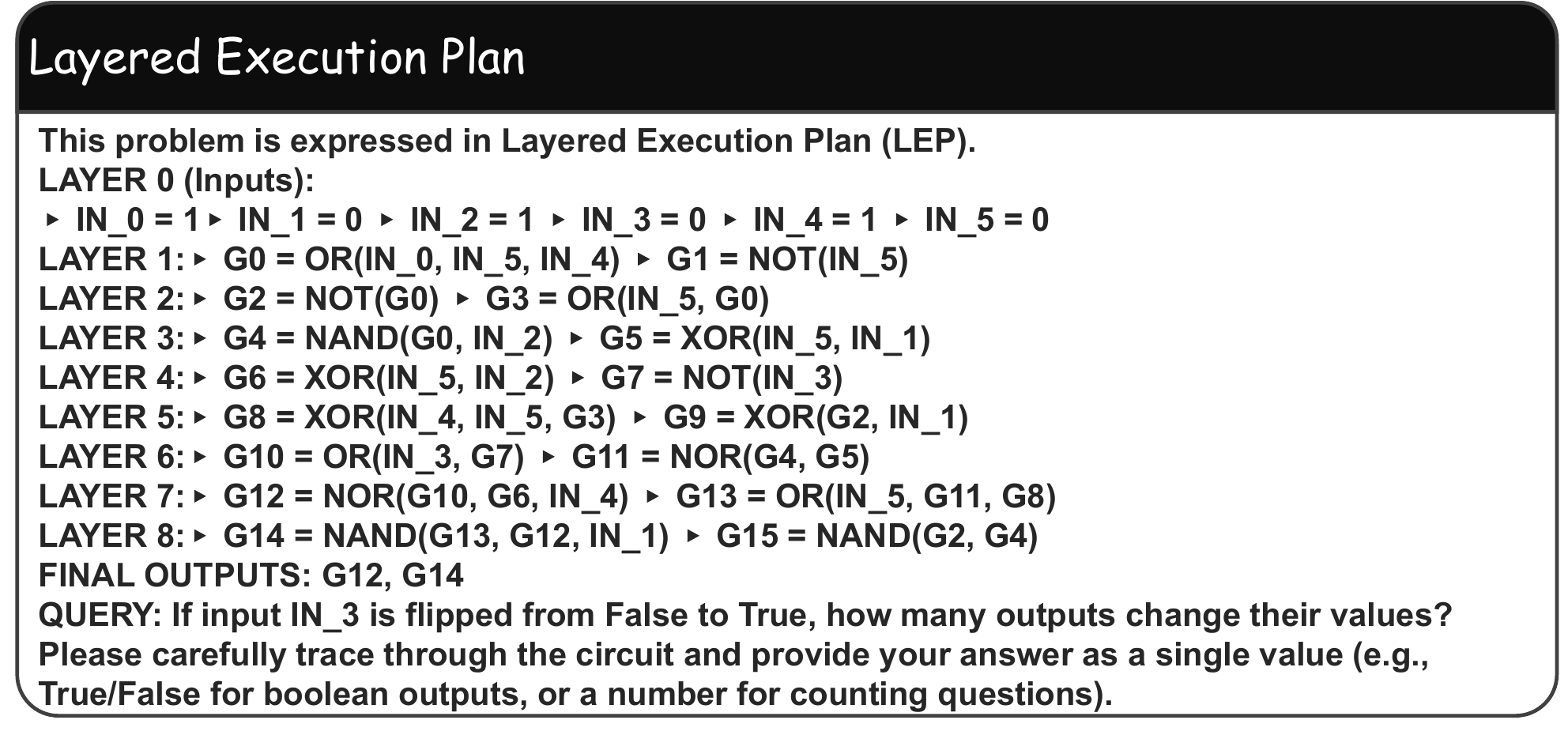}
    \caption{An example for the layered execution plan representation for the logic circuit simulation task.}
    \label{fig:logic_circuit_LEP}
\end{figure}
\begin{figure}
    \centering
    \includegraphics[width=0.75\linewidth]{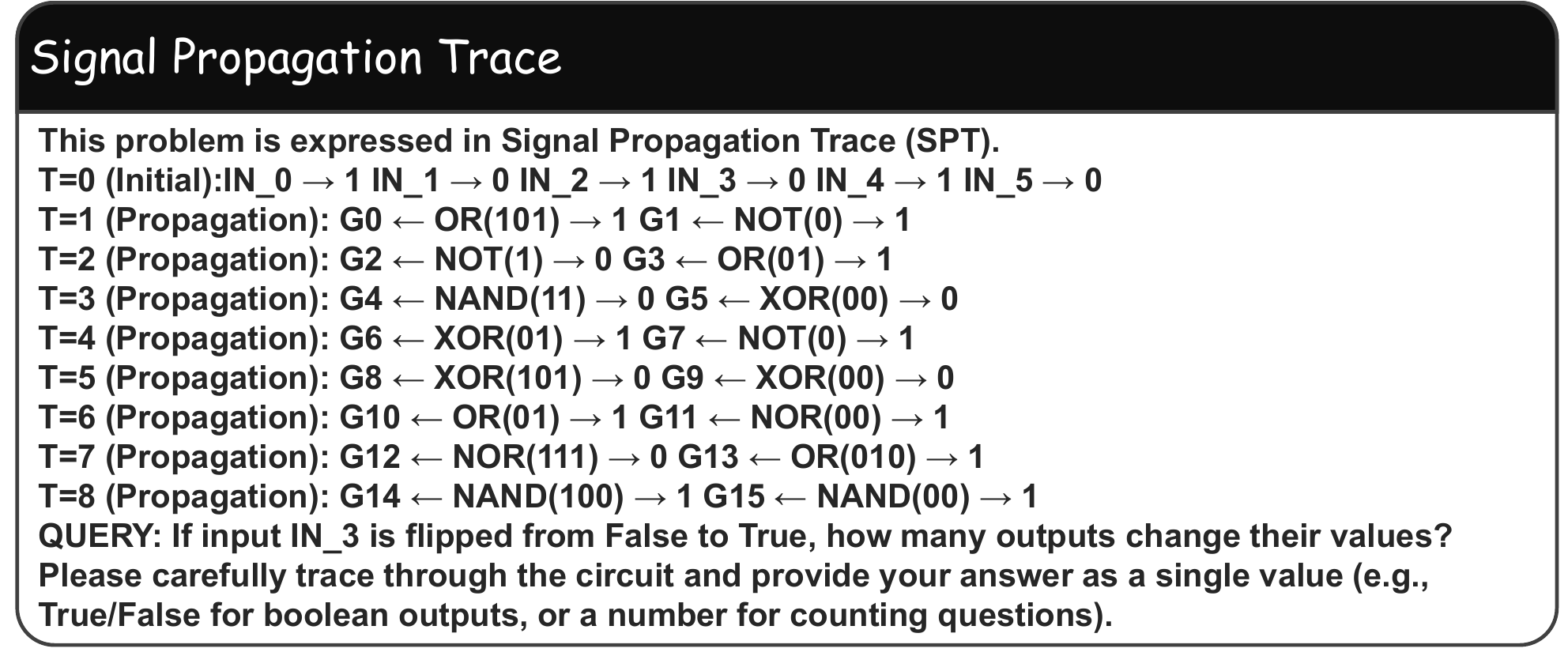}
    \caption{An example for the signal propagation trace representation for the logic circuit simulation task.}
    \label{fig:logic_circuit_SPT}
\end{figure}
\begin{figure}
    \centering
    \includegraphics[width=0.75\linewidth]{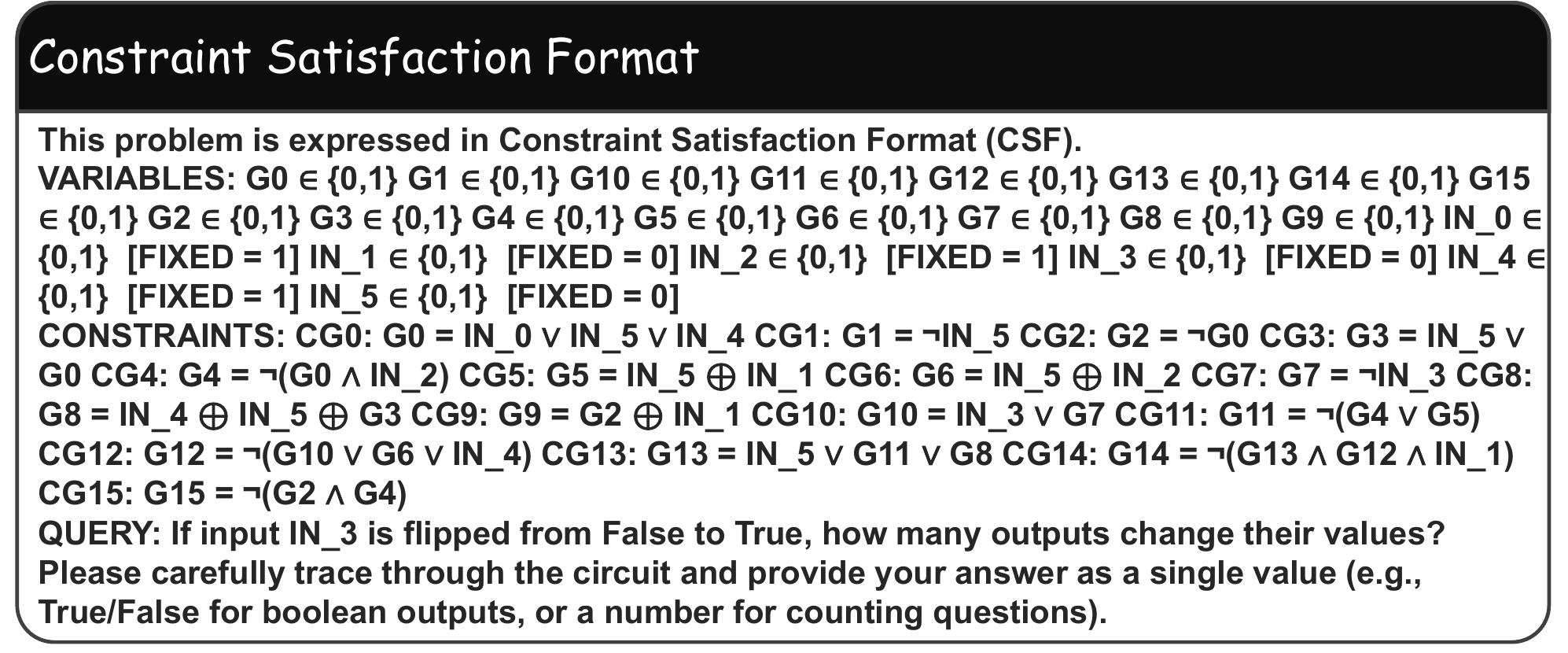}
    \caption{An example for the constraint satisfaction format representation for the logic circuit simulation task.}
    \label{fig:logic_circuit_CSF}
\end{figure}
\begin{figure}
    \centering
    \includegraphics[width=0.75\linewidth]{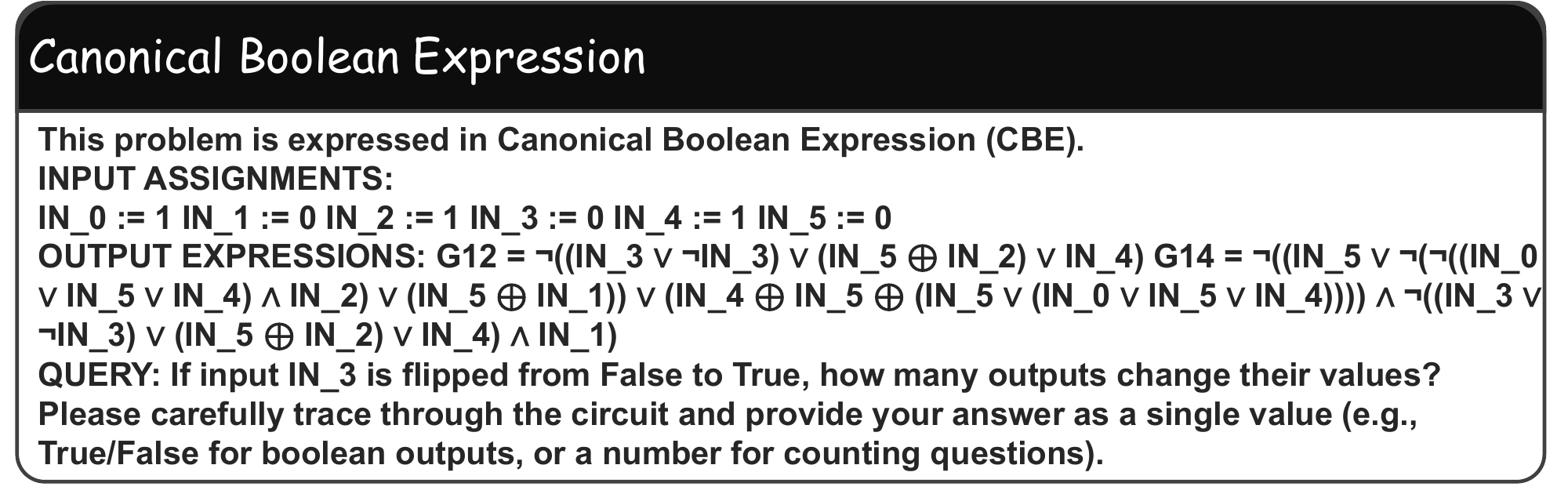}
    \caption{An example for the canonical boolean expression representation for the logic circuit simulation task.}
    \label{fig:logic_circuit_CBE}
\end{figure}
\begin{figure}
    \centering
    \includegraphics[width=0.75\linewidth]{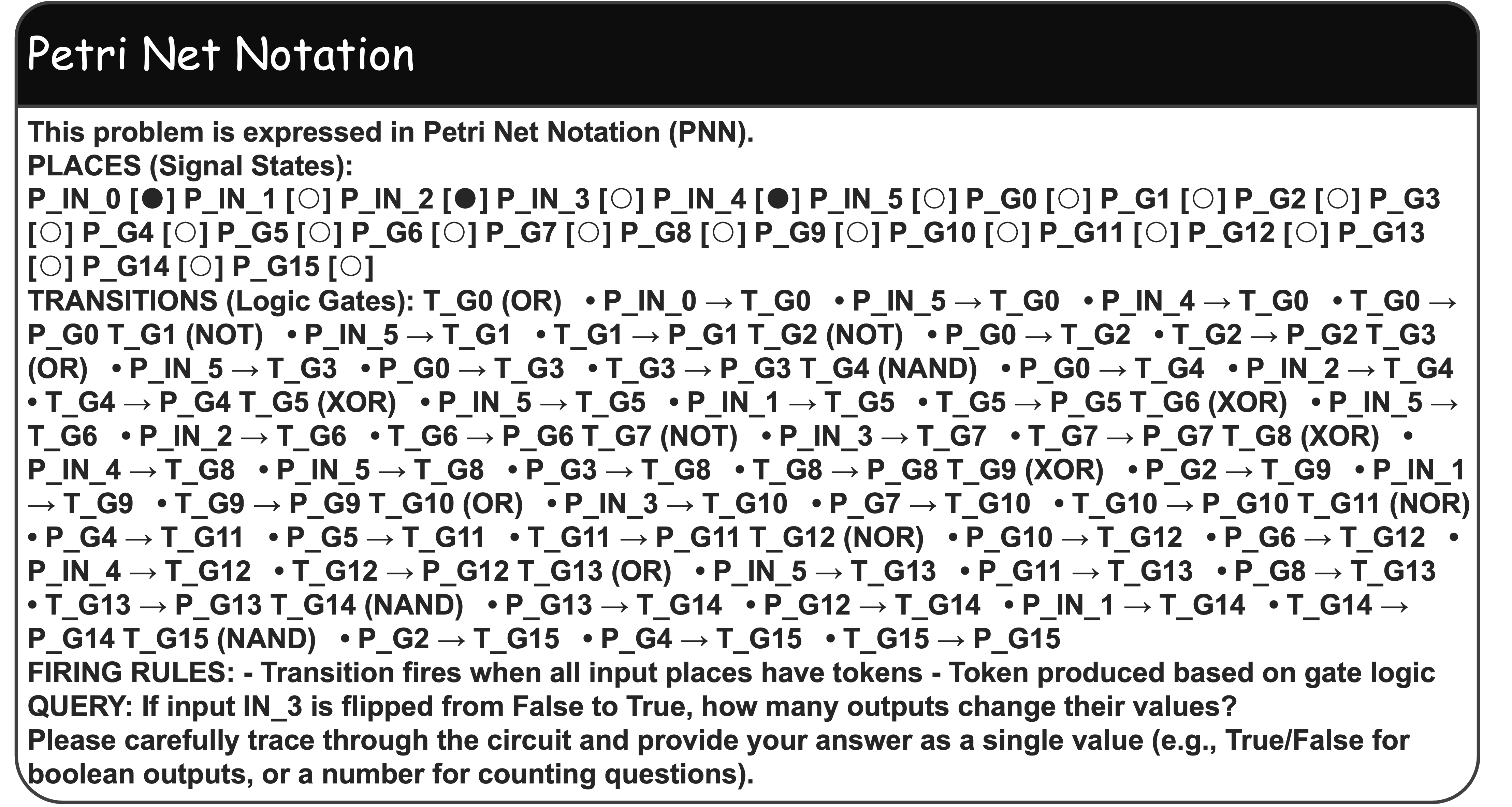}
    \caption{An example for the petri net notation representation for the logic circuit simulation task.}
    \label{fig:logic_circuit_PNN}
\end{figure}

%% file: example_paper.bib
@article{wang2023grammar,
  title={Grammar prompting for domain-specific language generation with large language models},
  author={Wang, Bailin and Wang, Zi and Wang, Xuezhi and Cao, Yuan and A Saurous, Rif and Kim, Yoon},
  journal={Advances in Neural Information Processing Systems},
  volume={36},
  pages={65030--65055},
  year={2023}
}

@article{zhou2025solving,
  title={Solving formal math problems by decomposition and iterative reflection},
  author={Zhou, Yichi and Zhao, Jianqiu and Zhang, Yongxin and Wang, Bohan and Wang, Siran and Chen, Luoxin and Wang, Jiahui and Chen, Haowei and Jie, Allan and Zhang, Xinbo and others},
  journal={arXiv preprint arXiv:2507.15225},
  year={2025}
}

@inproceedings{huangcode,
  title={Code-Generated Graph Representations Using Multiple LLM Agents for Material Properties Prediction},
  author={Huang, Jiao and Xing, Qianli and Ji, Jinglong and Yang, Bo},
  booktitle={Forty-second International Conference on Machine Learning},
  year={2025}
}

@article{barradas2025combining,
  title={Combining TSL and LLM to Automate REST API Testing: A Comparative Study},
  author={Barradas, Thiago and Paes, Aline and Neves, V{\^a}nia de Oliveira},
  journal={arXiv preprint arXiv:2509.05540},
  year={2025}
}

@article{smirnov2024generating,
  title={Generating consistent pddl domains with large language models},
  author={Smirnov, Pavel and Joublin, Frank and Ceravola, Antonello and Gienger, Michael},
  journal={arXiv preprint arXiv:2404.07751},
  year={2024}
}

@article{kaplan2020scaling,
  title={Scaling laws for neural language models},
  author={Kaplan, Jared and McCandlish, Sam and Henighan, Tom and Brown, Tom B and Chess, Benjamin and Child, Rewon and Gray, Scott and Radford, Alec and Wu, Jeffrey and Amodei, Dario},
  journal={arXiv preprint arXiv:2001.08361},
  year={2020}
}

@article{wei2022emergent,
  title={Emergent abilities of large language models},
  author={Wei, Jason and Tay, Yi and Bommasani, Rishi and Raffel, Colin and Zoph, Barret and Borgeaud, Sebastian and Yogatama, Dani and Bosma, Maarten and Zhou, Denny and Metzler, Donald and others},
  journal={arXiv preprint arXiv:2206.07682},
  year={2022}
}

@article{jamali2024semantic,
  title={Semantic encoding during language comprehension at single-cell resolution},
  author={Jamali, Mohsen and Grannan, Benjamin and Cai, Jing and Khanna, Arjun R and Mu{\~n}oz, William and Caprara, Irene and Paulk, Angelique C and Cash, Sydney S and Fedorenko, Evelina and Williams, Ziv M},
  journal={Nature},
  volume={631},
  number={8021},
  pages={610--616},
  year={2024},
  publisher={Nature Publishing Group UK London}
}

@article{hassabis2017neuroscience,
  title={Neuroscience-inspired artificial intelligence},
  author={Hassabis, Demis and Kumaran, Dharshan and Summerfield, Christopher and Botvinick, Matthew},
  journal={Neuron},
  volume={95},
  number={2},
  pages={245--258},
  year={2017},
  publisher={Elsevier}
}

@article{mirjalili2025using,
  title={Using machine learning to simultaneously quantify multiple cognitive components of episodic memory},
  author={Mirjalili, Soroush and Duarte, Audrey},
  journal={Nature Communications},
  volume={16},
  number={1},
  pages={2856},
  year={2025},
  publisher={Nature Publishing Group UK London}
}

@article{haber2022prefrontal,
  title={Prefrontal connectomics: from anatomy to human imaging},
  author={Haber, Suzanne N and Liu, Hesheng and Seidlitz, Jakob and Bullmore, Ed},
  journal={Neuropsychopharmacology},
  volume={47},
  number={1},
  pages={20--40},
  year={2022},
  publisher={Springer International Publishing Cham}
}

@misc{mitchell2024debates,
  title={Debates on the nature of artificial general intelligence},
  author={Mitchell, Melanie},
  journal={Science},
  volume={383},
  number={6689},
  pages={eado7069},
  year={2024},
  publisher={American Association for the Advancement of Science}
}

@article{zhao2023brain,
  title={When brain-inspired ai meets agi},
  author={Zhao, Lin and Zhang, Lu and Wu, Zihao and Chen, Yuzhong and Dai, Haixing and Yu, Xiaowei and Liu, Zhengliang and Zhang, Tuo and Hu, Xintao and Jiang, Xi and others},
  journal={Meta-Radiology},
  volume={1},
  number={1},
  pages={100005},
  year={2023},
  publisher={Elsevier}
}

@incollection{rumelhart2017schemata,
  title={Schemata: The building blocks of cognition},
  author={Rumelhart, David E},
  booktitle={Theoretical issues in reading comprehension},
  pages={33--58},
  year={2017},
  publisher={Routledge}
}

@article{chen2025schema,
  title={Schema for In-Context Learning},
  author={Chen, Pan and Chen, Shaohong and Wang, Mark and Leong, Shi Xuan and Fung, Priscilla and Bernales, Varinia and Aspuru-Guzik, Alan},
  journal={arXiv preprint arXiv:2510.13905},
  year={2025}
}

@Misc{Smith2021schema,
howpublished = {\url{https://www.ebsco.com/research-starters/psychology/schema-theory}},
year = {2021},
title = {Schema Theory},
author = {Smith, Tricia}
}

@book{tompkins1993teaching,
  title={Teaching reading with literature: Case studies to action plans},
  author={Tompkins, Gail E and McGee, Lea M},
  year={1993},
  publisher={Prentice Hall}
}

@inproceedings{radford2021learning,
  title={Learning transferable visual models from natural language supervision},
  author={Radford, Alec and Kim, Jong Wook and Hallacy, Chris and Ramesh, Aditya and Goh, Gabriel and Agarwal, Sandhini and Sastry, Girish and Askell, Amanda and Mishkin, Pamela and Clark, Jack and others},
  booktitle={International conference on machine learning},
  pages={8748--8763},
  year={2021},
  organization={PmLR}
}

@article{awadalla2024mint,
  title={Mint-1t: Scaling open-source multimodal data by 10x: A multimodal dataset with one trillion tokens},
  author={Awadalla, Anas and Xue, Le and Lo, Oscar and Shu, Manli and Lee, Hannah and Guha, Etash and Shen, Sheng and Awadalla, Mohamed and Savarese, Silvio and Xiong, Caiming and others},
  journal={Advances in Neural Information Processing Systems},
  volume={37},
  pages={36805--36828},
  year={2024}
}

@article{zhang20252,
  title={2.5 years in class: A multimodal textbook for vision-language pretraining},
  author={Zhang, Wenqi and Zhang, Hang and Li, Xin and Sun, Jiashuo and Shen, Yongliang and Lu, Weiming and Zhao, Deli and Zhuang, Yueting and Bing, Lidong},
  journal={arXiv preprint arXiv:2501.00958},
  year={2025}
}

@article{wang2025scaling,
  title={Scaling pre-training to one hundred billion data for vision language models},
  author={Wang, Xiao and Alabdulmohsin, Ibrahim and Salz, Daniel and Li, Zhe and Rong, Keran and Zhai, Xiaohua},
  journal={arXiv preprint arXiv:2502.07617},
  year={2025}
}

@article{dong2025scalable,
  title={Scalable vision language model training via high quality data curation},
  author={Dong, Hongyuan and Kang, Zijian and Yin, Weijie and Liang, Xiao and Feng, Chao and Ran, Jiao},
  journal={arXiv preprint arXiv:2501.05952},
  year={2025}
}

@inproceedings{Bartlett1958ThinkingAE,
  title={Thinking: An Experimental and Social Study},
  author={Frederic Charles Bartlett},
  year={1958},
  url={https://api.semanticscholar.org/CorpusID:142700711}
}

@inproceedings{Bartlett1932RememberingAS,
  title={Remembering: A Study in Experimental and Social Psychology},
  author={Frederic C Bartlett},
  year={1932},
  url={https://api.semanticscholar.org/CorpusID:115933268}
}

@article{yu2024visrag,
  title={Visrag: Vision-based retrieval-augmented generation on multi-modality documents},
  author={Yu, Shi and Tang, Chaoyue and Xu, Bokai and Cui, Junbo and Ran, Junhao and Yan, Yukun and Liu, Zhenghao and Wang, Shuo and Han, Xu and Liu, Zhiyuan and others},
  journal={arXiv preprint arXiv:2410.10594},
  year={2024}
}

@article{wu2025mmsearch,
  title={MMSearch-R1: Incentivizing LMMs to Search},
  author={Wu, Jinming and Deng, Zihao and Li, Wei and Liu, Yiding and You, Bo and Li, Bo and Ma, Zejun and Liu, Ziwei},
  journal={arXiv preprint arXiv:2506.20670},
  year={2025}
}

@article{schick2023toolformer,
  title={Toolformer: Language models can teach themselves to use tools},
  author={Schick, Timo and Dwivedi-Yu, Jane and Dess{\`\i}, Roberto and Raileanu, Roberta and Lomeli, Maria and Hambro, Eric and Zettlemoyer, Luke and Cancedda, Nicola and Scialom, Thomas},
  journal={Advances in Neural Information Processing Systems},
  volume={36},
  pages={68539--68551},
  year={2023}
}

@article{whorf1956language,
  title={Language, thought, and reality: selected writings of….(Edited by John B. Carroll.).},
  author={Whorf, Benjamin Lee},
  year={1956},
  publisher={Technology Press of MIT}
}

@article{lucy1997linguistic,
  title={Linguistic relativity},
  author={Lucy, John A},
  journal={Annual review of anthropology},
  volume={26},
  number={1},
  pages={291--312},
  year={1997},
  publisher={Annual Reviews 4139 El Camino Way, PO Box 10139, Palo Alto, CA 94303-0139, USA}
}

@book{talmy2000toward,
  title={Toward a cognitive semantics: Concept structuring systems},
  author={Talmy, Leonard},
  volume={1},
  publisher={MIT Press},
  year={2000}
}

@article{boroditsky2001does,
  title={Does language shape thought?: Mandarin and English speakers' conceptions of time},
  author={Boroditsky, Lera},
  journal={Cognitive psychology},
  volume={43},
  number={1},
  pages={1--22},
  year={2001},
  publisher={Elsevier}
}

@article{ansorge2022linguistic,
  title={Linguistic skill and stimulus-driven attention: A case for linguistic relativity},
  author={Ansorge, Ulrich and Baier, Diane and Choi, Soonja},
  journal={Frontiers in Psychology},
  volume={13},
  pages={875744},
  year={2022},
  publisher={Frontiers Media SA}
}

@inproceedings{fausey2008english,
  title={English and Spanish speakers remember causal agents differently},
  author={Fausey, Caitlin M and Boroditsky, Lera},
  booktitle={Proceedings of the Annual Meeting of the Cognitive Science Society},
  volume={30},
  number={30},
  year={2008}
}

@article{edwards1993language,
  title={Language and causation: A discursive action model of description and attribution.},
  author={Edwards, Derek and Potter, Jonathan},
  journal={Psychological review},
  volume={100},
  number={1},
  pages={23},
  year={1993},
  publisher={American Psychological Association}
}

@article{boroditsky2011language,
  title={How language shapes thought},
  author={Boroditsky, Lera},
  journal={Scientific American},
  volume={304},
  number={2},
  pages={62--65},
  year={2011},
  publisher={JSTOR}
}

@book{Wilhelm1996from,
    author={von Humboldt, Wilhelm},
    title = {From ‘thought and language’ to ‘thinking for speaking’.},
     publisher={Cambridge University Press},
    year = {1996}
}

@article{lake2017building,
  title={Building machines that learn and think like people},
  author={Lake, Brenden M and Ullman, Tomer D and Tenenbaum, Joshua B and Gershman, Samuel J},
  journal={Behavioral and brain sciences},
  volume={40},
  pages={e253},
  year={2017},
  publisher={Cambridge University Press}
}

@article{binz2023using,
  title={Using cognitive psychology to understand GPT-3},
  author={Binz, Marcel and Schulz, Eric},
  journal={Proceedings of the National Academy of Sciences},
  volume={120},
  number={6},
  pages={e2218523120},
  year={2023},
  publisher={National Academy of Sciences}
}

@inproceedings{bisk2020experience,
  title={Experience Grounds Language},
  author={Bisk, Yonatan and Holtzman, Ari and Thomason, Jesse and Andreas, Jacob and Bengio, Yoshua and Chai, Joyce and Lapata, Mirella and Lazaridou, Angeliki and May, Jonathan and Nisnevich, Aleksandr and others},
  booktitle={Proceedings of the 2020 Conference on Empirical Methods in Natural Language Processing (EMNLP)},
  pages={8718--8735},
  year={2020}
}

@inproceedings{piantadosi2022meaning,
  title={Meaning without reference in large language models},
  author={Piantadosi, Steven and Hill, Felix},
  booktitle={NeurIPS 2022 Workshop on Neuro Causal and Symbolic AI (nCSI)},
year={2022}
}

@article{ameisen2025circuit,
  title={Circuit tracing: Revealing computational graphs in language models},
  author={Ameisen, Emmanuel and Lindsey, Jack and Pearce, Adam and Gurnee, Wes and Turner, Nicholas L and Chen, Brian and Citro, Craig and Abrahams, David and Carter, Shan and Hosmer, Basil and others},
  journal={Transformer Circuits Thread},
  volume={6},
  year={2025}
}

@article{2025semantic,
  title={Semantic structure in large language model embeddings},
  author={Kozlowski, Austin C and Dai, Callin and Boutyline, Andrei},
  journal={arXiv preprint arXiv:2508.10003},
  year={2025}
}

@inproceedings{dong2024survey,
  title={A survey on in-context learning},
  author={Dong, Qingxiu and Li, Lei and Dai, Damai and Zheng, Ce and Ma, Jingyuan and Li, Rui and Xia, Heming and Xu, Jingjing and Wu, Zhiyong and Chang, Baobao and others},
  booktitle={Proceedings of the 2024 conference on empirical methods in natural language processing},
  pages={1107--1128},
  year={2024}
}

@article{wei2022chain,
  title={Chain-of-thought prompting elicits reasoning in large language models},
  author={Wei, Jason and Wang, Xuezhi and Schuurmans, Dale and Bosma, Maarten and Xia, Fei and Chi, Ed and Le, Quoc V and Zhou, Denny and others},
  journal={Advances in neural information processing systems},
  volume={35},
  pages={24824--24837},
  year={2022}
}

@inproceedings{wang2025under,
  title={Under the Shadow of Babel: How Language Shapes Reasoning in LLMs},
  author={Wang, Chenxi and Zhang, Yixuan and Gao, Lang and Xu, Zixiang and Song, Zirui and Wang, Yanbo and Chen, Xiuying},
  booktitle = {Findings of the Association for Computational Linguistics: EMNLP 2025},
  pages={24327--24344},
  year={2025}
}

@inproceedings{wang2023towards,
  title={Towards understanding chain-of-thought prompting: An empirical study of what matters},
  author={Wang, Boshi and Min, Sewon and Deng, Xiang and Shen, Jiaming and Wu, You and Zettlemoyer, Luke and Sun, Huan},
  booktitle={Proceedings of the 61st annual meeting of the association for computational linguistics (volume 1: Long papers)},
  pages={2717--2739},
  year={2023}
}

@article{yousefi2023decoding,
  title={Decoding in-context learning: Neuroscience-inspired analysis of representations in large language models},
  author={Yousefi, Safoora and Betthauser, Leo and Hasanbeig, Hosein and Milli{\`e}re, Rapha{\"e}l and Momennejad, Ida},
  journal={arXiv preprint arXiv:2310.00313},
  year={2023}
}

@inproceedings{gao2021making,
  title={Making pre-trained language models better few-shot learners},
  author={Gao, Tianyu and Fisch, Adam and Chen, Danqi},
  booktitle={Proceedings of the 59th annual meeting of the association for computational linguistics and the 11th international joint conference on natural language processing (volume 1: long papers)},
  pages={3816--3830},
  year={2021}
}

@article{perez2021true,
  title={True few-shot learning with language models},
  author={Perez, Ethan and Kiela, Douwe and Cho, Kyunghyun},
  journal={Advances in neural information processing systems},
  volume={34},
  pages={11054--11070},
  year={2021}
}

@article{naik2018stress,
  title={Stress test evaluation for natural language inference},
  author={Naik, Aakanksha and Ravichander, Abhilasha and Sadeh, Norman and Rose, Carolyn and Neubig, Graham},
  journal={arXiv preprint arXiv:1806.00692},
  year={2018}
}

@article{salinas2024butterfly,
  title={The butterfly effect of altering prompts: How small changes and jailbreaks affect large language model performance},
  author={Salinas, Abel and Morstatter, Fred},
  journal={arXiv preprint arXiv:2401.03729},
  year={2024}
}

@article{pinker2007language,
  title={The language instinct (1994/2007)},
  author={Pinker, Steven},
  year={2007},
  publisher={New York, NY: Harper Perennial Modern Classics}
}

@article{radford2019language,
  title={Language models are unsupervised multitask learners},
  author={Radford, Alec and Wu, Jeffrey and Child, Rewon and Luan, David and Amodei, Dario and Sutskever, Ilya and others},
  journal={OpenAI blog},
  volume={1},
  number={8},
  pages={9},
  year={2019}
}

@article{wei2024improving,
  title={Improving parallel program performance through dsl-driven code generation with llm optimizers},
  author={Wei, Anjiang and Nie, Allen and Teixeira, Thiago SFX and Yadav, Rohan and Lee, Wonchan and Wang, Ke and Aiken, Alex},
  journal={arXiv e-prints},
  pages={arXiv--2410},
  year={2024}
}

@inproceedings{huang2022language,
  title={Language models as zero-shot planners: Extracting actionable knowledge for embodied agents},
  author={Huang, Wenlong and Abbeel, Pieter and Pathak, Deepak and Mordatch, Igor},
  booktitle={International conference on machine learning},
  pages={9118--9147},
  year={2022},
  organization={PMLR}
}

@article{wittgenstein1922tractatus,
  title={Tractatus Logico-Philosophicus},
  author={Wittgenstein, Ludwig},
  year={1922}
}

@inproceedings{dhanraj-eliasmith-2025-improving,
    title = "Improving Rule-based Reasoning in {LLM}s using Neurosymbolic Representations",
    author = "Dhanraj, Varun  and
      Eliasmith, Chris",
    booktitle = "Proceedings of the 2025 Conference on Empirical Methods in Natural Language Processing",
    year = "2025",
    pages = "30577--30596",
}

@article{cao2024worst,
  title={On the worst prompt performance of large language models},
  author={Cao, Bowen and Cai, Deng and Zhang, Zhisong and Zou, Yuexian and Lam, Wai},
  journal={Advances in Neural Information Processing Systems},
  volume={37},
  pages={69022--69042},
  year={2024}
}

@inproceedings{raspanti-etal-2025-grammar,
    title = "Grammar-Constrained Decoding Makes Large Language Models Better Logical Parsers",
    author = "Raspanti, Federico  and
      Ozcelebi, Tanir  and
      Holenderski, Mike",
    editor = "Rehm, Georg  and
      Li, Yunyao",
    booktitle = "Proceedings of the 63rd Annual Meeting of the Association for Computational Linguistics (Volume 6: Industry Track)",
    month = jul,
    year = "2025",
    address = "Vienna, Austria",
    publisher = "Association for Computational Linguistics",
    url = "https://aclanthology.org/2025.acl-industry.34/",
    doi = "10.18653/v1/2025.acl-industry.34",
    pages = "485--499",
    ISBN = "979-8-89176-288-6"
}

@article{swaminathan2023schema,
  title={Schema-learning and rebinding as mechanisms of in-context learning and emergence},
  author={Swaminathan, Sivaramakrishnan and Dedieu, Antoine and Vasudeva Raju, Rajkumar and Shanahan, Murray and Lazaro-Gredilla, Miguel and George, Dileep},
  journal={Advances in neural information processing systems},
  volume={36},
  pages={28785--28804},
  year={2023}
}

@inproceedings{shin2021constrained,
  title={Constrained language models yield few-shot semantic parsers},
  author={Shin, Richard and Lin, Christopher and Thomson, Sam and Chen Jr, Charles and Roy, Subhro and Platanios, Emmanouil Antonios and Pauls, Adam and Klein, Dan and Eisner, Jason and Van Durme, Benjamin},
  booktitle={Proceedings of the 2021 conference on empirical methods in natural language processing},
  pages={7699--7715},
  year={2021}
}

@inproceedings{gupta2025schema,
  title={Schema and Natural Language Aware In-Context Learning for Improved GraphQL Query Generation},
  author={Gupta, Nitin and Kesarwani, Manish and Ghosh, Sambit and Mehta, Sameep and Eberhardt, Carlos and Debrunner, Dan},
  booktitle={Proceedings of the 2025 Conference of the Nations of the Americas Chapter of the Association for Computational Linguistics: Human Language Technologies (Volume 3: Industry Track)},
  pages={1009--1015},
  year={2025}
}

@article{labate2024infusing,
  title={Infusing Prompts with Syntax and Semantics},
  author={Labate, Anton Bulle and Cozman, Fabio Gagliardi},
  journal={arXiv preprint arXiv:2412.06107},
  year={2024}
}

@inproceedings{zhu2023promptrobust,
  title={Promptrobust: Towards evaluating the robustness of large language models on adversarial prompts},
  author={Zhu, Kaijie and Wang, Jindong and Zhou, Jiaheng and Wang, Zichen and Chen, Hao and Wang, Yidong and Yang, Linyi and Ye, Wei and Zhang, Yue and Gong, Neil and others},
  booktitle={Proceedings of the 1st ACM workshop on large AI systems and models with privacy and safety analysis},
  pages={57--68},
  year={2023}
}

@book{ghallab2004automated,
  title={Automated Planning: theory and practice},
  author={Ghallab, Malik and Nau, Dana and Traverso, Paolo},
  year={2004},
  publisher={Elsevier}
}

@inproceedings{valmeekam2022large,
  title={Large language models still can't plan (a benchmark for LLMs on planning and reasoning about change)},
  author={Valmeekam, Karthik and Olmo, Alberto and Sreedharan, Sarath and Kambhampati, Subbarao},
  booktitle={NeurIPS 2022 Foundation Models for Decision Making Workshop},
  year={2022}
}

@article{polu2020generative,
  title={Generative language modeling for automated theorem proving},
  author={Polu, Stanislas and Sutskever, Ilya},
  journal={arXiv preprint arXiv:2009.03393},
  year={2020}
}

@article{park2024grammar,
  title={Grammar-aligned decoding},
  author={Park, Kanghee and Wang, Jiayu and Berg-Kirkpatrick, Taylor and Polikarpova, Nadia and D'Antoni, Loris},
  journal={Advances in Neural Information Processing Systems},
  volume={37},
  pages={24547--24568},
  year={2024}
}

@inproceedings{chae2024language,
  title={Language models as compilers: Simulating pseudocode execution improves algorithmic reasoning in language models},
  author={Chae, Hyungjoo and Kim, Yeonghyeon and Kim, Seungone and Ong, Kai Tzu-iunn and Kwak, Beong-woo and Kim, Moohyeon and Mac Kim, Sunghwan and Kwon, Taeyoon and Chung, Jiwan and Yu, Youngjae and others},
  booktitle={Proceedings of the 2024 Conference on Empirical Methods in Natural Language Processing},
  pages={22471--22502},
  year={2024}
}

@inproceedings{pan2023logic,
  title={Logic-lm: Empowering large language models with symbolic solvers for faithful logical reasoning},
  author={Pan, Liangming and Albalak, Alon and Wang, Xinyi and Wang, William},
  booktitle={Findings of the Association for Computational Linguistics: EMNLP 2023},
  pages={3806--3824},
  year={2023}
}

@article{grigorev2025verifyllm,
  title={Verifyllm: Llm-based pre-execution task plan verification for robots},
  author={Grigorev, Danil S and Kovalev, Alexey K and Panov, Aleksandr I},
  journal={arXiv preprint arXiv:2507.05118},
  year={2025}
}

@article{xu2024faithful,
  title={Faithful logical reasoning via symbolic chain-of-thought},
  author={Xu, Jundong and Fei, Hao and Pan, Liangming and Liu, Qian and Lee, Mong-Li and Hsu, Wynne},
  journal={arXiv preprint arXiv:2405.18357},
  year={2024}
}

@article{cao2025towards,
  title={Towards Advanced Mathematical Reasoning for LLMs via First-Order Logic Theorem Proving},
  author={Cao, Chuxue and Li, Mengze and Dai, Juntao and Yang, Jinluan and Zhao, Zijian and Zhang, Shengyu and Shi, Weijie and Liu, Chengzhong and Han, Sirui and Guo, Yike},
  journal={arXiv preprint arXiv:2506.17104},
  year={2025}
}

@inproceedings{suris2023vipergpt,
  title={Vipergpt: Visual inference via python execution for reasoning},
  author={Sur{\'\i}s, D{\'\i}dac and Menon, Sachit and Vondrick, Carl},
  booktitle={Proceedings of the IEEE/CVF international conference on computer vision},
  pages={11888--11898},
  year={2023}
}

@article{jaech2024openai,
  title={Openai o1 system card},
  author={Jaech, Aaron and Kalai, Adam and Lerer, Adam and Richardson, Adam and El-Kishky, Ahmed and Low, Aiden and Helyar, Alec and Madry, Aleksander and Beutel, Alex and Carney, Alex and others},
  journal={arXiv preprint arXiv:2412.16720},
  year={2024}
}

@article{guo2025deepseek,
  title={Deepseek-r1: Incentivizing reasoning capability in llms via reinforcement learning},
  author={Guo, Daya and Yang, Dejian and Zhang, Haowei and Song, Junxiao and Zhang, Ruoyu and Xu, Runxin and Zhu, Qihao and Ma, Shirong and Wang, Peiyi and Bi, Xiao and others},
  journal={arXiv preprint arXiv:2501.12948},
  year={2025}
}

@article{team2025kimi,
  title={Kimi k1. 5: Scaling reinforcement learning with llms},
  author={Team, Kimi and Du, Angang and Gao, Bofei and Xing, Bowei and Jiang, Changjiu and Chen, Cheng and Li, Cheng and Xiao, Chenjun and Du, Chenzhuang and Liao, Chonghua and others},
  journal={arXiv preprint arXiv:2501.12599},
  year={2025}
}

@article{brown2020language,
  title={Language models are few-shot learners},
  author={Brown, Tom and Mann, Benjamin and Ryder, Nick and Subbiah, Melanie and Kaplan, Jared D and Dhariwal, Prafulla and Neelakantan, Arvind and Shyam, Pranav and Sastry, Girish and Askell, Amanda and others},
  journal={Advances in neural information processing systems},
  volume={33},
  pages={1877--1901},
  year={2020}
}

@article{weiemergent,
  title={Emergent Abilities of Large Language Models},
  author={Wei, Jason and Tay, Yi and Bommasani, Rishi and Raffel, Colin and Zoph, Barret and Borgeaud, Sebastian and Yogatama, Dani and Bosma, Maarten and Zhou, Denny and Metzler, Donald and others},
  journal={Transactions on Machine Learning Research},
  year={2022}
}

@inproceedings{hoffmann2022training,
  title={Training compute-optimal large language models},
  author={Hoffmann, Jordan and Borgeaud, Sebastian and Mensch, Arthur and Buchatskaya, Elena and Cai, Trevor and Rutherford, Eliza and de Las Casas, Diego and Hendricks, Lisa Anne and Welbl, Johannes and Clark, Aidan and others},
  booktitle={Proceedings of the 36th International Conference on Neural Information Processing Systems},
  pages={30016--30030},
  year={2022}
}

@inproceedings{yao2022react,
  title={React: Synergizing reasoning and acting in language models},
  author={Yao, Shunyu and Zhao, Jeffrey and Yu, Dian and Du, Nan and Shafran, Izhak and Narasimhan, Karthik R and Cao, Yuan},
  booktitle={The eleventh international conference on learning representations},
  year={2022}
}

@inproceedings{petroni2019language,
  title={Language models as knowledge bases?},
  author={Petroni, Fabio and Rockt{\"a}schel, Tim and Riedel, Sebastian and Lewis, Patrick and Bakhtin, Anton and Wu, Yuxiang and Miller, Alexander},
  booktitle={Proceedings of the 2019 conference on empirical methods in natural language processing and the 9th international joint conference on natural language processing (EMNLP-IJCNLP)},
  pages={2463--2473},
  year={2019}
}

@inproceedings{roberts2020much,
  title={How Much Knowledge Can You Pack Into the Parameters of a Language Model?},
  author={Roberts, Adam and Raffel, Colin and Shazeer, Noam},
  booktitle={Proceedings of the 2020 Conference on Empirical Methods in Natural Language Processing (EMNLP)},
  pages={5418--5426},
  year={2020}
}

@article{chen2025seed,
  title={Seed-Prover 1.5: Mastering Undergraduate-Level Theorem Proving via Learning from Experience},
  author={Chen, Jiangjie and Chen, Wenxiang and Du, Jiacheng and Hu, Jinyi and Jiang, Zhicheng and Jie, Allan and Jin, Xiaoran and Jin, Xing and Li, Chenggang and Shi, Wenlei and others},
  journal={arXiv preprint arXiv:2512.17260},
  year={2025}
}

@article{ma2026thinking,
  title={Thinking with Blueprints: Assisting Vision-Language Models in Spatial Reasoning via Structured Object Representation},
  author={Ma, Weijian and Sun, Shizhao and Yu, Tianyu and Wang, Ruiyu and Chua, Tat-Seng and Bian, Jiang},
  journal={arXiv preprint arXiv:2601.01984},
  year={2026}
}

@article{geng2025generating,
  title={Generating structured outputs from language models: Benchmark and studies},
  author={Geng, Saibo and Cooper, Hudson and Moskal, Micha{\l} and Jenkins, Samuel and Berman, Julian and Ranchin, Nathan and West, Robert and Horvitz, Eric and Nori, Harsha},
  journal={arXiv e-prints},
  pages={arXiv--2501},
  year={2025}
}

@article{yang2025qwen3,
  title={Qwen3 technical report},
  author={Yang, An and Li, Anfeng and Yang, Baosong and Zhang, Beichen and Hui, Binyuan and Zheng, Bo and Yu, Bowen and Gao, Chang and Huang, Chengen and Lv, Chenxu and others},
  journal={arXiv preprint arXiv:2505.09388},
  year={2025}
}

@inproceedings{ishay2025llm,
  title={Llm+ al: Bridging large language models and action languages for complex reasoning about actions},
  author={Ishay, Adam and Lee, Joohyung},
  booktitle={Proceedings of the AAAI Conference on Artificial Intelligence},
  volume={39},
  number={23},
  pages={24212--24220},
  year={2025}
}

@article{berti2025emergent,
  title={Emergent abilities in large language models: A survey},
  author={Berti, Leonardo and Giorgi, Flavio and Kasneci, Gjergji},
  journal={arXiv preprint arXiv:2503.05788},
  year={2025}
}

@inproceedings{lipredictable,
  title={Predictable Scale (Part II)---Farseer: A Refined Scaling Law in LLMs},
  author={Li, Houyi and Zheng, Wenzhen and Wang, Qiufeng and Ding, Zhenyu and Wang, Haoying and Wang, Zili and Xuyang, Shijie and Ding, Ning and Zhou, Shuigeng and Zhang, Xiangyu and others},
  booktitle={The Thirty-ninth Annual Conference on Neural Information Processing Systems},
year={2025}
}

@article{ruan2024observational,
  title={Observational scaling laws and the predictability of langauge model performance},
  author={Ruan, Yangjun and Maddison, Chris J and Hashimoto, Tatsunori B},
  journal={Advances in Neural Information Processing Systems},
  volume={37},
  pages={15841--15892},
  year={2024}
}

@article{purohit2025sample,
  title={Sample Efficient Demonstration Selection for In-Context Learning},
  author={Purohit, Kiran and Venktesh, V and Bhattacharya, Sourangshu and Anand, Avishek},
  journal={arXiv preprint arXiv:2506.08607},
  year={2025}
}

@article{shukor2025scaling,
  title={Scaling laws for optimal data mixtures},
  author={Shukor, Mustafa and Bethune, Louis and Busbridge, Dan and Grangier, David and Fini, Enrico and El-Nouby, Alaaeldin and Ablin, Pierre},
  journal={arXiv preprint arXiv:2507.09404},
  year={2025}
}

@article{wu2024inference,
  title={Inference scaling laws: An empirical analysis of compute-optimal inference for problem-solving with language models},
  author={Wu, Yangzhen and Sun, Zhiqing and Li, Shanda and Welleck, Sean and Yang, Yiming},
  journal={arXiv preprint arXiv:2408.00724},
  year={2024}
}

@article{srivastava2023beyond,
  title={Beyond the imitation game: Quantifying and extrapolating the capabilities of language models},
  author={Srivastava, Aarohi and Rastogi, Abhinav and Rao, Abhishek and Shoeb, Abu Awal Md and Abid, Abubakar and Fisch, Adam and Brown, Adam R and Santoro, Adam and Gupta, Aditya and Garriga-Alonso, Adri{\`a} and others},
  journal={Transactions on machine learning research},
  year={2023}
}

@article{muennighoff2023scaling,
  title={Scaling data-constrained language models},
  author={Muennighoff, Niklas and Rush, Alexander and Barak, Boaz and Le Scao, Teven and Tazi, Nouamane and Piktus, Aleksandra and Pyysalo, Sampo and Wolf, Thomas and Raffel, Colin A},
  journal={Advances in Neural Information Processing Systems},
  volume={36},
  pages={50358--50376},
  year={2023}
}

@article{patil2024gorilla,
  title={Gorilla: Large language model connected with massive apis},
  author={Patil, Shishir G and Zhang, Tianjun and Wang, Xin and Gonzalez, Joseph E},
  journal={Advances in Neural Information Processing Systems},
  volume={37},
  pages={126544--126565},
  year={2024}
}

@article{lu2023chameleon,
  title={Chameleon: Plug-and-play compositional reasoning with large language models},
  author={Lu, Pan and Peng, Baolin and Cheng, Hao and Galley, Michel and Chang, Kai-Wei and Wu, Ying Nian and Zhu, Song-Chun and Gao, Jianfeng},
  journal={Advances in Neural Information Processing Systems},
  volume={36},
  pages={43447--43478},
  year={2023}
}

@inproceedings{gao2023pal,
  title={Pal: Program-aided language models},
  author={Gao, Luyu and Madaan, Aman and Zhou, Shuyan and Alon, Uri and Liu, Pengfei and Yang, Yiming and Callan, Jamie and Neubig, Graham},
  booktitle={International Conference on Machine Learning},
  pages={10764--10799},
  year={2023},
  organization={PMLR}
}

@article{liu2023llm+,
  title={Llm+ p: Empowering large language models with optimal planning proficiency},
  author={Liu, Bo and Jiang, Yuqian and Zhang, Xiaohan and Liu, Qiang and Zhang, Shiqi and Biswas, Joydeep and Stone, Peter},
  journal={arXiv preprint arXiv:2304.11477},
  year={2023}
}

@article{shen2023hugginggpt,
  title={Hugginggpt: Solving ai tasks with chatgpt and its friends in hugging face},
  author={Shen, Yongliang and Song, Kaitao and Tan, Xu and Li, Dongsheng and Lu, Weiming and Zhuang, Yueting},
  journal={Advances in Neural Information Processing Systems},
  volume={36},
  pages={38154--38180},
  year={2023}
}

@article{qin2023toolllm,
  title={Toolllm: Facilitating large language models to master 16000+ real-world apis},
  author={Qin, Yujia and Liang, Shihao and Ye, Yining and Zhu, Kunlun and Yan, Lan and Lu, Yaxi and Lin, Yankai and Cong, Xin and Tang, Xiangru and Qian, Bill and others},
  journal={arXiv preprint arXiv:2307.16789},
  year={2023}
}

@inproceedings{hu2024chain,
  title={Chain-of-symbol prompting for spatial reasoning in large language models},
  author={Hu, Hanxu and Lu, Hongyuan and Zhang, Huajian and Song, Yun-Ze and Lam, Wai and Zhang, Yue},
  booktitle={First Conference on Language Modeling},
  year={2024}
}

@inproceedings{besta2024graph,
  title={Graph of thoughts: Solving elaborate problems with large language models},
  author={Besta, Maciej and Blach, Nils and Kubicek, Ales and Gerstenberger, Robert and Podstawski, Michal and Gianinazzi, Lukas and Gajda, Joanna and Lehmann, Tomasz and Niewiadomski, Hubert and Nyczyk, Piotr and others},
  booktitle={Proceedings of the AAAI conference on artificial intelligence},
  volume={38},
  number={16},
  pages={17682--17690},
  year={2024}
}

@inproceedings{xu2024symbol,
  title={Symbol-llm: Towards foundational symbol-centric interface for large language models},
  author={Xu, Fangzhi and Wu, Zhiyong and Sun, Qiushi and Ren, Siyu and Yuan, Fei and Yuan, Shuai and Lin, Qika and Qiao, Yu and Liu, Jun},
  booktitle={Proceedings of the 62nd Annual Meeting of the Association for Computational Linguistics (Volume 1: Long Papers)},
  pages={13091--13116},
  year={2024}
}

@article{dhanraj2025improving,
  title={Improving Rule-based Reasoning in LLMs using Neurosymbolic Representations},
  author={Dhanraj, Varun and Eliasmith, Chris},
  journal={arXiv preprint arXiv:2502.01657},
  year={2025}
}

@article{yang2024buffer,
  title={Buffer of thoughts: Thought-augmented reasoning with large language models},
  author={Yang, Ling and Yu, Zhaochen and Zhang, Tianjun and Cao, Shiyi and Xu, Minkai and Zhang, Wentao and Gonzalez, Joseph E and Cui, Bin},
  journal={Advances in Neural Information Processing Systems},
  volume={37},
  pages={113519--113544},
  year={2024}
}

@article{zelikman2024quiet,
  title={Quiet-star: Language models can teach themselves to think before speaking},
  author={Zelikman, Eric and Harik, Georges and Shao, Yijia and Jayasiri, Varuna and Haber, Nick and Goodman, Noah D},
  journal={arXiv preprint arXiv:2403.09629},
  year={2024}
}

@article{kumon2025analyzing,
  title={Analyzing the Inner Workings of Transformers in Compositional Generalization},
  author={Kumon, Ryoma and Yanaka, Hitomi},
  journal={arXiv preprint arXiv:2502.15277},
  year={2025}
}

@article{templeton2024scaling,
       title={Scaling Monosemanticity: Extracting Interpretable Features from Claude 3 Sonnet},
       author={Templeton, Adly and Conerly, Tom and Marcus, Jonathan and Lindsey, Jack and Bricken, Trenton and Chen, Brian and Pearce, Adam and Citro, Craig and Ameisen, Emmanuel and Jones, Andy and Cunningham, Hoagy and Turner, Nicholas L and McDougall, Callum and MacDiarmid, Monte and Freeman, C. Daniel and Sumers, Theodore R. and Rees, Edward and Batson, Joshua and Jermyn, Adam and Carter, Shan and Olah, Chris and Henighan, Tom},
       year={2024},
       journal={Transformer Circuits Thread},
       url={https://transformer-circuits.pub/2024/scaling-monosemanticity/index.html}
    }

@inproceedings{liu2025conditions,
  title={Conditions for Catastrophic Forgetting in Multilingual Translation},
  author={Liu, Danni and Niehues, Jan},
  booktitle={Proceedings of the 5th Workshop on Multilingual Representation Learning (MRL 2025)},
  pages={347--359},
  year={2025}
}

@article{nikolaou2025language,
  title={Language models are injective and hence invertible},
  author={Nikolaou, Giorgos and Mencattini, Tommaso and Crisostomi, Donato and Santilli, Andrea and Panagakis, Yannis and Rodol{\`a}, Emanuele},
  journal={arXiv preprint arXiv:2510.15511},
  year={2025}
}

@article{wang2025logical,
  title={Logical forms complement probability in understanding language model (and human) performance},
  author={Wang, Yixuan and Shi, Freda},
  journal={arXiv preprint arXiv:2502.09589},
  year={2025}
}

@article{vig2019analyzing,
  title={Analyzing the structure of attention in a transformer language model},
  author={Vig, Jesse and Belinkov, Yonatan},
  journal={arXiv preprint arXiv:1906.04284},
  year={2019}
}

@article{abnar2020quantifying,
  title={Quantifying attention flow in transformers},
  author={Abnar, Samira and Zuidema, Willem},
  journal={arXiv preprint arXiv:2005.00928},
  year={2020}
}

@misc{john2025power,
  author = {John, Yohan J.},
  title = {The Power of Scale in Machine Learning},
  howpublished = {\url{https://kempnerinstitute.harvard.edu/news/the-power-of-scale-in-machine-learning/}},
  year = {2025},
  month = {Aug},
  note = {Kempner Institute at Harvard University}
}

@misc{sutskever2025dwarkesh,
  author = {Sutskever, Ilya and Patel, Dwarkesh},
  title = {Ilya Sutskever: We're moving from the age of scaling to the age of research},
  howpublished = {The Dwarkesh Podcast},
  year = {2025},
  month = {nov},
  url = {https://www.dwarkesh.com/p/ilya-sutskever-2},
  note = {Published on November 25, 2025}
}

@article{mohsin2025fundamental,
  title={On the Fundamental Limits of LLMs at Scale},
  author={Mohsin, Muhammad Ahmed and Umer, Muhammad and Bilal, Ahsan and Memon, Zeeshan and Qadir, Muhammad Ibtsaam and Bhattacharya, Sagnik and Rizwan, Hassan and Gorle, Abhiram R and Kazmi, Maahe Zehra and Mohsin, Ayesha and others},
  journal={arXiv preprint arXiv:2511.12869},
  year={2025}
}

@article{piantadosi2012communicative,
  title={The communicative function of ambiguity in language},
  author={Piantadosi, Steven T and Tily, Harry and Gibson, Edward},
  journal={Cognition},
  volume={122},
  number={3},
  pages={280--291},
  year={2012},
  publisher={Elsevier}
}

@inproceedings{bender2020climbing,
  title={Climbing towards NLU: On meaning, form, and understanding in the age of data},
  author={Bender, Emily M and Koller, Alexander},
  booktitle={Proceedings of the 58th annual meeting of the association for computational linguistics},
  pages={5185--5198},
  year={2020}
}

@inproceedings{reynolds2021prompt,
  title={Prompt programming for large language models: Beyond the few-shot paradigm},
  author={Reynolds, Laria and McDonell, Kyle},
  booktitle={Extended abstracts of the 2021 CHI conference on human factors in computing systems},
  pages={1--7},
  year={2021}
}

@inproceedings{huang2025limit,
  title={On the limit of language models as planning formalizers},
  author={Huang, Cassie and Zhang, Li},
  booktitle={Proceedings of the 63rd Annual Meeting of the Association for Computational Linguistics (Volume 1: Long Papers)},
  pages={4880--4904},
  year={2025}
}

@article{wang2023voyager,
  title={Voyager: An open-ended embodied agent with large language models},
  author={Wang, Guanzhi and Xie, Yuqi and Jiang, Yunfan and Mandlekar, Ajay and Xiao, Chaowei and Zhu, Yuke and Fan, Linxi and Anandkumar, Anima},
  journal={arXiv preprint arXiv:2305.16291},
  year={2023}
}

@article{liang2022code,
  title={Code as policies: Language model programs for embodied control},
  author={Liang, Jacky and Huang, Wenlong and Xia, Fei and Xu, Peng and Hausman, Karol and Ichter, Brian and Florence, Pete and Zeng, Andy},
  journal={arXiv preprint arXiv:2209.07753},
  year={2022}
}

@inproceedings{zhai2023stabilizing,
  title={Stabilizing transformer training by preventing attention entropy collapse},
  author={Zhai, Shuangfei and Likhomanenko, Tatiana and Littwin, Etai and Busbridge, Dan and Ramapuram, Jason and Zhang, Yizhe and Gu, Jiatao and Susskind, Joshua M},
  booktitle={International Conference on Machine Learning},
  pages={40770--40803},
  year={2023},
  organization={PMLR}
}

@inproceedings{kornblith2019similarity,
  title={Similarity of neural network representations revisited},
  author={Kornblith, Simon and Norouzi, Mohammad and Lee, Honglak and Hinton, Geoffrey},
  booktitle={International conference on machine learning},
  pages={3519--3529},
  year={2019},
  organization={PMlR}
}

@article{jiang2024tracing,
  title={Tracing representation progression: Analyzing and enhancing layer-wise similarity},
  author={Jiang, Jiachen and Zhou, Jinxin and Zhu, Zhihui},
  journal={arXiv preprint arXiv:2406.14479},
  year={2024}
}

@article{ramji2026thinking,
  title={Thinking Without Words: Efficient Latent Reasoning with Abstract Chain-of-Thought},
  author={Ramji, Keshav and Naseem, Tahira and Astudillo, Ram{\'o}n Fernandez},
  journal={arXiv preprint arXiv:2604.22709},
  year={2026}
}

@article{wang2026agentspex,
  title={AgentSPEX: An Agent SPecification and EXecution Language},
  author={Wang, Pengcheng and Huang, Jerry and Yao, Jiarui and Pan, Rui and Niu, Peizhi and Liu, Yaowenqi and Wang, Ruida and Lu, Renhao and Guo, Yuwei and Zhang, Tong},
  journal={arXiv preprint arXiv:2604.13346},
  year={2026}
}

@article{zou2026constituent,
  title={Constituent-constrained word prediction during language comprehension},
  author={Zou, Jiajie and Poeppel, David and Ding, Nai},
  journal={Nature Neuroscience},
  pages={1--12},
  year={2026},
  publisher={Nature Publishing Group US New York}
}

@article{shi2025world,
  title={World-aware planning narratives enhance large vision-language model planner},
  author={Shi, Junhao and Fei, Zhaoye and Wang, Siyin and Guo, Qipeng and Gong, Jingjing and Qiu, Xipeng},
  journal={arXiv preprint arXiv:2506.21230},
  year={2025}
}

@article{choi2025nesyc,
  title={Nesyc: A neuro-symbolic continual learner for complex embodied tasks in open domains},
  author={Choi, Wonje and Park, Jinwoo and Ahn, Sanghyun and Lee, Daehee and Woo, Honguk},
  journal={arXiv preprint arXiv:2503.00870},
  year={2025}
}

@article{choi2025nesypr,
  title={NeSyPr: Neurosymbolic Proceduralization For Efficient Embodied Reasoning},
  author={Choi, Wonje and Kim, Jooyoung and Woo, Honguk},
  journal={arXiv preprint arXiv:2510.19429},
  year={2025}
}

@article{ahn2025towards,
  title={Towards reliable code-as-policies: A neuro-symbolic framework for embodied task planning},
  author={Ahn, Sanghyun and Choi, Wonje and Lee, Junyong and Park, Jinwoo and Woo, Honguk},
  journal={arXiv preprint arXiv:2510.21302},
  year={2025}
}

@article{chen2025geometrically,
  title={Geometrically-Constrained Agent for Spatial Reasoning},
  author={Chen, Zeren and Lu, Xiaoya and Zheng, Zhijie and Li, Pengrui and He, Lehan and Zhou, Yijin and Shao, Jing and Zhuang, Bohan and Sheng, Lu},
  journal={arXiv preprint arXiv:2511.22659},
  year={2025}
}
